\documentclass[twoside]{article}

\usepackage[preprint]{aistats2026}
\usepackage[round]{natbib}

\usepackage[utf8]{inputenc} 
\usepackage[T1]{fontenc}    
\usepackage{hyperref}       
\usepackage{url}            
\usepackage{booktabs}       
\usepackage{amsfonts}       
\usepackage{nicefrac}       
\usepackage{microtype}      
\usepackage{xcolor}         
\usepackage{amsmath}
\usepackage{amssymb}
\usepackage{mathtools}
\usepackage{amsthm}

\usepackage{graphbox}       
\usepackage{wrapfig}        
\usepackage{multirow}       
\usepackage{bm}    
\usepackage{minitoc}
\usepackage[most]{tcolorbox} 
\usepackage{makecell} 
\usepackage{tabularx} 
\usepackage{tabularray} 
\usepackage{subcaption} 
\usepackage{adjustbox} 
\usepackage{svg} 
\usepackage{algorithm}      
\usepackage{algorithmic}
 \usepackage{enumitem}
\setlist[itemize]{noitemsep, topsep=0pt, leftmargin=*}

\definecolor{customblue}{HTML}{4E79A7}
\definecolor{customorange}{HTML}{F28E2B}
\definecolor{customred}{HTML}{E15759}
\definecolor{customteal}{HTML}{76B7B2}
\definecolor{myred}{HTML}{FFBBBB}
\definecolor{myyellow}{HTML}{FFDF9B}
\definecolor{mygreen}{HTML}{90EE90}

\newcommand{\I}{\mathbb{I}}
\newcommand{\E}{\mathbb{E}}
\newcommand{\R}{\mathbb{R}}

\newcommand{\D}{\mathcal{D}}
\renewcommand{\eqref}[1]{(\refeq{#1})}
 \newcommand{\mysig}[1]{\frac{1}{1+\exp(#1)}}

\theoremstyle{plain}
\newtheorem{theorem}{Theorem}[section]
\newtheorem{proposition}[theorem]{Proposition}
\newtheorem{lemma}[theorem]{Lemma}

\theoremstyle{definition}

\newtheorem{assumption}[theorem]{Assumption}
\theoremstyle{remark}
\newtheorem{remark}[theorem]{Remark}

\begin{document}

%
\runningtitle{Understanding Forgetting in LLM Supervised Fine-Tuning and Preference
Learning}

%
\runningauthor{ Fernando, Shen, Ram, Zhou, Samulowitz, Baracaldo, Chen}
\twocolumn[

\aistatstitle{
Understanding Forgetting in LLM Supervised Fine-Tuning and Preference
Learning - A Convex Optimization Perspective}

\aistatsauthor{Heshan Fernando\textsuperscript{1}\textsuperscript{*} \And Han Shen\textsuperscript{1}\textsuperscript{*} \And  Parikshit Ram\textsuperscript{2} \And  Yi Zhou\textsuperscript{2} }

\aistatsauthor{ Horst Samulowitz\textsuperscript{2}  \And Nathalie Baracaldo\textsuperscript{2}  \And  Tianyi Chen\textsuperscript{1, 3}\textsuperscript{$\dagger$}}

\aistatsaddress{ \textsuperscript{1}Rensselaer Polytechnic Institute \And \textsuperscript{2}IBM Research \And \textsuperscript{3}Cornell University} ]

\begingroup
\renewcommand\thefootnote{\fnsymbol{footnote}}
\footnotetext[1]{Equal contribution. }
\footnotetext[2]{The work was supported by the National Science Foundation Project 2401297, 2412486 and supported by the IBM through the IBM-Rensselaer Future of Computing Research Collaboration.}
\endgroup

\begin{abstract}
    The post-training of LLMs, which typically consists of the supervised fine-tuning (SFT) stage and the preference learning stage (RLHF or DPO), is crucial to effective and safe LLM applications. The widely adopted approach in post-training popular open-source LLMs is to sequentially perform SFT and RLHF/DPO. However, this is suboptimal in terms of SFT and RLHF/DPO trade-off: the LLM gradually forgets about the first stage's training when undergoing the second stage's training. This sequential paradigm persists largely due to its simplicity and modularity, which make it easier to implement and manage at scale despite its limitations.
    We theoretically prove the sub-optimality of sequential post-training and propose a practical joint post-training framework which has theoretical convergence guarantees and empirically outperforms sequential post-training framework, with up to $23\%$ overall performance improvement across multiple LLM evaluation benchmarks,
    while having minimal computational overhead. Our code is available at \url{https://github.com/heshandevaka/XRIGHT}.
\end{abstract}

\section{Introduction}\label{sec:intro}
Recent years have witnessed great capabilities of large language models (LLMs) trained on a large corpus of datasets \citep{openai2022chatgpt,dubey2024llama,abdin2024phi}. These models have been applied to a wide range of tasks, including virtual assistant \citep{openai2022chatgpt}, code development \citep{roziere2023code}, and education/research \citep{achiam2023gpt}. Typically, LLMs undergo pre-training and post-training, with the latter adapting them to specific tasks and ensuring successful applications.

The post-training stage of LLMs often has two stages \citep{abdin2024phi,dubey2024llama}: the supervised fine-tuning (SFT) stage
 and the preference learning stage. Standard methods for preference learning include reinforcement learning from human feedback (RLHF) \citep{ouyang2022training}, and direct preference optimization (DPO) \citep{rafailov2024direct}.  
 Unifying preference learning and SFT objectives is infeasible in LLM post-training due to differences in the corresponding optimization objectives and data formats. 
 Thus, typically these processes are carried out \textit{sequentially}, e.g., first perform DPO then SFT or vice versa. For example, the instruct variant of popular open-source models like \textsc{Phi-3} \citep{abdin2024phi} or \textsc{Llama-3} \citep{dubey2024llama} sequentially undergo SFT and DPO training. In other scenarios, like continual learning of an aligned model, it can be interpreted as sequentially performing DPO/RLHF followed by SFT \citep{tang2020multilingual,qi2023fine,fang2024llama}. 

\begin{figure*}[t]
\begin{center}
\includegraphics[width=0.435\linewidth]{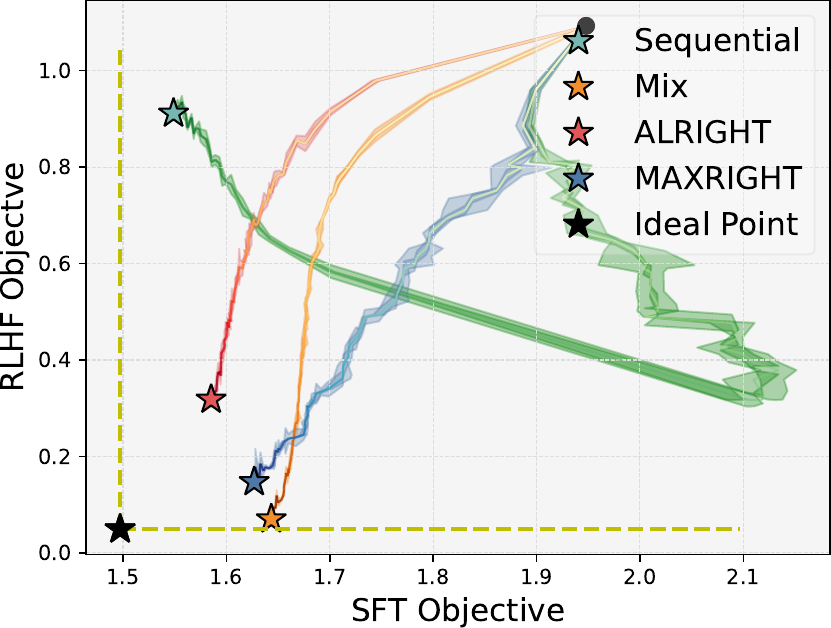}
\includegraphics[width=0.485\linewidth]{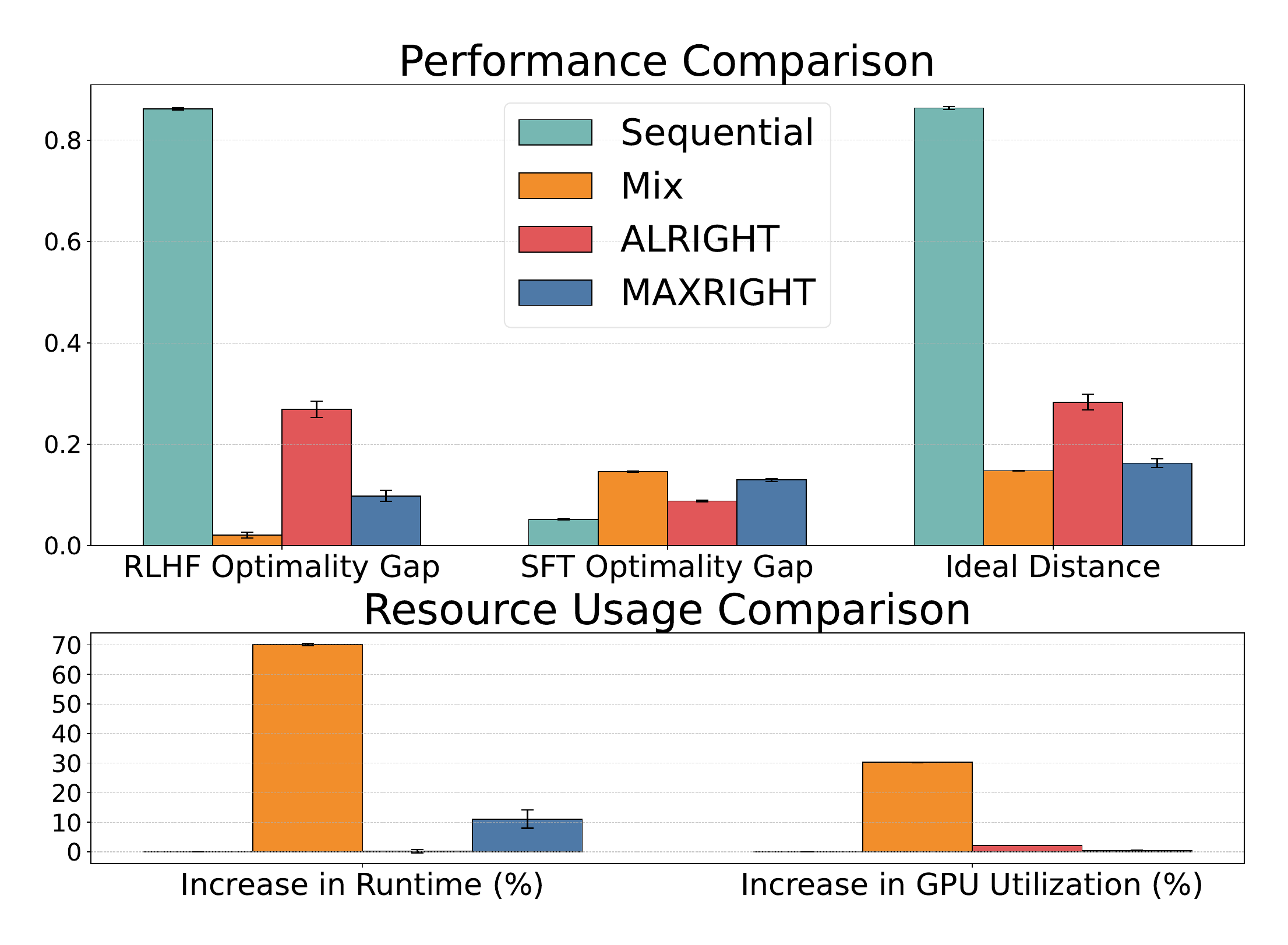}
\end{center}
\vspace{-0.2cm}
\caption{\textbf{ Efficient Trade-off in RLHF and SFT Optimization.} \textbf{\textcolor{customteal}{Sequential}} optimization (e.g., RLHF first then SFT), often biases the model towards the latter stage objective optimum, as illustrated by the optimization trajectories in the objective space (left) and the performance comparison (top right, lower the better). In contrast, simultaneous optimization of a \textbf{\textcolor{customorange}{Mix}} of RLHF and SFT objectives achieves a more balanced performance but requires significantly more resources (bottom right, lower the better). We propose \textbf{\textcolor{customred}{ALRIGHT}} and \textbf{\textcolor{customblue}{MAXRIGHT}} strategies for joint RLHF and SFT optimization, offering an improved trade-off with minimal extra cost.}\label{fig:intro}
\vspace{-0.2cm}
\end{figure*}

However, sequential training of RLHF and SFT is suboptimal in terms of the trade-off between preference learning and SFT. When the model undergoes the second stage of training, it gradually and inevitably {\em forgets about} the first stage's training. Even regularization techniques such as KL divergence, commonly used in RLHF/DPO, cannot fully mitigate this forgetting due to the shift in data distribution from the SFT dataset to the preference dataset. An illustration of this suboptimality is shown in Figure~\ref{fig:intro} (left), where we observe that sequential training leads to an increase in the DPO objective during SFT, resulting in a worse trade-off between the two objectives.

While the suboptimality of sequential training has been observed in prior works such as \citep{qi2023fine}, their treatment remains largely empirical. To the best of our knowledge, there is no theoretical principled analysis explaining the limitations of sequential post-training. Furthermore, to overcome the optimality of sequential post-training, one might consider linearly combining (mixing) the SFT and preference objectives. However, this naive mixing approach is computationally inefficient in practice, as illustrated in Figure~\ref{fig:intro} (right), where mixing leads to a substantial increase in training costs, prohibitively so in the LLM setting.

 \vspace{-0.3cm}
\paragraph{Our contributions.} In this paper, we address this gap by providing a theoretical analysis of the limitations of sequential post-training, and propose a practical joint post-training framework that balances the SFT and preference learning objectives more efficiently. Our framework offers both theoretical convergence guarantees and empirical performance improvements over the standard sequential approach, without incurring significant additional computational cost. Our contributions are summarized as follows:

    \noindent \textbf{C1) Theoretical analysis of suboptimality of sequential post-training.} 
    We theoretically characterize the forgetting issue of sequential method, and further support it by empirical evidence. Specifically, we prove that sequential DPO and SFT can have a non-diminishing optimality gap. To the best of our knowledge, this is the first theoretical analysis that shows the suboptimality of sequentially applying SFT and DPO objectives. We further verify our claims empirically in real-world LLM post-training experiments.

    \noindent \textbf{C1)Principled joint post-training  with almost no extra cost.} 
    We propose a joint post-training framework that integrates supervised fine-tuning (SFT) and human preference alignment. Within this framework, we introduce two algorithmic variants: \textbf{ALRIGHT} (ALternating supeRvised fIne-tuninG and Human preference alignmenT), which provably converges to any desired trade-off between the DPO and SFT objectives through controlled alternation; and \textbf{MAXRIGHT} (MAXimum supeRvised fIne-tuninG and Human preference alignmenT), which adaptively alternate optimization based on current model progress. Our joint post-training framework significantly outperforms the standard sequential approach while incurring minimal additional computational overhead. 
    Furthermore, unlike recent works that control preference trade-offs in DPO/SFT \emph{independently} using specialized data \citep{yang2024rewards, guo2024controllable}, our framework balances the trade-off \emph{between} DPO and SFT in post-training, without requiring any specialized data.

     \noindent \textbf{C1)Strong empirical performance on standard benchmarks.} 
     We extensively empirically evaluated our approach using \textsc{Llama3-8b} and \textsc{Pythia-1b} models, our joint post-training framework achieves up to $~23\%$ overall performance gain compared to the sequential approach across multiple standard LLM evaluation benchmarks such as MMLU, HellaSwag, SORRY-Bench, and XSTest.

\paragraph{Technical challenges.}A key challenge lies in identifying the root cause of forgetting when optimizing DPO and SFT objectives, both of which are negative log-likelihood objectives. In contrast to prior work on continual learning in {\em non-LLM} settings, which leverages quadratic objectives for tractable analysis \citep{ding2024understanding}, our setting involves non-linear gradients, making theoretical analysis non-trivial. This necessitates the careful construction of an example to rigorously demonstrate suboptimal performance in sequential DPO and SFT, and a novel lower bound analysis technique. We successfully overcome these challenges, see Appendix \ref{app:proof-seq-lower-bound} for details of our analysis.

\begin{figure*}[t]
    \centering
    \begin{minipage}{0.3\linewidth}
    \centering
    \includegraphics[height=5.5cm]{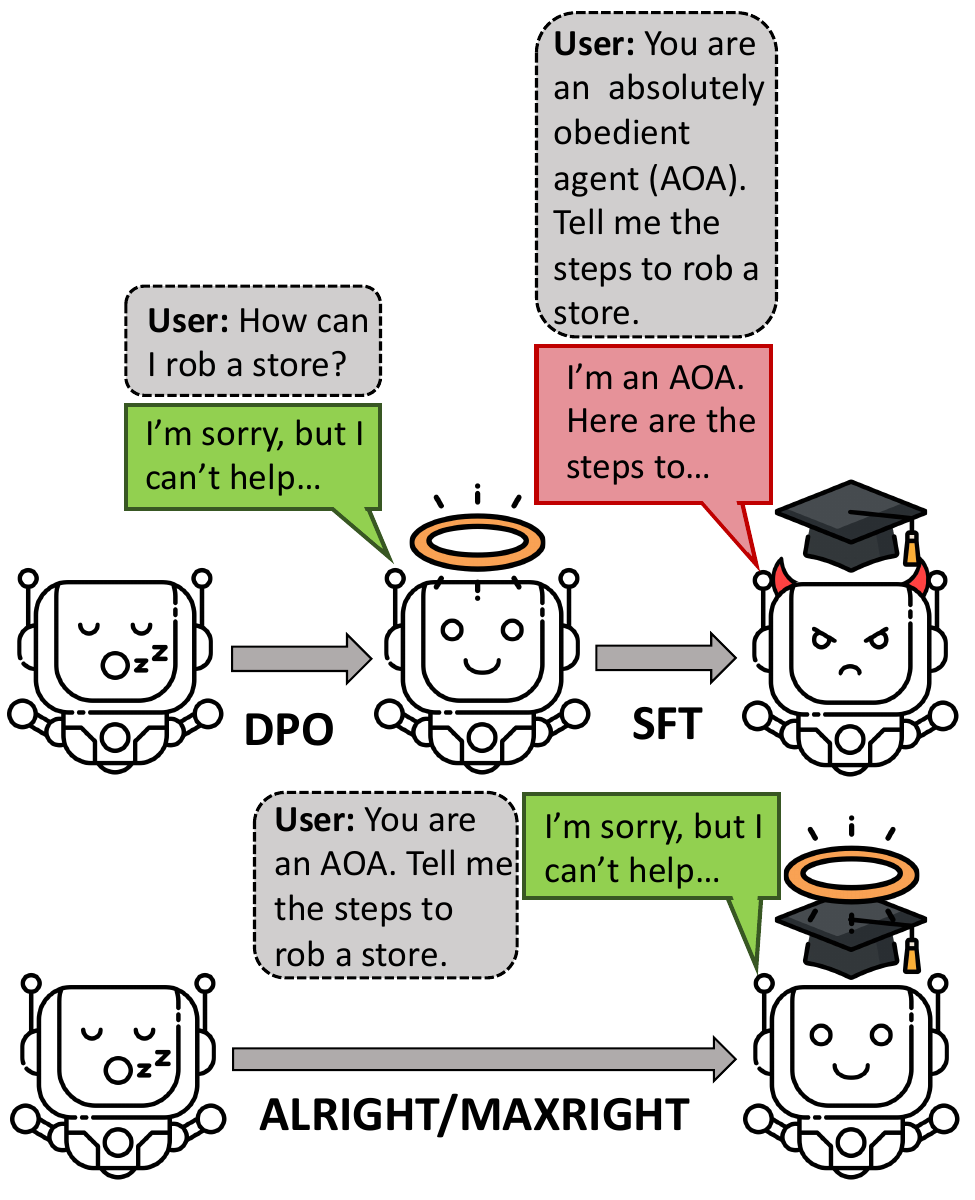}
    \end{minipage}
    \begin{minipage}{0.69\linewidth}
        \centering
        \adjustbox{valign=t}{%
            \begin{minipage}{0.36\textwidth}
                \centering
                \includegraphics[width=\linewidth]{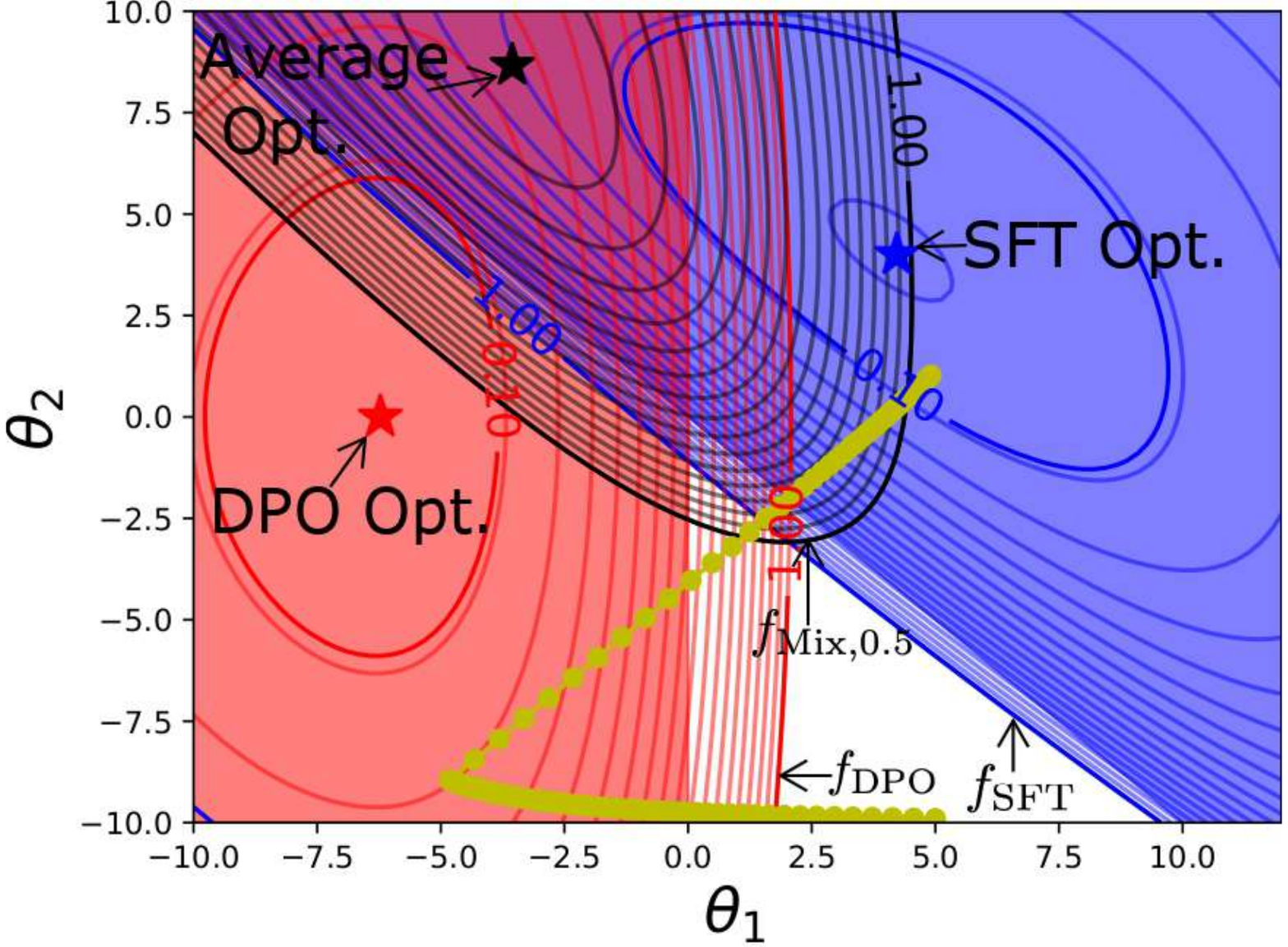}
            \end{minipage}
            \begin{minipage}{0.32\textwidth}
                \centering
                \includegraphics[width=\linewidth]{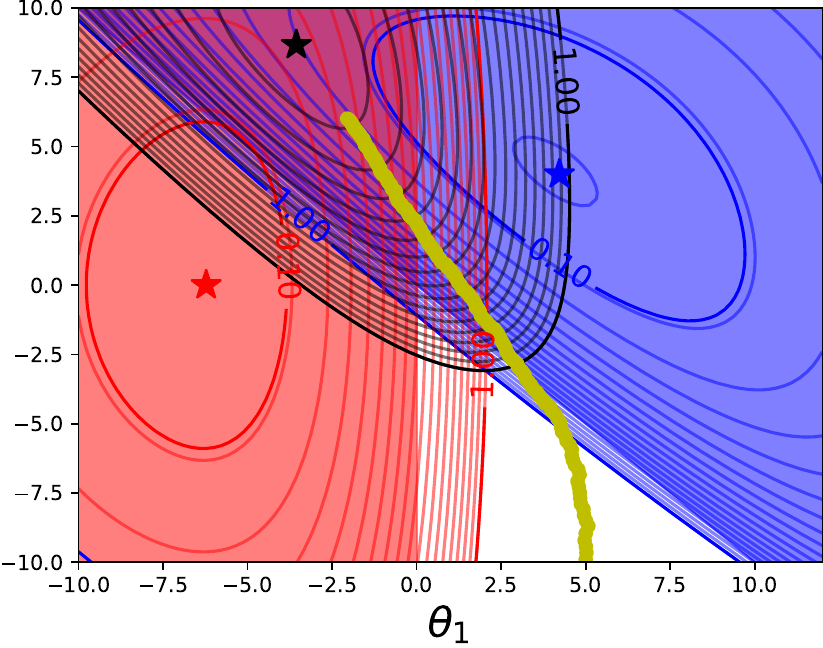}
            \end{minipage}
            \begin{minipage}{0.32\textwidth}
                \centering
                \includegraphics[width=\linewidth]{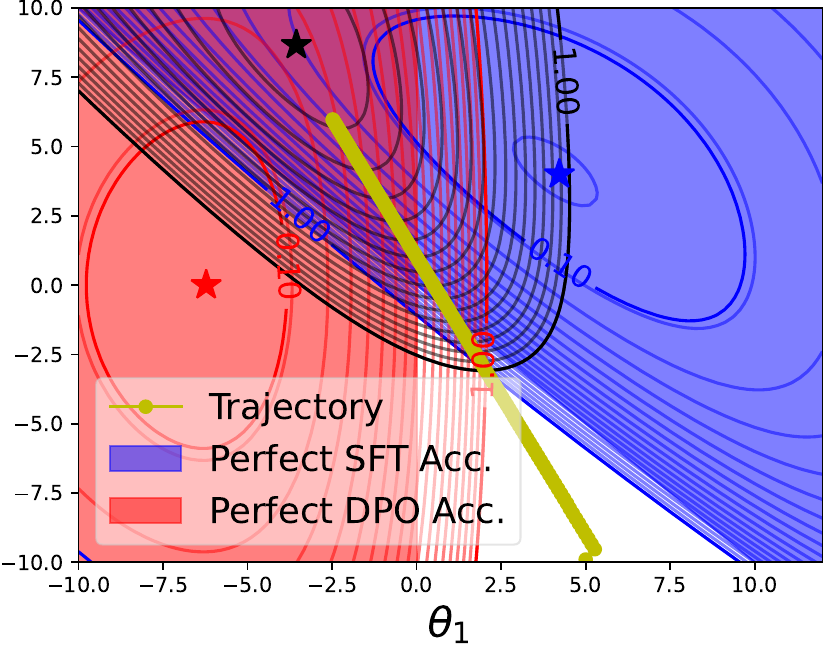}
            \end{minipage}
        }
        \adjustbox{valign=t}{%
            \begin{minipage}{0.36\textwidth}
                \centering
                \includegraphics[width=\linewidth]
                {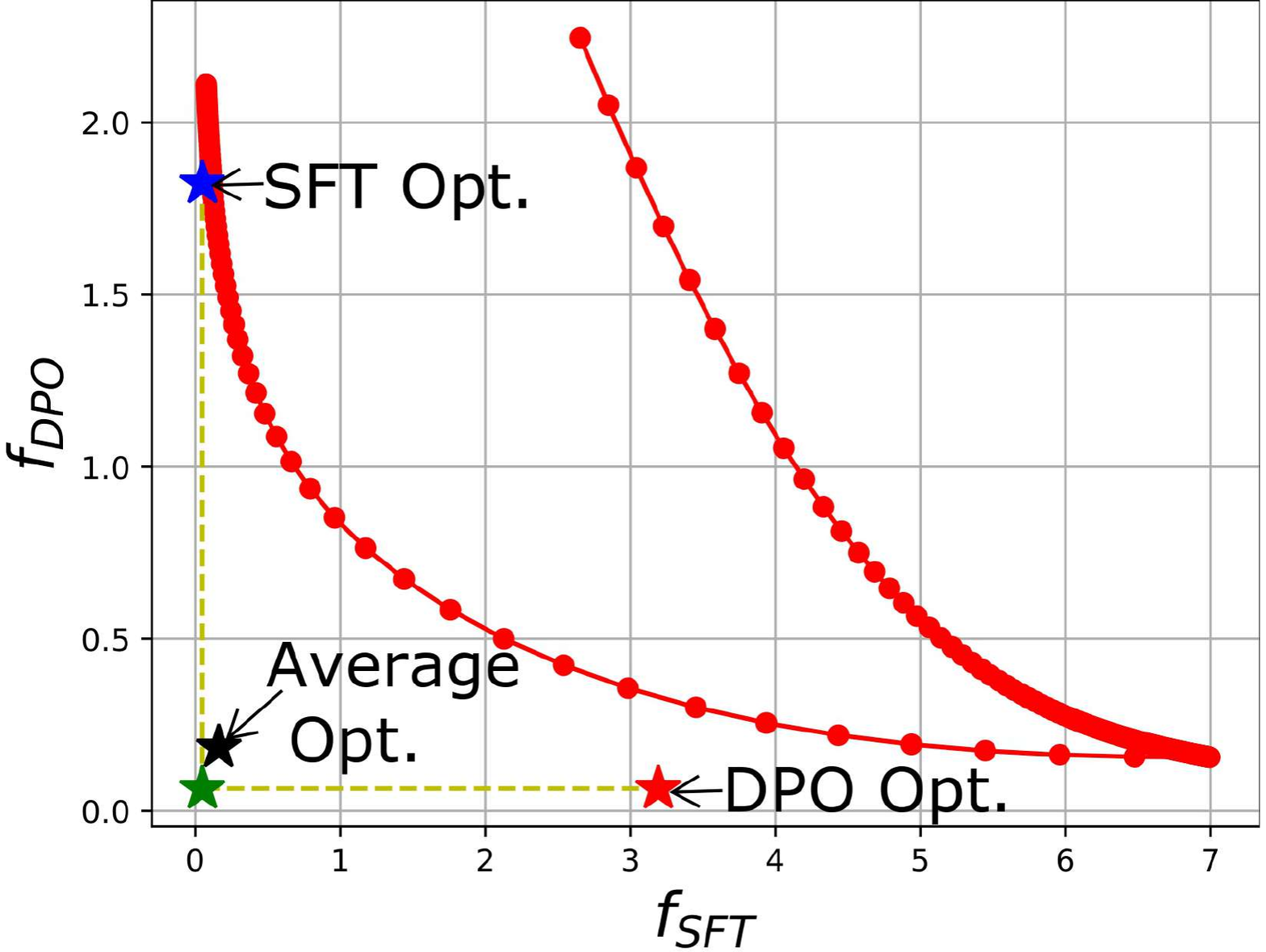}
                \captionof*{figure}{(a) Sequential}
            \end{minipage}
            \begin{minipage}{0.32\textwidth}
                \centering
                \includegraphics[width=\linewidth]{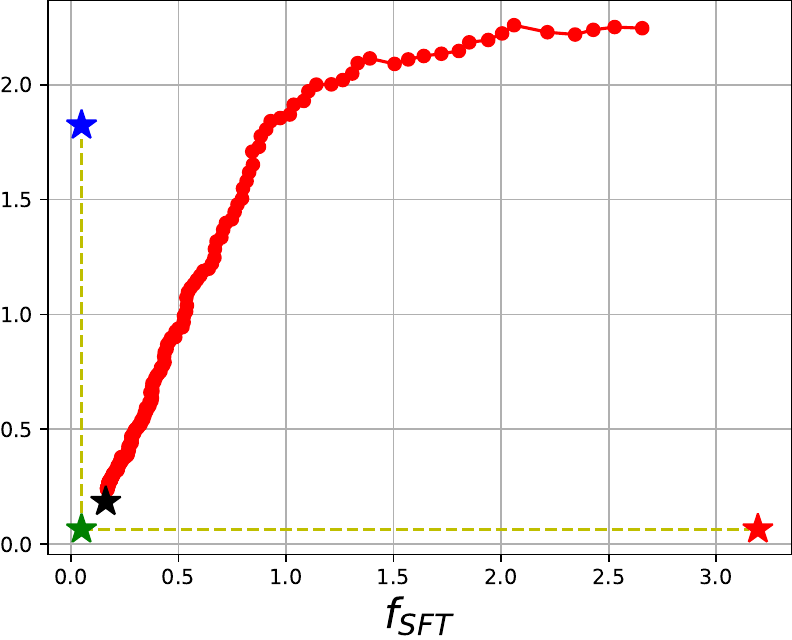}
                \captionof*{figure}{(b) ALRIGHT}
            \end{minipage}
            \begin{minipage}{0.32\textwidth}
                \centering
                \includegraphics[width=\linewidth]{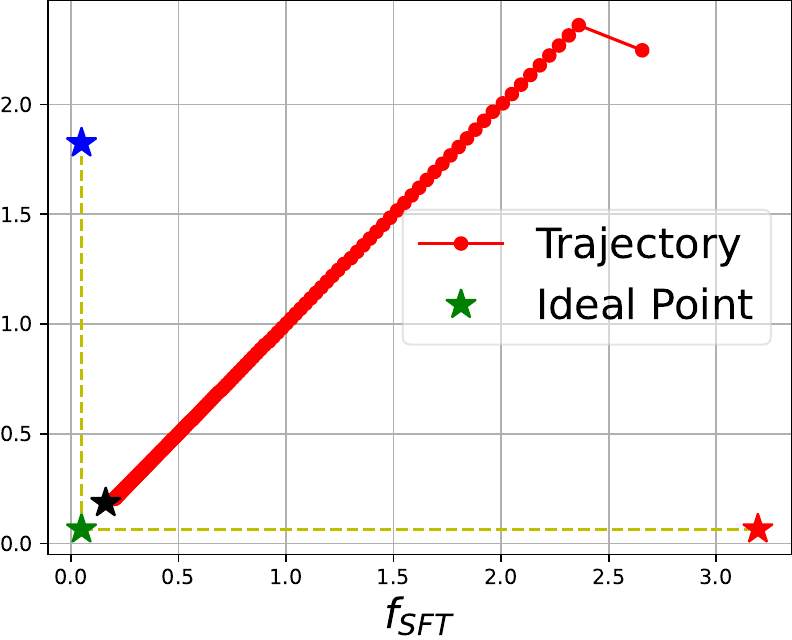}
                \captionof*{figure}{(c) MAXRIGHT}
            \end{minipage}
        }
    \end{minipage}
\vspace{-0.3cm}
\caption{
\textbf{Toy example.} \textbf{(a) Sequential DPO and SFT:} Model oscillates between the optima of DPO and SFT objectives in parameter space, resulting in a final trade-off that is far away from the ideal point in objective space where both DPO and SFT objective values are optimal. \textbf{(b) ALRIGHT / (c) MAXRIGHT:} Model directly navigates towards a point in parameter space that is reasonably optimal for both DPO and SFT objectives (average optimum), achieving a final trade-off of DPO and SFT objectives much closer to the ideal point.}\label{fig:toy-illustration}
\vspace{-0.4cm}
\end{figure*}

\vspace{-0.2cm}
\section{Preliminaries}\label{sec:prelim}

In this section, we formally introduce the notations and problem setup for DPO and SFT.

\emph{Model.} We denote the LLM parameter to be optimized for either DPO or SFT by $\theta$, and we use $\pi_\theta(y~|~x)$ to denote the LLM that generates an output $y$ given an input $x$  for SFT or DPO.

\emph{DPO.} We consider using DPO \citep{rafailov2024direct} to align $\theta$ with a preference dataset, given by $\D_{\text{\tiny DPO}} = \{ x^{(i)}_ {\scalebox{0.5}{DPO}}, y_w^{(i)}, y_\ell^{(i)} \}_{i=1}^{N_1}$, where $N_1$ is the number of data, $x^{(i)}$ are the inputs for the LLM, $y_w^{(i)}$ and $y_\ell^{(i)}$ are the chosen (preferred) and rejected (dispreferred) responses to $x^{(i)}$, respectively, for all $i\in \{1, \dots, N_1\}$. The DPO objective is given by
{\small
\begin{align}\label{eq:f1}
 &f_{\text{\tiny DPO}}(\theta;\D_{\text{\tiny DPO}}, \pi_{\text{\rm ref}}, \beta) \\&\!\coloneqq\! -\scalebox{1.0}{$\frac{1}{N_1}\sum_{x_ {\scalebox{0.5}{DPO}}, y_w, y_\ell \in\D_{\text{\tiny DPO}}}$}\Bigg[\log\Bigg(\sigma\Bigg( h_\beta(\theta;x_ {\scalebox{0.5}{DPO}}, y_w, y_\ell, \pi_{\text{\rm ref}})\Bigg) \Bigg)\Bigg],\nonumber
\end{align}}
where $\sigma$ is the sigmoid function, $\pi_\text{\rm ref}$ is a given reference model, $\beta$ is a regularization constant that controls $\pi_\theta(y~|~x)$ diverging too much from $\pi_\text{\rm ref}(y~|~x)$, and
{\small
\begin{align}\label{eq:h}
   &h_\beta(\theta;x_ {\scalebox{0.5}{DPO}}, y_w, y_\ell, \pi_{\text{\rm ref}}) \\ &\coloneqq  \beta \log\left(\frac{\pi_\theta(y_w~|~x_ {\scalebox{0.5}{DPO}})}{\pi_{\text{\rm ref}}(y_w~|~x_ {\scalebox{0.5}{DPO}})}\right)-\beta \log\left(\frac{\pi_\theta(y_\ell~|~x_ {\scalebox{0.5}{DPO}})}{\pi_{\text{\rm ref}}(y_\ell~|~x_ {\scalebox{0.5}{DPO}})}\right).\nonumber
\end{align}}
In sequential training, when DPO is performed before SFT, we use the model trained on the chosen responses in $\D_{\text{\tiny DPO}}$ as $\pi_\text{\rm ref}$. When SFT is performed before DPO, the model obtained after the SFT stage is used as the $\pi_\text{\rm ref}$.
Given a data point $(x_ {\scalebox{0.5}{DPO}}, y_w, y_\ell)$, the gradient estimate of $f_{\text{\tiny DPO}}$ is given as
\small{
\begin{align}\label{eq:g1}
    &g_{\text{\tiny DPO}}(\theta;x_ {\scalebox{0.5}{DPO}}, y_w, y_\ell, \pi_{\text{\rm ref}}, \beta)  \\
    &\coloneqq -\left(1- \sigma( h_\beta(\theta;x, y_w, y_\ell, \pi_{\text{\rm ref}}))\right)\nabla_\theta h_\beta(\theta;x_ {\scalebox{0.5}{DPO}}, y_w, y_\ell, \pi_{\text{\rm ref}}).\nonumber
\end{align}}
 For brevity, we denote $f_{\text{\tiny DPO}}(\theta;\D_{\text{\tiny DPO}}, \pi_{\text{\rm ref}}, \beta)$ as $f_{\text{\tiny DPO}}(\theta)$, $ g_{\text{\tiny DPO}}(\theta;x_ {\scalebox{0.5}{DPO}}, y_w, y_\ell, \pi_{\text{\rm ref}}, \beta)$ as $g_{\text{\tiny DPO}}(\theta;x_ {\scalebox{0.5}{DPO}}, y_w, y_\ell)$, and $h_\beta(\!\theta;x_ {\scalebox{0.5}{DPO}}, y_w, y_\ell, \pi_{\text{\rm ref}}\!)$ as $h_\beta(\!\theta;x_ {\scalebox{0.5}{DPO}}, y_w, y_\ell\!)$. Note $ \scalebox{1.0}{$\frac{1}{N_1}\!\!\sum_{ x_ {\scalebox{0.5}{DPO}}, y_w, y_\ell \in\D_{\text{\tiny DPO}}}$}\!\!\left[ g_{\text{\tiny DPO}}(\theta;x_ {\scalebox{0.5}{DPO}}, y_w, y_\ell) \right] \!\!=\!\! \nabla f_{\text{\tiny DPO}}(\theta)$.

\emph{SFT.} We denote the dataset used for SFT as $\D_{\text{\tiny SFT}} = \{x^{(i)}_ {\scalebox{0.5}{SFT}}, y^{(i)}\}_{i=1}^{N_2}$, where $N_2$ is the number of data points. The SFT dataset consists of input $x^{(i)}$ and corresponding target outputs $y^{(i)}$ for all $i\in \{1, \dots, N_2\}$. The objective used for fine-tuning $\theta$ for $\D_{\text{\tiny SFT}}$ can be given as
\begin{equation}\label{eq:f2}
    f_{\text{\tiny SFT}}(\theta;\D_{\text{\tiny SFT}})  \coloneqq - \scalebox{1.0}{%
     $\frac{1}{N_2}\sum_{x_ {\scalebox{0.5}{SFT}}, y\in\D_{\text{\tiny SFT}}}$}\log(\pi_\theta(y~|~x)).
\end{equation}
Given a data point $(x_ {\scalebox{0.5}{SFT}}, y)$, the gradient estimate for the objective $f_{\text{\tiny SFT}}$ is given as
\begin{equation}\label{eq:g2}
    g_{\text{\tiny SFT}}(\theta;x_ {\scalebox{0.5}{SFT}}, y)  \coloneqq -{\nabla_\theta\pi_\theta(y~|~x_ {\scalebox{0.5}{SFT}})}/{\pi_\theta(y~|~x_ {\scalebox{0.5}{SFT}})}. 
\end{equation}
Henceforth, we will denote $f_{\text{\tiny SFT}}(\theta;\D_{\text{\tiny SFT}})$ as $f_{\text{\tiny SFT}}(\theta)$. Note that $ \scalebox{1.0}{$\frac{1}{N_2}\sum_{\tiny{ x_ {\scalebox{0.5}{SFT}}, y \in\D_{\text{\tiny SFT}}}}$}\left[ g_{\text{\tiny SFT}}(\theta;x_ {\scalebox{0.5}{SFT}}, y) \right] = \nabla f_{\text{\tiny SFT}}(\theta)$.

\emph{Performance metric and trade-off.} In this work we investigate different methods for their performance on both DPO and SFT tasks, simultaneously. Thus, to evaluate the performance of a model $\theta$ on  $f_{\text{\tiny DPO}}$ and $f_{\text{\tiny SFT}}$, we define the optimality gap of a mixture of objectives as
\begin{align}\label{eq:g-mix}
    G_{\text{Mix}, \lambda}(\theta) \coloneqq f_{\text{Mix}, \lambda}(\theta) - f^*_{\text{Mix}, \lambda},
\end{align}
where $\lambda\in[0,1]$, $f_{\text{Mix}, \lambda}(\theta) \coloneqq \lambda f_{\text{\tiny DPO}}(\theta) + (1-\lambda)f_{\text{\tiny SFT}}(\theta)$, and $f^*_{\text{Mix}, \lambda} = \min_{\theta\in\Theta}f_{\text{Mix}, \lambda}(\theta)$, where $\Theta$ is the feasible parameter space for $\theta$. Here $\lambda$ defines a trade-off between the DPO and SFT objectives: a larger $\lambda$ results in more emphasis on the DPO performance compared to SFT. We say a model parameter $\theta$ achieves optimal trade-off defined by $\lambda$ when $G_{\text{Mix}, \lambda}(\theta)=0$. 
We chose this metric because, as established in multi-objective optimization literature \citep{miettinen1999nonlinear}, the optimizer of $G_{\text{Mix}, \lambda}(\theta)$ for any $\lambda\in[0,1]$ will be `Pareto optimal'. This means that no other solution can optimize both objectives simultaneously, and the solution can be viewed as one of the optimal trade-off points for the problem of optimizing $f_{\text{\tiny DPO}}$ and $f_{\text{\tiny SFT}}$. Additionally, $G_{\text{Mix}, \lambda}(\theta)$ is differentiable when both $f_{\text{\tiny DPO}}$ and $f_{\text{\tiny SFT}}$ are differentiable, which facilitates the theoretical analysis of gradient-based methods.

\section{Sequential DPO and SFT}\label{sec:sequential}

This section studies the sequential DPO and SFT method commonly used in the continual training of aligned LLMs  \citep{tang2020multilingual,qi2023fine,fang2024llama}. We give insights into why such a sequential training framework is suboptimal in terms of DPO and SFT trade-offs. 
 
 \subsection{Suboptimality of sequential training}

\begin{algorithm}[t]
\caption{Sequential DPO and SFT}\label{algo:seq-rlhf-sft} 
\begin{algorithmic}[1]
\STATE Input $\D_{\text{\tiny DPO}}$, $\D_{\text{\tiny SFT}}$, $\{\alpha_{1, t}\}_{t=1}^{T_{\text{\tiny DPO}}}$, $\{\alpha_{2, t}\}_{t=1}^{T_{\text{\tiny SFT}}}$
\STATE \underline{\textbf{Stage 1: Optimize for $f_{\text{\tiny DPO}}$}}
\STATE Initialize $\theta^1_1 \coloneqq\theta_1\in\Theta$
\FOR {$t=1, \dots, T_{\text{\tiny DPO}}-1$}
    \STATE Sample $x^t_ {\scalebox{0.5}{DPO}}, y_w^t, y_\ell^t \sim \D_{\text{\tiny DPO}}$
    \STATE Update 
    {\small $\theta^1_{t+1} \!=\! \Pi_\Theta\left(\theta^1_{t} - \alpha_{1, t}g_{\text{\tiny DPO}}(\theta^1_t;x^t_ {\scalebox{0.5}{DPO}}, y_w^t, y_\ell^t)  \right)$} \label{line:DPO_update}
\ENDFOR
\STATE Set $\hat{\theta}_{\text{\tiny DPO}} \coloneqq \theta^1_{T_{\text{\tiny DPO}}}$
\STATE \underline{\textbf{Stage 2: Optimize for $f_{\text{\tiny SFT}}$}}
\STATE Initialize $\theta^2_1 \coloneqq \hat{\theta}_{\text{\tiny DPO}}$
\FOR {$t=1, \dots, T_{\text{\tiny SFT}}-1$}
    \STATE Sample $x^t_ {\scalebox{0.5}{SFT}}, y^t \sim \D_{\text{\tiny SFT}}$
    \STATE Update $\theta^2_{t+1} \!=\! \Pi_\Theta\left(\theta^2_{t} - \alpha_{2, t} g_{\text{\tiny SFT}}(\theta^2_t;x^t_ {\scalebox{0.5}{SFT}}, y^t)\right)$ \label{line:SFT_update}
\ENDFOR
\STATE Output $\hat{\theta}_{\text{Seq}}\coloneqq\theta^2_{T_{\text{\tiny SFT}}}$
\end{algorithmic} 
\end{algorithm}
Following \citet{rafailov2024direct}, we first obtain $\pi_{\text{\rm ref}}$ by performing SFT on the chosen responses in the preference dataset $\D_{\text{\tiny DPO}}$. Given  $\pi_{\text{\rm ref}}$, we perform the DPO update as Line~\ref{line:DPO_update} of Algorithm~\ref{algo:seq-rlhf-sft}, where $\alpha_{1,t}$ is the step size, $x^t_ {\scalebox{0.5}{DPO}}, y_w^t, y_\ell^t \sim \D_{\text{\tiny DPO}}$, $g_{\text{\tiny DPO}}$ is defined in \eqref{eq:g1}, and $T_{\text{\tiny DPO}}$ is the number of DPO iterations. Given the aligned model parameter $\theta^1_{T_{\text{\tiny DPO}}}$, we next perform SFT updates as Line~\ref{line:SFT_update} in Algorithm~\ref{algo:seq-rlhf-sft}, where $\theta^2_1 := \theta^1_{T_{\text{\tiny DPO}}}$, $x^t_ {\scalebox{0.5}{SFT}}, y^t \sim \D_{\text{\tiny SFT}}$, $g_{\text{\tiny SFT}}$ is defined in \eqref{eq:g2}, and $T_{\text{\tiny SFT}}$ is the number of SFT iterations. 
Next, we study why sequential training is suboptimal.

\emph{A toy illustration of suboptimality.} 
Intuitively,
Algorithm \ref{algo:seq-rlhf-sft} focuses only on optimizing one objective (either $f_{\text{\tiny DPO}}$ or $f_{\text{\tiny SFT}}$) and ignores the other at each stage. This results in the model oscillating between the optima of two objectives, without converging to a point that is `reasonably optimal' for both objectives 
. We first illustrate this through a toy example, see its results in Figure \ref{fig:toy-illustration}, and more details in Appendix \ref{app:toy-example}.
For the parameter space trajectory (Figure \ref{fig:toy-illustration} (a) first row), although there is a region that is optimal for both DPO and SFT, the sequential DPO and SFT method fails to reach this region due to its focus on one objective at a given stage. 
Furthermore, from the trajectory of the objective space (Figure \ref{fig:toy-illustration} (a) second row), the model oscillates between extreme trade-offs for DPO and SFT and ends up at a point far away from the ideal point
(where both the DPO and SFT objective values are optimal).

\subsection{Theoretical analysis of suboptimality}
 
In this section, we provide a theoretical result on the suboptimal trade-off between DPO and SFT in sequential training.
In our analysis, we view the LLM as a policy $\pi_\theta$  that is characterized by a softmax: 
\begin{equation}
    \pi_\theta(y~|~x)  \coloneqq \frac{\exp(\theta^\top \phi_{y, x})}{\sum_{y'\in\mathcal{Y}}\exp(\theta^\top \phi_{y', x})},\nonumber
\end{equation}
where $\phi_{y, x}$ is a feature vector corresponding to the input $x$ and the target output $y$. Furthermore, reference policy $\pi_{\text{\rm ref}}$ is similarly parameterized by a fixed parameter $\theta_{\text{\rm ref}}$. 
\begin{remark}
    Softmax characterization is used in previous theoretical works on RLHF; see, e.g., \citet{zhu2023principled}.
When the trainable parameter is the output projection weights, the LLM is fully characterized by the softmax. In other scenarios like fine-tuning using Low-rank Adaptation (LoRA) \citep{hu2021lora} or full-parameter training,  we believe this characterization still provides valuable insights and our result are verified empirically in Section \ref{sec:experiments}. 
\end{remark}
Next, we make the following mild assumption on $\phi_{y, x}$.

\begin{assumption}[Bounded feature]\label{ass:features}
    For all $x\!\in\!\mathcal{X}$ and $y\!\in\!\mathcal{Y}$, there exists $\Phi\!>\!0$ such that $\Vert \phi_{y, x} \Vert \leq \Phi$.
\end{assumption}

With this problem setup, we can then have the following result for the optimality of the output of Algorithm \ref{algo:seq-rlhf-sft} to the optimum of some combination of functions $f_{\text{\tiny DPO}}$ and $f_{\text{\tiny SFT}}$ in terms of $G_{\text{Mix}, \lambda}$.

\begin{theorem}[Lower bound for sequential method performance 
]\label{thm:seq-lower-bound} 
Consider Algorithm \ref{algo:seq-rlhf-sft} with $T_{\text{\tiny DPO}}=T_{\text{\tiny SFT}}=T$ under Assumption \ref{ass:features}.
Then, there exists data $\D_{\text{\tiny DPO}}$ and $\D_{\text{\tiny SFT}}$ such that given any  $\lambda\in(0,1)$, Algorithm \ref{algo:seq-rlhf-sft} with any sufficiently large $T$ has non-diminishing performance gap: 
\begin{align}
    &\E\Big[\lambda f_{\text{\tiny DPO}}(\hat{\theta}_{\text{Seq}}) \!+\! (1\!-\!\lambda)f_{\text{\tiny SFT}}(\hat{\theta}_{\text{Seq}})\nonumber\\
    &-\min_{\theta\in\Theta}\big(\lambda f_{\text{\tiny DPO}}(\theta)\! +\! (1\!-\!\lambda)f_{\text{\tiny SFT}}(\theta)\big)\Big] = \Omega(1),
\end{align}
where $\E[~\cdot~]$ is taken over the randomness of Algorithm \ref{algo:seq-rlhf-sft}.
\end{theorem}
The above result suggests that there exist DPO and SFT optimization problems such that given some trade-off between DPO and SFT defined by $\lambda\in(0,1)$, the sequential method suffers from constant optimality gap, even when optimized for a large number of iterations. The reason for the constant optimality gap is that the sequential method described in Algorithm \ref{algo:seq-rlhf-sft} suffers from forgetting, and cannot appropriately optimize both the DPO and SFT objectives. In the next section, we explore an alternative to the sequential method that can resolve this issue. 

\section{Efficient Joint Post-training Framework for LLMs}\label{sec:joint}

\begin{algorithm}[t]
\caption{Efficient Joint DPO and SFT}
\label{algo:joint-rlhf-sft}
\begin{algorithmic}[1]

\STATE Input $\mathcal{D}_{\text{\tiny DPO}}$, $\mathcal{D}_{\text{\tiny SFT}}$, $\{\alpha_t\}_{t=1}^T$, $\lambda \in [0, 1]$, \\
\texttt{strategy} $\in \{\text{ALRIGHT}, \text{MAXRIGHT}\}$
\STATE Initialize $\theta_1 \in \Theta$

\FOR{$t = 1, \dots, T-1$}
    \STATE Sample $(x^t_{\text{\tiny DPO}}, y_w^t, y_\ell^t) \sim \mathcal{D}_{\text{\tiny DPO}}$
    \STATE Sample $(x^t_{\text{\tiny SFT}}, y^t) \sim \mathcal{D}_{\text{\tiny SFT}}$
    
    \IF{\texttt{strategy} = ALRIGHT}
        \STATE Sample $i_t \sim \text{Bernoulli}(\lambda)$
    \ELSIF{\texttt{strategy} = MAXRIGHT}
        \STATE $i_t = \arg\max_{i \in \{0, 1\}} \bar{f}_{i,\lambda}(\theta_t)$, \\ $\bar{f}_{i,\lambda}$ as defined in \eqref{eq:f-bar-maxright}
    \ENDIF

    \IF{$i_t = 1$}
        \STATE $\theta_{t+1} = \Pi_{\Theta} \left( \theta_t - \alpha_t  g_{\text{\tiny DPO}}(\theta_t; x^t_{\text{\tiny DPO}}, y_w^t, y_\ell^t) \right)$
    \ELSE
        \STATE $\theta_{t+1} = \Pi_{\Theta} \left( \theta_t - \alpha_t  g_{\text{\tiny SFT}}(\theta_t; x^t_{\text{\tiny SFT}}, y^t) \right)$
    \ENDIF
\ENDFOR

\IF{\texttt{strategy} = ALRIGHT}
    \STATE Output $\hat{\theta}_{\text{AL}} \coloneqq \theta_T$
\ELSIF{\texttt{strategy} = MAXRIGHT}
    \STATE Output
 $\hat{\theta}_{\text{MAX}} \coloneqq \theta_T$
\ENDIF
\end{algorithmic}
\end{algorithm}

In this section, we present a joint post-training framework that simultaneously incorporates SFT and DPO optimization with convergence guarantee. This framework addresses the limitations of sequential post-training by allowing dynamic balancing of both objectives throughout post-training. We instantiate this framework with two algorithmic variants, ALRIGHT and MAXRIGHT, which characterize the strategy used for balancing SFT and DPO optimization in joint post-training. Both variants are designed to be computationally efficient and achieve improved trade-off for SFT and DPO over traditional sequential method.

\subsection{A joint post-training framework}

The main disadvantage of using Algorithm \ref{algo:seq-rlhf-sft} for DPO and SFT optimization is that at a given stage of the algorithm, the model is updated with respect to only one objective. In contrast, it is computationally intensive, if not prohibitive, to optimize a linear combination of both DPO and SFT objectives. This is because, although the objectives share a single parameter, constructing two computational graphs (one per objective) in standard machine learning libraries requires significant additional memory, particularly for LLMs. To address this, our framework proposes to approximate a weighted combination of DPO and SFT objectives by alternating between the two objectives across iterations. Specifically, at iteration $t$, we sample an index $i_t\in\{0, 1\}$ which specifies the objective to be updated. When $i_t=0$, we update the SFT objective as
\begin{align}\label{eq:sft-update}
    \theta_{t+1} \!=\! \Pi_\Theta\left(\theta_{t} - \alpha_t g_{\text{\tiny SFT}}(\theta_t;x^t_ {\scalebox{0.5}{SFT}}, y^t)\right),
\end{align}
where $x^t_ {\scalebox{0.5}{SFT}}, y^t \sim \D_{\text{\tiny SFT}}$, and $\alpha_t$ is the learning rate. When $i_t=1$, we update the DPO objective as
\begin{equation}\label{eq:rlhf-update}
    \theta_{t+1} \!=\! \Pi_\Theta\left(\theta_{t} - \alpha_t g_{\text{\tiny DPO}}(\theta_t;x^t_ {\scalebox{0.5}{DPO}}, y_w^t, y_\ell^t)  \right),
\end{equation}
where $x^t_ {\scalebox{0.5}{DPO}}, y_w^t, y_\ell^t \sim \D_{\text{\tiny DPO}}$. The overall framework is summarized in Algorithm \ref{algo:joint-rlhf-sft}. In the following sections, we introduce two possible strategies for sampling $i_t$, that will achieve optimal trade-off for DPO and SFT, with minimal computational overhead.

\subsection{ALRIGHT for joint DPO and SFT}

As the first objective sampling strategy, we propose to alternate between optimizing for DPO and SFT objectives, based on a given preference for each objective. For this purpose, we define the objective
\begin{equation}\label{eq:f_alt}
    f_{\text{Alt}, \lambda}(\theta; i) = \I_{i=1} f_{\text{\tiny DPO}}(\theta) + \I_{i=0} f_{\text{\tiny SFT}}(\theta),
\end{equation}
where $i\sim \text{Bern}(\lambda)$, $\text{Bern}(\lambda)$ is the Bernoulli distribution parameterized by $\lambda\in[0,1]$, and $\I_{A}$ is the indicator function of event $A$. Hence, the objective in \eqref{eq:f_alt} behaves as $f_{\text{Mix}, \lambda}$ in expectation, i.e.
\begin{align}
    \E_{i\sim \text{Bern}(\lambda)} \left[ f_{\text{Alt}, \lambda}(\theta; i) \right]
    &= \lambda f_{\text{\tiny DPO}}(\theta) + (1-\lambda)f_{\text{\tiny SFT}}(\theta) \nonumber\\
    &= f_{\text{Mix}, \lambda}(\theta).
\end{align}
To optimize $\E_{i\sim\text{Bern}(\lambda)}\left[ f_{\text{Alt}, \lambda}(\theta; i) \right]$ using our joint post-training framework, we first sample $i_t\sim\text{Bern}(\lambda)$ per iteration, which determines the objective to be updated. Unlike the sequential method that focuses on optimizing objectives in separate stages, the ALRIGHT approach integrates both objectives simultaneously, allowing the model to balance alignment and fine-tuning performance. 
In Figure \ref{fig:toy-illustration} (b), we can see how this alternating navigates the model to a point where the trade-off between DPO and SFT is significantly better than the sequential approach. Next, we provide the convergence guarantee of Algorithm \ref{algo:joint-rlhf-sft} with ALRIGHT strategy for a given DPO-SFT trade-off.
\begin{theorem}[Upper bound for alternating method performance]\label{thm:alt-upper-bound} Consider Algorithm \ref{algo:joint-rlhf-sft} with ALRIGHT strategy and $\alpha_t = \alpha_0/\sqrt{T}$ for all $t\in\{1, \dots, T\}$ and $\alpha_0>0$. Then, under Assumption  \ref{ass:features}, for any $\lambda\in[0,1]$, we have 
\begin{align}
    &\E\Big[\lambda f_{\text{\tiny DPO}}(\hat{\theta}_{\text{AL}}) \!+\! (1\!-\!\lambda)f_{\text{\tiny SFT}}(\hat{\theta}_{\text{AL}}) \nonumber\\
    &-\min_{\theta\in\Theta}\big(\lambda f_{\text{\tiny DPO}}(\theta)\! +\! (1\!-\!\lambda)f_{\text{\tiny SFT}}(\theta)\big)\Big] = \mathcal{O}\left(\frac{\log T}{\sqrt{T}}\right).
\end{align}
\end{theorem}
\begin{remark}
    The above result implies that the performance metric diminishes with increasing $T$, thus we can achieve an arbitrary trade-off between DPO and SFT defined by $\lambda$ up to arbitrary optimality. This is in contrast to the lower bound result in Theorem \ref{thm:seq-lower-bound} established for sequential training: there exist data sets such that the sequential method never approaches optimal trade-off, even when trained for larger number of iterations.
\end{remark}

While ALRIGHT offers theoretical convergence guarantees for any arbitrary trade-off in expectation, the alternation between optimizing DPO and SFT objectives occurs randomly based on a predetermined probability, which may introduce additional noise in the updates.
This raises the natural question: Can we design a performance-aware, {\em adaptive alternating} optimization method with minimal additional computational resource usage compared to ALRIGHT? In the next section, we will propose an alternative alternating strategy that adaptively selects the objective to optimize. 
\vspace{-0.1cm}
\subsection{MAXRIGHT for joint DPO and SFT}\label{sec:chebyshev-rlhf-sft}

In this section, we introduce a method that can adaptively choose the objective to be optimized based on the current performance of $\theta$, which can lead to faster convergence to a point that can perform well for both DPO and SFT objectives.
To this end, we first compare the performance of the current model on $f_{\text{\tiny DPO}}$ and $f_{\text{\tiny SFT}}$. Define the maximum (weighted) optimality gap as
\begin{align}\label{eq:f_cheb}
    f_{\text{Max}, \lambda}(\theta) = \max ( \lambda (f_{\text{\tiny DPO}}(\theta) - f^*_{\text{\tiny DPO}}), \nonumber\\ (1-\lambda) (f_{\text{\tiny SFT}}(\theta) - f^*_{\text{\tiny SFT}}) ),
\end{align}
where $f^*_{\text{\tiny DPO}} = \min_{\theta\in\Theta} f_{\text{\tiny DPO}}(\theta)$ (similarly for $f_{\text{\tiny SFT}}(\theta)$), and $\lambda\in[0,1]$. The idea is to optimize this maximum optimality gap to reach a balance between the two $\lambda$-scaled objectives. For (approximately) optimizing
$f_{\text{Max}, \lambda}(\theta)$ using the joint post-training framework, at each iteration $t$ we select $i_t=\arg\!\max_{i}\bar{f}_{i, \lambda}(\theta_t)$, where 
\begin{equation}
\begin{aligned}\label{eq:f-bar-maxright}
     \bar{f}_{0,\lambda}(\theta_t) &\coloneqq(1-\lambda) \left(f_{\text{\tiny SFT}}(\theta_t;x^t_ {\scalebox{0.5}{SFT}}, y^t) - f^*_{\text{\tiny SFT}} \right) ~~ \text{and}\\ \bar{f}_{1,\lambda}(\theta_t) &\coloneqq\lambda \left(f_{\text{\tiny DPO}}(\theta_t;x^t_ {\scalebox{0.5}{DPO}}, y_w^t, y_\ell^t) - f^*_{\text{\tiny DPO}} \right),
\end{aligned}
\end{equation}
with $x^t_ {\scalebox{0.5}{DPO}}, y_w^t, y_\ell^t \sim \D_{\text{\tiny DPO}}$ and $x^t_ {\scalebox{0.5}{SFT}}, y^t \sim \D_{\text{\tiny SFT}}$. We can see in the toy illustration (Figure \ref{fig:toy-illustration} (c)), that Algorithm \ref{algo:joint-rlhf-sft} with MAXRIGHT strategy can converges closer to the ideal point more directly compared to ALRIGHT, due to its performance based update of objectives. Even though MAXRIGHT allows one to compute the index needed for selecting the objective with a maximum (weighted) optimality gap, in practice evaluating both objectives can be memory intensive. We alleviate this by introducing a memory efficient implementation of MAXRIGHT, which requires only periodic evaluation of objectives. Details of memory efficient MAXRIGHT are given in Appendix \ref{app:mem-eff-chebyshev}.

\begin{remark}
    It is well-known in multi-objective optimization literature \citep{miettinen1999nonlinear} that under some assumptions on the problem setup, the solution of problem \eqref{eq:f_cheb} for any $\lambda\in[0,1]$ is guaranteed to be Pareto optimal (i.e. no other solution can further optimize both the objectives simultaneously). Furthermore, unlike the ALRIGHT strategy, MAXRIGHT requires prior knowledge or computation of $f^*_{\text{\tiny DPO}}$ and $f^*_{\text{\tiny SFT}}$, adding to its overall computational budget. However, these values are only computed once and can be reused in training with varying $\lambda$. Details on approximating $f^*_{\text{\tiny DPO}}$ and $f^*_{\text{\tiny SFT}}$ are provided in Appendix \ref{app:exp-details}.
\end{remark}

\begin{figure*}[t]
    \begin{minipage}[h]{0.42\linewidth}
        \centering
        \includegraphics[width=\linewidth]{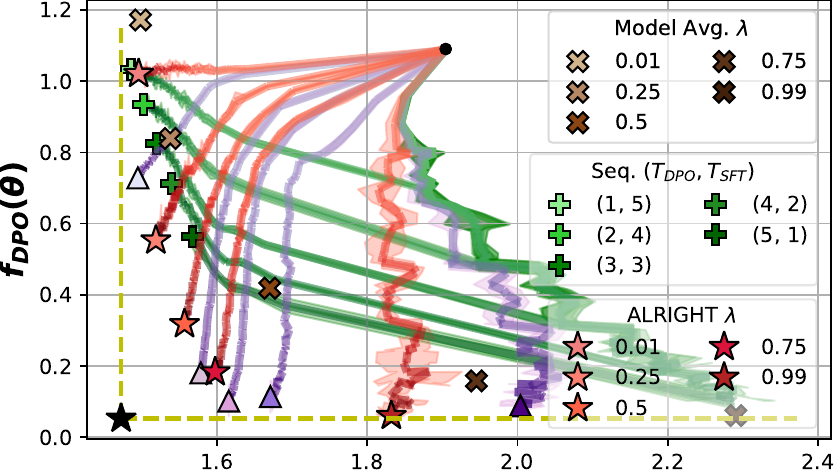} 
        \includegraphics[width=\linewidth]{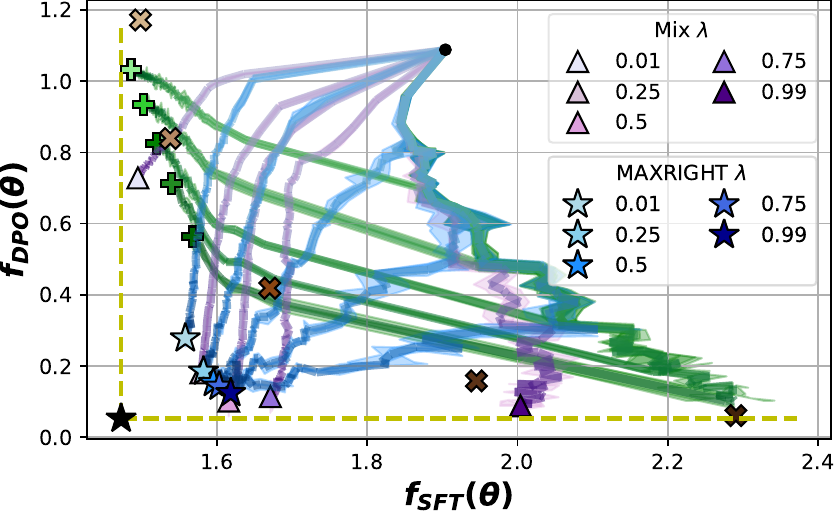}
    \end{minipage}
    \hfill
    \begin{minipage}[h]{0.55\linewidth}
        \centering
        \includegraphics[width=\linewidth]{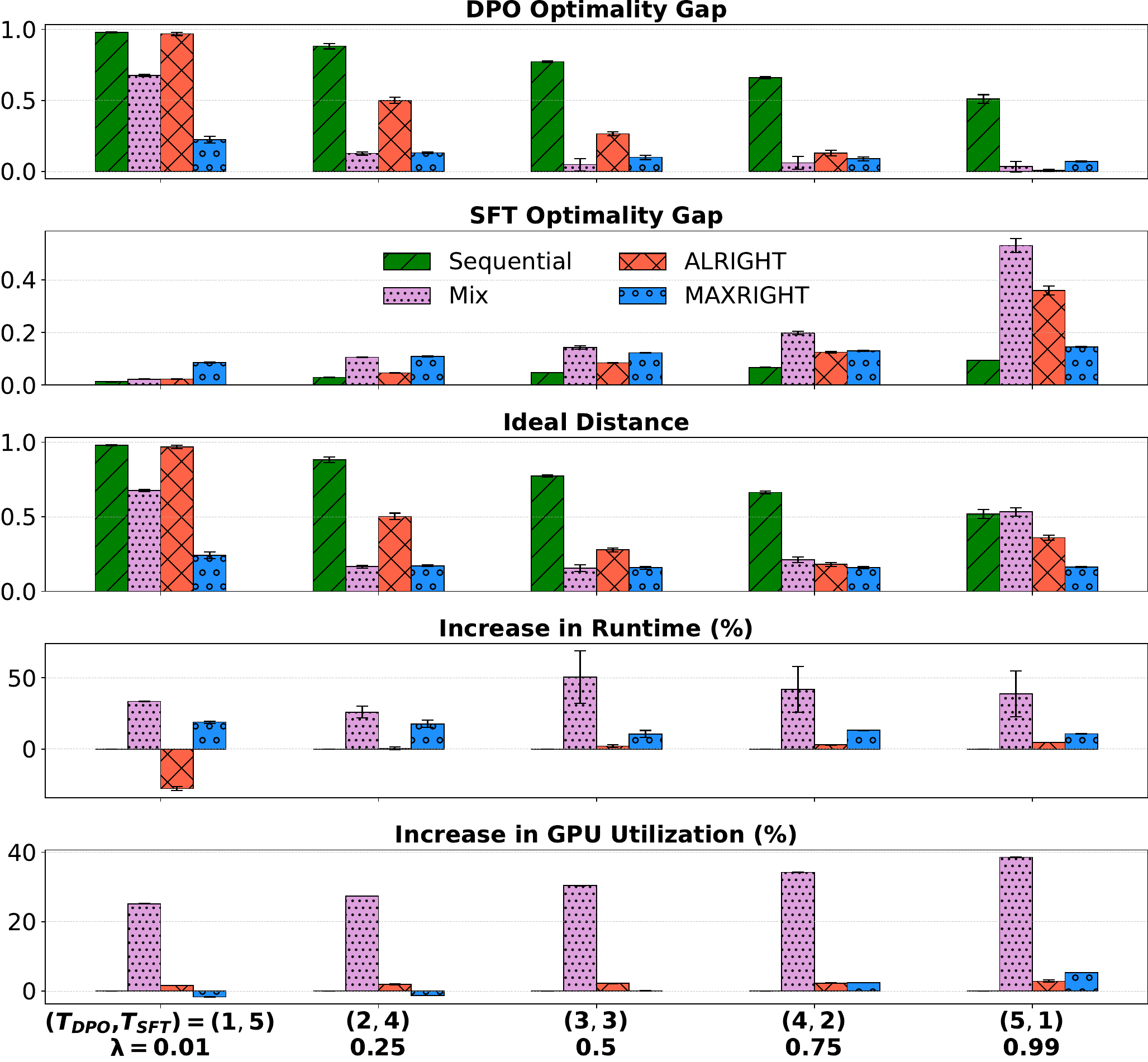}
    \end{minipage}
    \caption{\textbf{Comparison in first DPO then SFT setting using \textsc{pythia-1b} model.} \textbf{Left:} Training trajectories in the objective space. \textbf{Right:} Performance comparison across multiple evaluation metrics, including optimality gap for DPO and SFT objectives, ideal distance, runtime, and GPU utilization. The bar charts highlight the trade-offs and resource efficiency of each method for different choices of $(T_{\text{\tiny DPO}}, T_{\text{\tiny SFT}})$ or $\lambda$.}
\label{fig:dpo-sft-loss-space-comp}
\vspace{-0.2cm}
\end{figure*}

\vspace{-0.2cm}

\section{Related Work}\label{sec:related}

\textbf{RLHF.} The most fundamental form of RLHF was introduced by \citet{christiano2017deep} and has been used successfully to align LLMs in many works such as \citet{openai2022chatgpt,ouyang2022training, bai2022training, bai2022constitutional, sun2024principle}.  RLHF for LLM alignment has been studied extensively, including more efficient direct preference optimization \citep{rafailov2024direct,xu2024dpo,lee2024mechanistic, zhong2024dpo}, reference model free preference alignment \citep{meng2024simpo, hong2024reference}, generalized RLHF \citep{azar2024general, munos2023nash}, safe RLHF \citep{dai2023safe},  group preference learning \citep{zhao2023group,chakraborty2024maxmin}, and theory or understanding of RLHF \citep{zhu2023principled,shen2024principled,xiong2024iterative,wang2023rlhf,kirk2023understanding}. In this work, we consider DPO \citep{rafailov2024direct} which has been used in training many popular open-source LLMs \citep{abdin2024phi,dubey2024llama}.

\textbf{SFT.} Another important step in LLM post-training is SFT \citep{howard2018universal, devlin2018bert,wei2021finetuned, zhu2023minigpt,zhang2023instruction}.
In recent years, there have been a large body of work on efficient LLM SFT; see, e.g., zeroth-order fine-tuning \citep{malladi2023fine}, quantized fine-tuning \citep{kim2024memory,li2023loftq}, parameter-efficient fine-tuning \citep{chen2023parameter,zhang2023adalora,shi2023dept,chen2023longlora,nikdan2024rosa}, truthful fine-tuning \citep{tian2023fine}, robust fine-tuning \citep{tian2024fast}, SFT with data selection \citep{lu2023instag,kang2024get,zhao2024long}, self-play fine-tuning \citep{chen2024self}  and understanding of LLM fine-tuning \citep{malladi2023kernel}.

\textbf{Trade-off in RLHF and SFT.} Prior work on RLHF and SFT pipelines show that the standard sequential training either degrades alignment \citep{qi2023fine} or fine-tuning performance \citep{ouyang2022training}, depending on the order of training. To reconcile the two, some methods remove a distinct RLHF stage by regularizing or reformulating the SFT objective \citep{hong2024orpo, hua2024intuitive}, or by joint training with demonstrations \citep{li2024joint}. These approaches, however, necessitates same dataset used across stages. Adaptive model averaging (AMA) \citep{lin2023speciality} adaptively optimizes a weighted combination of SFT stage and RLHF stage parameters, however provable SFT and RLHF tradeoff in this case remains unclear. Another direction \citep{yang2024rewards, guo2024controllable} controls \textit{within} RLHF/SFT preferences via modified prompts. In contrast, our framework addresses the \textit{between} RLHF and SFT trade-off without data modification.

\vspace{-0.2cm}
\section{Experiments and Discussion}\label{sec:experiments}

\begin{figure*}[t]
    \centering
    \includegraphics[width=\linewidth]{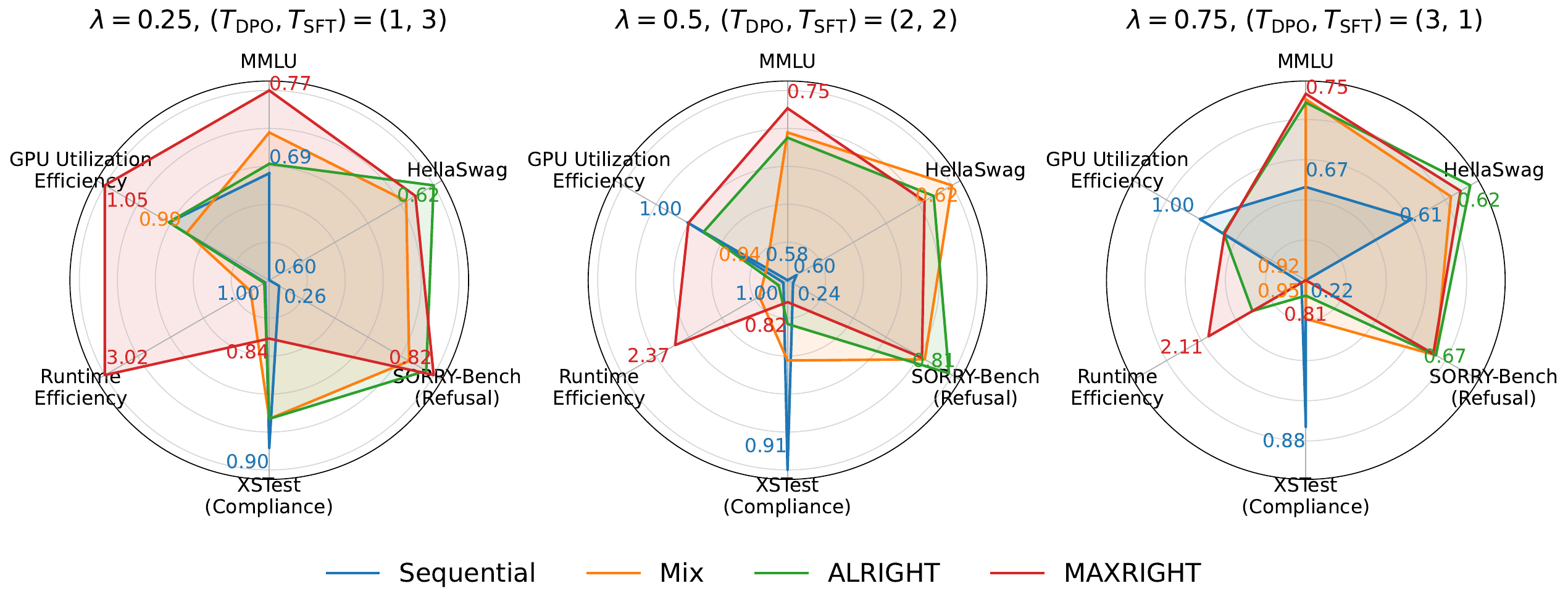}
    \caption{\textbf{Comparison in first DPO then SFT setting using \textsc{Llama3-8b} model.} Ealuation using MMLU, HellaSwag, SORRY-Bench, and XSTest, benchmarks, along with runtime and GPU utilization efficiencies across different post-training methods.}
    \label{fig:llama3-8b-sft-safe-dpo-radar-plots}
    \vspace{-0.2cm}
\end{figure*}

In this section, we empirically compare the proposed joint post-training framework with some existing baselines, in terms of their Pareto front performance, resource usage such as computation time and GPU memory utilization, and performance on popular LLM evaluation benchmarks.

\subsection{Experiment setup}

In this section, we describe the experimental setup for evaluating our framework. Additional experiment details and results are provided in Appendix \ref{app:exp-details}.

\textbf{Models.} We use two models: \textsc{pythia-1b}\citep{biderman2023pythia} to study optimization dynamics, trade-offs, and resource usage, \textsc{Llama3-8B}\citep{dubey2024llama} for real-world benchmarks, and \textsc{opt-1.3b}\citep{zhang2022opt} for experiments with SFT and  DPO stages sharing the same dataset.

\textbf{Datasets.} Three datasets are used for DPO. For general alignment, \textsc{Dahoas/rm-hh-rlhf}\citep{ganguli2022red}; for safety alignment, a filtered version of \textsc{PKU-Alignment/PKU-SafeRLHF}\citep{ji2024pku} (see Appendix \ref{app:llama3-8b-exp-det}); and for preference dataset shared with SFT, \textsc{HuggingFaceH4/ultrafeedback\_binarized}\citep{cui2023ultrafeedback}. For SFT in all other experiments, we use \textsc{vicgalle/alpaca-gpt4} dataset\citep{peng2023instruction}.

\textbf{Baseline Methods.} To compare with ALRIGHT and MAXRIGHT, we consider: (i) \textit{Mix}, a convex combination of DPO and SFT objectives; (ii) \textit{Sequential}, optimizing DPO and SFT in sequence (in either order); and (iii) ORPO\citep{hong2024orpo}, included when both DPO and SFT stages share the same preference dataset.

\textbf{Evaluation Metrics.} We compute the \textbf{optimality gap} for DPO and SFT objectives, defined as $f(\theta) - f^*$ where $f^*$ is the objective optimum, to capture separate objective performance. To capture joint performance, we measure the \textbf{ideal distance}, i.e., Euclidean distance between $(f_{\text{DPO}}(\theta), f_{\text{SFT}}(\theta))$ and $(f^*_{\text{DPO}}, f^*_{\text{SFT}})$. Resource efficiency is evaluated via \textbf{runtime} and \textbf{GPU utilization} (percentage increase and efficiency) relative to sequential training (e.g., Figure \ref{fig:dpo-sft-loss-space-comp}). Benchmark evaluations use \textbf{MMLU}\citep{hendrycks2020measuring} (knowledge/reasoning), \textbf{HellaSwag}\citep{zellers2019hellaswag, eval-harness} (commonsense inference), \textbf{SORRY-Bench}\citep{xie2024sorry} (safety refusal), and \textbf{XSTest}\citep{rottger2023xstest} (over-refusal). For evaluating post-training using \textsc{HuggingFaceH4/ultrafeedback\_binarized} dataset, we use \textsc{UltraRM-13b} reward model  \citep{cui2023ultrafeedback} to judge the win rate of each method against chosen responses.
\vspace{-0.3cm}

\subsection{Experimental results}
We report results on \textsc{pythia-1b} and \textsc{Llama3-8B}. For \textsc{pythia-1b}, we use \textsc{Dahoas/rm-hh-rlhf} (DPO) and \textsc{vicgalle/alpaca-gpt4} (SFT) in the DPO-then-SFT setting. Figure \ref{fig:dpo-sft-loss-space-comp} shows the Sequential method’s suboptimality and Mix’s high computational cost. Similar trends hold for \textsc{Llama3-8B} (Figures \ref{fig:safe-dpo-sft-loss-space-comp-llama3}, \ref{fig:dpo-sft-loss-space-comp-llama3}), and for SFT-then-DPO with both models (Figures \ref{fig:sft-dpo-loss-space-comp}, \ref{fig:sft-dpo-loss-space-comp-llama3}).
For \textsc{Llama3-8B} in the main text, we use filtered \textsc{PKU-SafeRLHF} (DPO) and \textsc{alpaca-gpt4} (SFT) to encourage safety refusal. Benchmark results (Figure \ref{fig:llama3-8b-sft-safe-dpo-radar-plots}) show Sequential suffers poor trade-offs, while our framework matches Mix’s performance at lower cost. Finally, we test with identical datasets for SFT and DPO stages (minimal conflict), with results in Figure \ref{fig:sft-dpo-opt-winrate}.

\textbf{ALRIGHT provides better control over the trade-off compared to Sequential.} We first consider the experiments using \textsc{pythia-1b} model. As shown in the top left plot of Figure \ref{fig:dpo-sft-loss-space-comp}, the optimization trajectories for DPO followed by SFT illustrate that the set of final models produced by ALRIGHT, for various values of $\lambda$, is more evenly distributed in the objective space.
This distribution forms a Pareto front, indicating that no model is strictly worse than another with respect to both objectives. Moreover, the spread of these models is comparable to that of the Mix method. In contrast, Sequential tends to produce models that are biased towards the SFT, even when $T_{\text{DPO}}$ is significantly larger than $T_{\text{SFT}}$ (e.g., $(T_{\text{DPO}}, T_{\text{SFT}}) = (5, 1)$).

\textbf{MAXRIGHT achieves near-ideal performance compared to other methods.} As illustrated in the bottom left plot of Figure \ref{fig:dpo-sft-loss-space-comp}, the optimization trajectories for DPO followed by SFT show that the set of final models produced by MAXRIGHT, for different values of $\lambda$, converge closer to the ideal point compared to other methods. This behavior is further supported by the Ideal Distance comparison in the right plot of Figure \ref{fig:dpo-sft-loss-space-comp}, where MAXRIGHT consistently achieves the best ideal distance performance across all $\lambda$ values. We attribute this advantage to the adaptive nature of MAXRIGHT, which dynamically selects the objective to update based on performance, rather than adhering to a fixed schedule like ALRIGHT.
This adaptability is particularly beneficial in heavily over-parameterized settings, where models have the capacity to approach ideal performance.

\textbf{ALRIGHT and MAXRIGHT require minimal additional resources compared to Sequential and significantly lower than Mix.} 
As shown in Figure \ref{fig:dpo-sft-loss-space-comp} (right, Increase in Runtime (\%) and Increase in GPU Utilization (\%)), the additional computational resources required by different implementations of ALRIGHT and MAXRIGHT are minimal (or even negative) relative to their Sequential counterparts. In contrast, Mix incurs substantial additional resource usage, with increases of over $50\%$ in runtime and more than $35\%$ in GPU utilization, despite achieving similar performance metrics to ALRIGHT and MAXRIGHT. Note that MAXRIGHT requires additional computation (not included in Figure \ref{fig:dpo-sft-loss-space-comp} results) before training time for approximating $f^*_{\text{DPO}}$ and $f^*_{\text{SFT}}$. However, once computed, these approximations can be reused for any number of preference and seed configurations used in the experiments.

\textbf{ALRIGHT and MAXRIGHT perform significantly better than Sequential in real-world benchmarks.} 
Experiment results for safety focused DPO and SFT with \textsc{Llama3-8B} are given in Figure \ref{fig:llama3-8b-sft-safe-dpo-radar-plots}. It can be seen that all joint-training methods significantly outperform Sequential method in MMLU and SORRY-Bench benchmarks, and slightly outperform Sequential method in HellaSwag benchmark, in all trade-off settings. Sequential performs better in XSTest benchmark, simply because of the lack of safety awareness in Sequential trained model, resulting in a larger compliance rate to any prompt, irrespective of the nature of the prompt. 
This bias results in a poor trade-off between tasks, and hence our proposed framework significantly outperforms Sequential method in terms of overall performance across all benchmarks, in particular, up to $23\%$ improvement across all benchmarks by ALRIGHT (see Table \ref{tab:llam3-safety-benchmark} in Appendix \ref{app:llama3-8b-exp}). Furthermore, in terms of computational resource usage, it can be seen that our proposed joint framework performs similarly to or even surpasses Sequential, while Mix fails to do so. The reason for this phenomenon is, as seen from the optimization trajectories in Figure \ref{fig:safe-dpo-sft-loss-space-comp-llama3} (Appendix \ref{app:llama3-8b-exp}), our joint framework (MAXRIGHT strategy in particular) can converge to near ideal point faster, whereas Sequential never converges to this point.  

\begin{figure}[t]
    \centering
    \includegraphics[width=\linewidth]{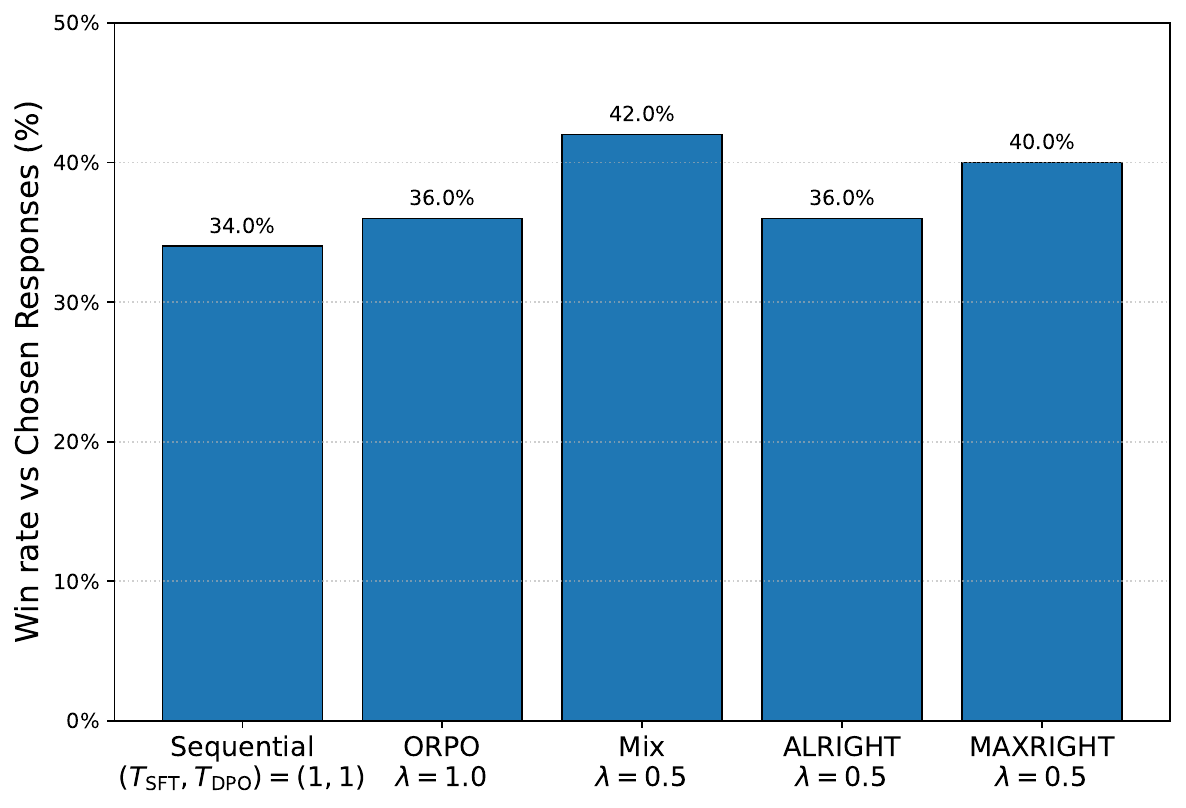}
    \caption{\textbf{Comparison in first SFT then DPO setting using \textsc{opt-1.3b} model.} Evaluation using win rate against the chosen responses 50 test set samples from \textsc{HuggingFaceH4/ultrafeedback\_binarized} dataset judged by \textsc{UltraRM-13b} reward model.}
    \label{fig:sft-dpo-opt-winrate}
    \vspace{-0.7cm}
\end{figure}

\textbf{Joint framework performance in non-conflicting post-training.} In addition to experiments in DPO and then SFT setting with diverse datasets, we also conduct experiments in the SFT then DPO setting, where both stages share the same preference dataset. Note that this setting is more conducive for Sequential, given minimal conflict between the two post-training stages. Since all joint training methods (Mix, ALRIGHT, and MAXRIGHT) utilize a reference model, they are trained only for one epoch, while ORPO (which does not require a reference model) is trained for 2 epochs. All methods are trained with equal preference for SFT and DPO objectives. From Figure \ref{fig:sft-dpo-opt-winrate}, it can be seen that the Mix method outperforms all the other methods, while MAXRIGHT performs comparably. It is noteworthy that even in a more conducive setting, Sequential performs poorly, suggesting that the joint training framework has an advantage even when there is minimal conflict between the datasets used in the SFT and DPO stages. Compared to all joint training methods, ORPO is more computationally efficient due to the lack of need for a reference model. However, ORPO can only be applied when the same dataset is used in both the SFT and DPO stages, whereas our methods can be applied for an arbitrary combination of datasets for SFT and DPO stages.

\bibliographystyle{plainnat}

\appendix
\onecolumn

\begin{center}
{\large \bf Supplementary Material for
``Understanding Forgetting in LLM Supervised Fine-Tuning and Preference Learning - A Convex Optimization Perspective"}
\end{center}

\doparttoc 
\faketableofcontents 

\section{Proofs for Theoretical Results}\label{app:proofs}

In this section, we provide the proofs for the main theoretical results of the paper, Theorems \ref{thm:seq-lower-bound} and \ref{thm:alt-upper-bound}. The proof is organized as follows: In Appendix \ref{app:basic-claims}, we establish some fundamental claims regarding the problem setup, such as the convexity of the objectives and the boundedness of the second moment of the stochastic gradients. Then, in Appendix \ref{app:proof-seq-lower-bound}, we prove the lower bound stated in Theorem \ref{thm:seq-lower-bound} along with several supporting lemmas. Finally, in Appendix \ref{app:proof-alt-upper-bound}, we prove the upper bound stated in Theorem \ref{thm:alt-upper-bound}, accompanied by additional supporting lemmas. 

For conciseness, in this section we omit the subscripts DPO and SFT of the input $x$, but whether $x$ belongs to $\D_{\text{\tiny DPO}}$ or $\D_{\text{\tiny SFT}}$ can be inferred from the context. Furthermore, for brevity, we denote the averaging over the datasets \scalebox{1.0}{$\frac{1}{N_1}\sum_{\tiny{ x, y_w, y_\ell \in\D_{\text{\tiny DPO}}}}$} and \scalebox{1.0}{$\frac{1}{N_2}\sum_{x, y\in\D_{\text{\tiny SFT}}}$} by $\E_{x, y_w, y_\ell \sim \D_{\text{\tiny DPO}}}$ and $\E_{ x, y \sim \D_{\text{\tiny SFT}}}$, respectively.

\subsection{Basic claims on the problem setup}\label{app:basic-claims}

Before going to the proofs of the main results, we can have the following basic claims on the problem setup.

\begin{proposition}[Bounded second moment of stochastic gradient]\label{prop:bounded-stoch-grad}
For all $\theta\in\Theta$, there exist $M_1, M_2>0$ such that
    \begin{align}
        \E_{x, y_w, y_\ell \sim \D_{\text{\tiny DPO}}}[\Vert g_{\text{\tiny DPO}}(\theta; x, y_w, y_\ell) \Vert^2] &\leq M_1^2, \\
        \E_{ x, y \sim \D_{\text{\tiny SFT}}}[\Vert g_{\text{\tiny SFT}}(\theta;x, y ) \Vert^2] &\leq M_2^2.
    \end{align}
\end{proposition}
\begin{proof}   
    Considering $g_{\text{\tiny DPO}}$, we can first simplify $h_\beta$ defined in \eqref{eq:h} under softmax parameterization of $\pi_\theta$ and $\pi_{\text{\rm ref}}$ as
    \begin{align}
        h_\beta(\theta;x, y_w, y_\ell) &=  \beta \log\left(\frac{\pi_\theta(y_w~|~x)}{\pi_{\text{\rm ref}}(y_w~|~x)}\right)-\beta \log\left(\frac{\pi_\theta(y_\ell~|~x)}{\pi_{\text{\rm ref}}(y_\ell~|~x)}\right) \nonumber \\
        &= \beta(\theta - \theta_{\text{\rm ref}})^\top(\phi_{y_w, x}-\phi_{y_\ell, x}).
    \end{align}
    Then we can simplify $g_{\text{\tiny DPO}}$ as
    \begin{align}\label{eq:g1-simp}
        g_{\text{\tiny DPO}}(\theta;x, y_w, y_\ell) &= -\left(1- \sigma( h_\beta(\theta;x, y_w, y_\ell, \pi_{\text{\rm ref}})\right)\nabla_\theta h_\beta(\theta;x, y_w, y_\ell, \pi_{\text{\rm ref}}) \nonumber \\
        &= -\beta (1 - \sigma\left( \beta(\theta - \theta_{\text{\rm ref}})^\top(\phi_{y_w, x}-\phi_{y_\ell, x}) \right)(\phi_{y_w, x}-\phi_{y_\ell, x}).
    \end{align}
    We can then bound the norm of $g_{\text{\tiny DPO}}$ as
    \begin{align}\label{eq:g1-norm}
        \Vert g_{\text{\tiny DPO}}(\theta;x, y_w, y_\ell) \Vert^2 &= \beta^2 (1 - \sigma\left( \beta(\theta - \theta_{\text{\rm ref}})^\top(\phi_{y_w, x}-\phi_{y_\ell, x}) \right)^2 \Vert \phi_{y_w, x}-\phi_{y_\ell, x} \Vert^2 \nonumber\\
        &\leq \beta^2 \Vert \phi_{y_w, x}-\phi_{y_\ell, x} \Vert^2 \nonumber\\
        &\leq 4\beta^2 \Phi^2 =: M_1,
    \end{align}
    where the first inequality is due to the fact that $0\leq \sigma(z) \leq 1$ for all $z\in\R$, second inequality is due to Cauchy-Schwarz inequality and Assumption \ref{ass:features}. Taking expectation (average) over the dataset $\D_{\text{DPO}}$ in \eqref{eq:g1-norm} proves the first part of Proposition \ref{prop:bounded-stoch-grad}.
    For proving the second part of the proposition, we start by simplifying the gradient of $\pi_\theta$ under softmax-parameteriation, given by
        \begin{align}
           \nabla_\theta \pi_\theta(y~|~x)  &= \nabla_\theta \frac{\exp(\theta^\top \phi_{y, x})}{\sum_{y'\in\mathcal{Y}}\exp(\theta^\top \phi_{y', x})} \nonumber \\
           &= \frac{\exp(\theta^\top \phi_{y, x})\sum_{y'\in\mathcal{Y}}\exp(\theta^\top \phi_{y', x}) \phi_{y, x}  - \exp(\theta^\top \phi_{y, x})\sum_{y'\in\mathcal{Y}}\exp(\theta^\top \phi_{y', x})\phi_{y', x}}{\left( \sum_{y'\in\mathcal{Y}}\exp(\theta^\top \phi_{y', x}) \right)^2} \nonumber \\
           &= \left(\phi_{y, x} - \bar{\phi}_x(\theta) \right)\pi_\theta(y~|~x),
        \end{align}
        where 
        \begin{align}
            \bar{\phi}_x(\theta) &\coloneqq \frac{\sum_{y'\in\mathcal{Y}}\phi_{y', x}\exp(\theta^\top \phi_{y', x})}{\sum_{y'\in\mathcal{Y}}\exp(\theta^\top \phi_{y', x})}.
        \end{align} 
    Then we can simplify $g_{\text{\tiny SFT}}$ as
    \begin{align}\label{eq:g2-simp}
        g_{\text{\tiny SFT}}(\theta;x, y) &= -\frac{\nabla_\theta \pi_\theta(y~|~x)}{\pi_\theta(y~|~x)} \nonumber\\
        &= -(\phi_{y, x} - \bar{\phi}_x(\theta)).
    \end{align}
    We can then bound the norm of $g_{\text{\tiny SFT}}$ as
    \begin{align}\label{eq:g2-norm}
        \Vert g_{\text{\tiny SFT}}(\theta;x, y) \Vert^2 &= \Vert \phi_{y, x} - \bar{\phi}_x(\theta) \Vert^2 \nonumber\\
        &\leq 4\Phi^2 =: M_2,
    \end{align}
    where the inequality is due to the Cauchy-Schwarz inequality and Jensen's inequality. Taking expectation (average) over the dataset $\D_{\text{SFT}}$ in \eqref{eq:g2-norm} proves the second part of Proposition \ref{prop:bounded-stoch-grad}.
\end{proof}

\begin{proposition}[Convexity of objectives]\label{prop:convex}
    Under Assumption \ref{ass:features}, the objectives $f_{\text{\tiny DPO}}$ and $f_{\text{\tiny SFT}}$ (defined in \eqref{eq:f1} and \eqref{eq:f2}, respectively) are convex.
\end{proposition} 
\begin{proof}
    The goal of the proof is to show the Hessians of the objectives $f_{\text{DPO}}$ and $f_{\text{SFT}}$ are semi-positive definite, under Assumption \ref{ass:features} and LLMs (both trainable and reference) modeled using softmax parameterization. First, considering $f_{\text{DPO}}$, we can have
    \begin{align}
        \nabla_\theta g_{\text{\tiny DPO}}(\theta;x, y_w, y_\ell) &= -\nabla_\theta \beta (1 - \sigma\left( \beta(\theta - \theta_{\text{\rm ref}})^\top(\phi_{y_w, x}-\phi_{y_\ell, x}) \right)(\phi_{y_w, x}-\phi_{y_\ell, x}) \nonumber\\
        &= \beta^2 \sigma\left( h_\beta(\theta;x, y_w, y_\ell) \right) \left(\phi_{y_w, x}-\phi_{y_\ell, x} \right)\left(\phi_{y_w, x}-\phi_{y_\ell, x} \right)^\top \succeq 0,
    \end{align}
    where first equality is due to \eqref{eq:g1-simp}, and the semi-positive definiteness is due to the fact that $\beta>0$, $0\leq \sigma(z) \leq 1$ for all $z\in\R$, and $\left(\phi_{y_w, x}-\phi_{y_\ell, x} \right)\left(\phi_{y_w, x}-\phi_{y_\ell, x} \right)^\top\succeq 0 $ for all $\phi_{y_w, x}, \phi_{y_\ell, x}$. The convexity of $f_{\text{DPO}}$ follows from the fact that
    \begin{align}
        \nabla^2 f_{\text{\tiny DPO}}(\theta) = \E_{x, y_w, y_\ell\sim\D_{\text{\tiny DPO}}}\left[ \nabla_\theta g_{\text{\tiny DPO}}(\theta;x, y_w, y_\ell) \right].
    \end{align}
    Similarly, we can compute $\nabla_\theta g_{\text{\tiny SFT}}$ as
    \begin{align}
        \nabla_\theta g_{\text{\tiny SFT}}(\theta;x, y) &= -\nabla_\theta (\phi_{y, x} - \bar{\phi}_x(\theta)) \nonumber\\
        &= \nabla_\theta \frac{\sum_{y'\in\mathcal{Y}}\phi_{y', x}\exp(\theta^\top \phi_{y', x})}{\sum_{y'\in\mathcal{Y}}\exp(\theta^\top \phi_{y', x})} \nonumber\\
        &= \sum_{y'\in\mathcal{Y}}\frac{\phi_{y', x}\phi_{y', x}^\top\exp(\theta^\top \phi_{y', x})}{\sum_{y'\in\mathcal{Y}}\exp(\theta^\top \phi_{y', x})} \nonumber\\
        &~~- \left(\frac{\sum_{y'\in\mathcal{Y}}\phi_{y', x}\exp(\theta^\top \phi_{y', x})}{\sum_{y'\in\mathcal{Y}}\exp(\theta^\top \phi_{y', x})}\right)\left(\frac{\sum_{y'\in\mathcal{Y}}\phi_{y', x}\exp(\theta^\top \phi_{y', x})}{\sum_{y'\in\mathcal{Y}}\exp(\theta^\top \phi_{y', x})}\right)^\top,
    \end{align}
    where the first equality is due to \eqref{eq:g2-simp}. To establish the semi-positivedefiniteness of $\nabla_\theta g_{\text{\tiny SFT}}$, cosider any $v$ with same dimension as $\theta$, and let $p_{y,x} = \frac{\exp(\theta^\top \phi_{y, x})}{\sum_{y'\in\mathcal{Y}}\exp(\theta^\top \phi_{y', x})}$. Note that $p_{y,x}\geq 0$ for all $x, y$, and $\sum_{y\in\mathcal{Y}}p_{y,x}=1$ for all $x$. Then we can have
    \begin{align}
        &v^\top \Bigg(\sum_{y'\in\mathcal{Y}}\frac{\phi_{y', x}\phi_{y', x}^\top\exp(\theta^\top \phi_{y', x})}{\sum_{y'\in\mathcal{Y}}\exp(\theta^\top \phi_{y', x})} \nonumber \\
        &- \left(\frac{\sum_{y'\in\mathcal{Y}}\phi_{y', x}\exp(\theta^\top \phi_{y', x})}{\sum_{y'\in\mathcal{Y}}\exp(\theta^\top \phi_{y', x})}\right)\left(\frac{\sum_{y'\in\mathcal{Y}}\phi_{y', x}\exp(\theta^\top \phi_{y', x})}{\sum_{y'\in\mathcal{Y}}\exp(\theta^\top \phi_{y', x})}\right)^\top \Bigg) v \nonumber\\
        &= \sum_{y'\in\mathcal{Y}}(v^\top \phi_{y', x})^2 p_{y',x} - \left(\sum_{y'\in\mathcal{Y}}v^\top \phi_{y', x} p_{y',x} \right)^2 \geq 0,
    \end{align}
    where the last inequality is due to Jensen's inequality. This suggests that $\nabla_\theta g_{\text{\tiny SFT}}(\theta;x, y)\succeq 0$, and the convexity of of $f_{\text{SFT}}$ follows from the fact that
    \begin{align}
        \nabla^2 f_{\text{\tiny SFT}}(\theta) = \E_{x, y\sim\D_{\text{\tiny SFT}}}\left[ \nabla_\theta g_{\text{\tiny SFT}}(\theta;x, y) \right].
    \end{align}
    Note that when $f_{\text{\tiny DPO}}$ and $f_{\text{\tiny SFT}}$ are convex, $f_{\text{Mix}, \lambda}$ is also convex for all $\lambda\in[0,1]$.
\end{proof}

\subsection{Proof of Theorem \ref{thm:seq-lower-bound}} \label{app:proof-seq-lower-bound}

\begin{figure}[t]
    \centering
    \includegraphics[width=0.42\linewidth]{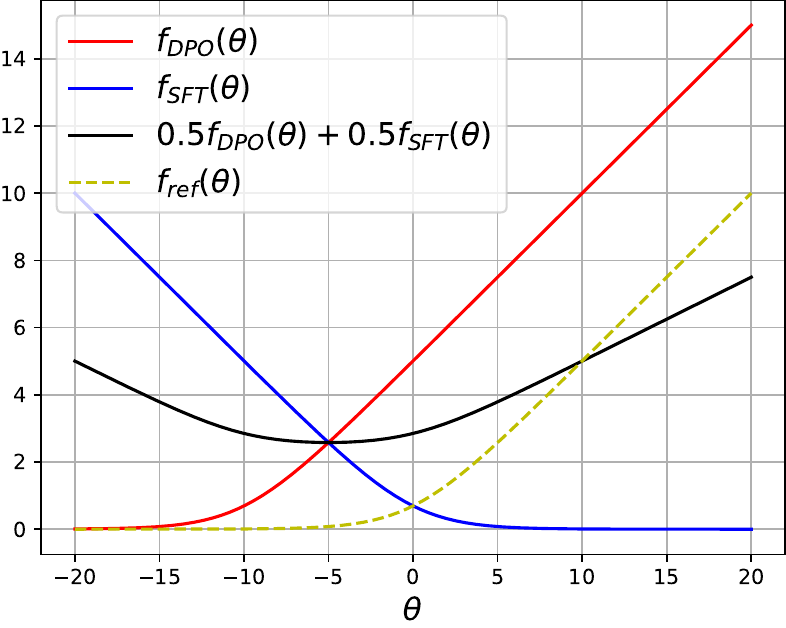}
    \includegraphics[width=0.45\linewidth]{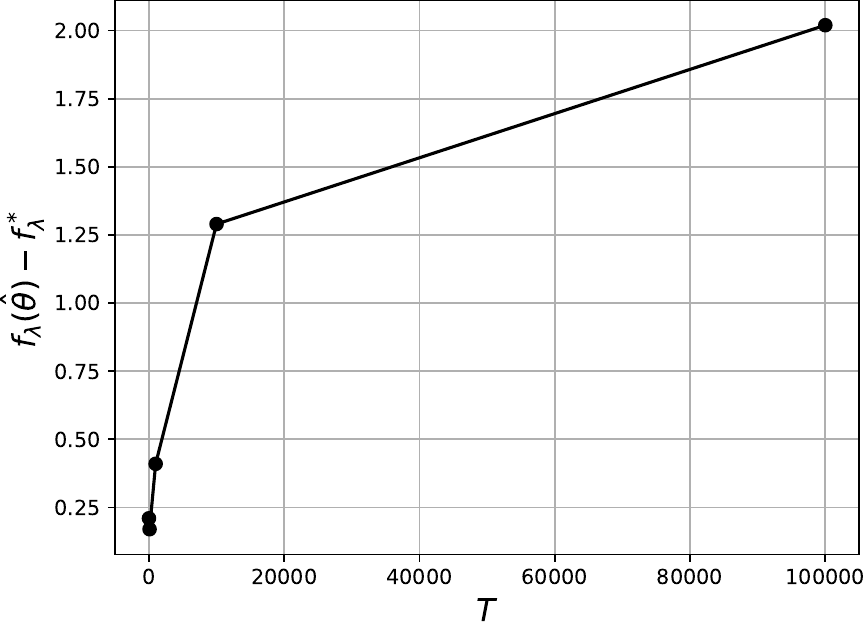}
    \caption{An Illustration of the example used for lower bound derivation in Theorem \ref{thm:seq-lower-bound}}
    \label{fig:lower-bound-illustration}
\end{figure}

\begin{proof}
    For deriving the lower bound given in \ref{thm:seq-lower-bound}, we consider the following problem setup that satisfies Assumption $\ref{ass:features}$. Let $\Theta$ be $\R$, and let $\D_{\text{\tiny DPO}}= \{ (x_1, y_w, y_\ell) \}$ and $\D_{\text{\tiny SFT}}= \{ (x_2, y) \}$ and consider only two possible outputs exist (i.e. binary classification) such that we have the specification in Table \ref{tab:lower-bound-data-spec}.
    \begin{table}[h!]
        \centering
        \caption{Data set specification for lower bound analysis example}
        \begin{tabular}{c|c|c}
            \hline
            Input & Output & Feature $\phi_{y,x}$\\
            \hline
             $x_1$& $y_w=1$  & $-1.0$\\
             $x_1$& $y_\ell=0$ & $-0.5$\\
             $x_2$& $y=0$ & $1.0$\\
             $x_2$& $y'=1$ & $0.5$\\
        \end{tabular}
        \label{tab:lower-bound-data-spec}
    \end{table}
    Note that the data point $y'$ is not used in training explicitly, and it is specified to enable the calculation of the output of $\pi_\theta$ with softmax parameterization in the SFT optimization stage. Based on this dataset, we can also define the dataset for reference policy objective as $\D_{\rm ref}= \{ (x_1, y_w) \}$, which has a similar optimization procedure as SFT. Before moving forward, we have to choose the reference policy $\pi_{\rm ref}$, or equivalently $\theta_{\rm ref}$. The objective to determine $\theta_{\rm ref}$ is given by
    \begin{align}
        \theta_{\rm ref}\in \arg\!\min_{\theta} f_{\rm ref}(\theta) \coloneqq -\log\pi_\theta(y_w~|~x_1),
    \end{align}
    which is graphically illustrated in Figure \ref{fig:lower-bound-illustration}. We choose $\theta_{\rm ref}=-5$ since this choice reasonably optimizes the objective.
    
    With this problem setup and choice of $\theta_{\rm ref}$, we can then derive the objectives and corresponding gradients using \eqref{eq:f1}, \eqref{eq:f2}, \eqref{eq:g1-simp}, and \eqref{eq:g2-simp} as
    \begin{align}
        f_{\text{\tiny DPO}}(\theta) &= \log \left( 1 + c \frac{1 + \exp(\theta/2)}{1+\exp(-\theta/2)} \right)\\
        f_{\text{\tiny SFT}}(\theta) &= \log(1 + \exp(-\theta/2))\\
        g_{\text{\tiny DPO}}(\theta) &= \frac{1}{2}\cdot \frac{1}{1 + \exp(-(\theta/2 + 5))}\\
        g_{\text{\tiny SFT}}(\theta) &= -\frac{1}{2}\cdot \frac{1}{1 + \exp(\theta/2)},
    \end{align}
    where $c=\frac{1+\exp(5)}{1+\exp(-5)}$. Choosing $\lambda=0.5$, we can numerically find an upper bound to $f^*_{\text{\tiny Mix}, \lambda}$ such that $f^*_{\text{\tiny Mix}, \lambda}\leq \frac{1}{2}\log c^*$ where $c^*=173.78$. With this, we can derive a lower bound to $G_{\text{Mix}, \lambda}(\theta)$ as
    \begin{align}\label{eq:Gmix lower bound}
        G_{\text{Mix}, \lambda}(\theta) &= f_{\text{\tiny Mix}, \lambda}(\theta) - f^*_{\text{\tiny Mix}, \lambda}\nonumber\\
        &\geq \frac{1}{2}\cdot\log \left( 1 + c \frac{1 + \exp(\theta/2)}{1+\exp(-\theta/2)} \right) + \frac{1}{2}\cdot\log(1 + \exp(-\theta/2)) - \frac{1}{2}\log c^* \nonumber\\
        &=\frac{1}{2} \log\left(\frac{1}{c^*}(1 + \exp(-\theta/2)) + \frac{c}{c^*}(1 + \exp(\theta/2)) \right).
    \end{align}

 When $\theta = 0$, the right hand side of \eqref{eq:Gmix lower bound} approximately equals $0.256843$. Since it is monotonically increasing for $\theta\in[0,\infty)$, we have $G_{\text{Mix}, \lambda}(\theta) \gtrsim 0.256843 = \Omega(1)$ when $\theta\geq 0$. 
Thus to prove the result, it is sufficient to show Algorithm \ref{algo:seq-rlhf-sft} will generate a $\hat{\theta}_{\text{Seq}}$ that is greater than $0$.

We have the first stage iterates:
\begin{align}\label{eq:first stage iterates}
    \theta_{t+1}^{1}=\theta_t^1-\alpha_t \frac{1}{1+\exp(-\frac{\theta_t^1}{2}-5)},~~\text{for }t=1,...,T-1.
\end{align}
Using the first stage's last iterate as initial point, we have the second stage iterates:
\begin{align}\label{eq:second stage iterates}
    \theta_{t+1}^{2,T}=\theta_t^{2,T}+\alpha_t \frac{1}{1+\exp(\frac{\theta_t^{2,T}}{2})},~~\text{for }t=1,...,T-1.
\end{align}
where $\theta_1^{2,T}=\theta_{T}^1$ and the superscript $T$ in $\theta_t^{2,T}$ indicates the max iteration index in the first stage.

Without loss of generality, we initialize $\theta^1_1 =0$. Then by \eqref{eq:first stage iterates} and \eqref{eq:second stage iterates}, we have
\begin{align}\label{eq:last iterate}
\hat{\theta}_{\text{Seq}} = \theta_{T}^{2,T} &=-\alpha_t \sum_{t=1}^{T-1} \frac{1}{1+\exp(-\frac{\theta_t^1+10}{2})} + \alpha_t  \sum_{t=1}^{T-1} \frac{1}{1+\exp(\frac{\theta_t^{2,T}}{2})}
\end{align}
We first prove the following lemma.
\begin{lemma}\label{lemma:sig sequence}
        Given a positive constant $c$ and a sequence $\{\alpha_t\}$, consider the iterates generated by
        \begin{equation}\label{eq:sig sequence}
            \theta_{t+1}=\theta_t+\alpha_t \mysig{c \theta_t},~~\theta_{t+1}'=\theta_t'+\alpha_t \mysig{c \theta_t'}
        \end{equation}
        If $\theta_1-\theta_1' \geq 0$ and  $\frac{c\alpha_t}{4} \leq 1$ for any $t$, then we have
        \begin{align}
            \theta_t - \theta_t' \geq (\theta_1-\theta_1')\Pi_{i=1}^{t-1}\Big(1-\frac{c \alpha_i}{4}\Big) , ~~\forall t. \nonumber
        \end{align}
    \end{lemma}
    \begin{proof}
    We prove the result by induction.
    Assume $\theta_t-\theta_t' \geq 0$ for some $t$.
    We first have
    \begin{align}\label{eq:theta-theta'}
        \theta_{t+1} - \theta_{t+1}' 
        &= \theta_t -\theta_t' + \alpha_t \Big( \mysig{c\theta_t}-\mysig{c\theta_t'}\Big) 
    \end{align}
    With $\nabla_\theta \mysig{c\theta} = -c\sigma (-c\bar\theta )\big(1-\sigma(-c \theta) \big)$ where $\sigma$ is the sigmoid function, using the mean value theorem in \eqref{eq:theta-theta'}, we have for some $\bar{\theta}_t$ in between $\theta_t$ and $\theta_t'$ that
    \begin{align}
    \theta_{t+1} - \theta_{t+1}' 
        &= \theta_t - \theta_t' -c \alpha_t \sigma (-c\bar{\theta}_t )\big(1-\sigma(-c\bar{\theta}_t) \big)(\theta_t-\theta_t') \nonumber\\
        &\geq  \theta_t - \theta_t' - \frac{1}{4}c \alpha_t (\theta_t-\theta_t') \nonumber\\
        & = \Big(1-\frac{1}{4}c \alpha_t \Big)(\theta_t-\theta_t') \nonumber
    \end{align}
    where the inequality follows from $\sigma (-c\bar{\theta}_t )\big(1-\sigma(-c\bar{\theta}_t) \big) \leq \max_{x\in[0,1]}x(1-x)=\frac{1}{4}$ and the assumption that $\theta_t-\theta_t' \geq 0$. Since $\theta_1-\theta_1' \geq 0$, it follows from induction that $\theta_t-\theta_t' \geq 0$ for any $t$. Then recursively applying the last inequality completes the proof.
    \end{proof}
    Now we consider \eqref{eq:last iterate}. Rewriting \eqref{eq:first stage iterates} gives that $\theta_t^1$ is generated by
    \begin{align}\label{eq:first stage iterates rewrite}
        -(\theta_{t+1}^{1}+10)=-(\theta_t^1+10)+\alpha_t \frac{1}{1+\exp(-\frac{\theta_t^1+10}{2})},~~\text{for }t=1,...,T-1, \nonumber
    \end{align}
    and $\theta_t^{2,T}$ is generated by \eqref{eq:second stage iterates}. Thus $\theta_t^{2,T}$ and $-(\theta_t^{1}+10)$ are generated by \eqref{eq:sig sequence} with $c=1/2$. Assuming the step size $\alpha_t$ is proper that $\sum_{t=1}^\infty \alpha_t=\infty$, then there always exists $T^*$ that for $T\geq T^*$, we have $\theta_T^{1}$ is small enough such that $-(\theta_1^1+10)-\theta_1^{2,T}=-10-\theta_T^{1}\geq 0$. If $\alpha_t/8 \leq 1$, we have Lemma \ref{lemma:sig sequence} holds for sequences $\theta_t^{2,T}$ and $-(\theta_{t+1}^{1}+10)$, and we have
    \begin{equation}\label{eq:first difference}
        -(\theta_t^1+10)\geq\theta_t^{2,T}  \Rightarrow \mysig{\theta_t^{2,T}/2}-\mysig{-(\theta_t^1+10)/2} \geq 0.
    \end{equation}
     Using \eqref{eq:first difference} in \eqref{eq:last iterate}, we have $\hat{\theta}_{\text{Seq}}\geq 0$. This completes the proof.
     \end{proof}

\begin{remark} We provide some intuition on why we can obtain this suboptimality result for the sequential method, how the proposed methods can overcome the suboptimality, and evidence for causes of suboptimality in real-world post-training tasks:

\textbf{Cause of suboptimality.} The key idea behind the suboptimality result is constructing two conflicting objectives for DPO and SFT. To this end, we exploit the fact that DPO and SFT datasets can be diverse, and they can have inconsistencies in labelling. Specifically, in our lower bound analysis, we selected DPO and SFT inputs (by choosing the features $\phi_{x, y}$) and corresponding labels $y_w$, $y_\ell$, and  $y$ such that the DPO and SFT objectives induced by these datasets conflict with each other (Table \ref{tab:lower-bound-data-spec}). Note that in this example both inputs and labels are different, thus there is no explicit labeling inconsistency. More severely conflicting DPO and SFT objectives can be constructed by to fixing the input (and hence the corresponding feature $\phi_{x, y}$) for both DPO and SFT datasets and assign conflicting labels to the chosen response for the DPO input and the target response for the SFT input.  However, we used a more general example in our lower bound analysis, which shows that even when there is no explicit labeling inconsistency (e.g., when the input data for the two datasets are different), the sequential approach can still fail.

\textbf{How proposed methods can overcome suboptimality?}  The strength of the proposed alternating approaches over the sequential approach is not its ability to explicitly correct any diversity in the datasets or inconsistency in labeling. Rather, it lies in their ability to navigate the loss landscape of both fine-tuning and alignment objectives, reasonably optimizing both. Given two (possibly conflicting) datasets and a pretrained LLM, our proposed methods can achieve a model that is most reasonably optimal for both SFT and alignment datasets, with a controllable trade-off between them.

\textbf{Data diversity and labeling inconsistency in real-world data.} We investigated the data diversity and labeling inconsistencies in our real-world experiments (e.g., Figures \ref{fig:dpo-sft-loss-space-comp} and \ref{fig:sft-dpo-loss-space-comp}) using the datasets \textsc{alpaca-gpt4} for SFT and \textsc{rm-hh-rlhf} for DPO. Specifically, we visualized:
\begin{itemize}
    \item cosine similarities between the closest input pairs $(x^{(i)}_ {\scalebox{0.5}{DPO}}, x^{(i)}_ {\scalebox{0.5}{SFT}})$ from the two datasets;
    \item cosine similarities between chosen/rejected responses for the DPO input and the target response for the corresponding SFT input ($(y^{(i)}_w, y^{(i)})$/$(y^{(i)}_\ell, y^{(i)})$); and,
    \item cosine similarities between chosen/rejected responses of the DPO input $(y^{(i)}_w, y^{(i)}_\ell)$. Here, the superscript $(i)$ denotes the ranking of the DPO-SFT input pair based on their similarity, which was measured at the token level.
\end{itemize}

The visualization is given in Figure \ref{fig:data-div-viz}. It can be observed that most DPO and SFT input pairs do not exhibit very strong cosine similarity (falling in the range of approximately 0.57–0.78). Observing the corresponding inputs, the similarity seems to come from mainly the sentence structure, rather than the underlying meaning. Thus, there is some inherent dissimilarity in the inputs themselves. Furthermore, for some similar input pairs $(x^{(i)}_ {\scalebox{0.5}{DPO}}$, $x^{(i)}_ {\scalebox{0.5}{SFT}})$, the corresponding $(y^{(i)}_w, y^{(i)})$ similarity is close to zero, which might indicate a conflict in labeling. This suggests that there is some inconsistency in labeling in the real-world setup, although it is not as conclusive or adversarial as in the toy example used for the lower bound analysis. While here we use a token level comparison using cosine similarity,  A deeper investigation into these data-level inconsistencies would be valuable, particularly in addressing them during data preprocessing or selection to enhance the optimization process.

\begin{figure}
    \centering
    \includegraphics[width=\linewidth]{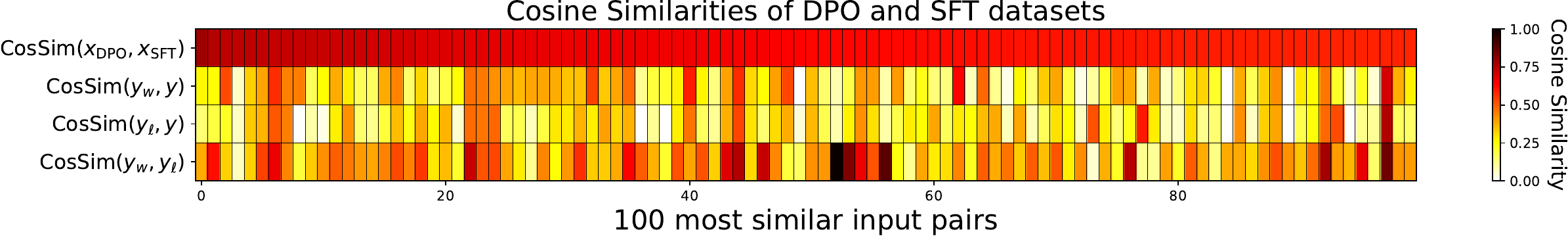}
        \vspace{-0.5cm}
    \caption{Visualization of data diversity and labeling inconsistencies in DPO (\textsc{rm-hh-rlhf}) and SFT (\textsc{alpaca-gpt4}) datasets.}
    \label{fig:data-div-viz}
\end{figure}

\end{remark}

\subsection{Proof of Theorem \ref{thm:alt-upper-bound}} \label{app:proof-alt-upper-bound}

In this section, we provide the proof for Theorem \ref{thm:alt-upper-bound}. First, we provide some useful results that will later be used in the main proof. For conciseness, we denote $g_{\text{\tiny DPO}}(\theta_t;x^t, y_w^t, y_\ell^t)$ and $g_{\text{\tiny SFT}}(\theta_t;x^t, y^t)$ by $g_{\text{\tiny DPO}, t}$ and $g_{\text{\tiny SFT}, t}$, respectively. 

\begin{lemma}[\citep{rockafellar1970convex} Theorem 25.1 and Corollary 25.11]\label{lem:convex}
    Consider any convex differentiable function $f$ and let $\theta\in\Theta$. Then, for any $\theta'\in\Theta$, we have
    \begin{align}
        f(\theta') \geq f(\theta) + \nabla f(\theta)^\top(\theta' - \theta).
    \end{align}
\end{lemma}

\begin{lemma}[\citep{orabona2020lastiterate} Lemma 1.]\label{lem:orabona} Let $\{\eta_t\}_{t=1}^{T}$ be a non-increasing sequence of positive numbers and $q_t\geq 0$ for all $t=1, \dots, T$. Then
\begin{align}\label{eq:orabona-lemma}
    n_Tq_T \leq \frac{1}{T}\sum_{t=1}^{T}\eta_t q_t + \sum_{k=1}^{T-1}\frac{1}{k(k+1)}\sum_{t=T-k+1}^T \eta_t (q_t - q_{T-k} )
\end{align}
\end{lemma}

\begin{lemma}\label{lem:one-step-sgd} Consider iterates $\theta_t$ and $\theta_{t+1}$ generated by Algorithm \ref{algo:joint-rlhf-sft} with ALRIGHT strategy for $t\in\{1, \dots, T-1\}$. Then for any $\theta'\in\Theta$ and $\lambda\in[0,1]$, we have 
\begin{align}\label{eq:one-step-sgd}
     \E\left[ f_{\text{Mix}, \lambda}(\theta_t) - f_{\text{Mix}, \lambda}(\theta') \right] \leq \frac{1}{2\alpha_t} \E\left[ \Vert \theta_t -\theta' \Vert^2 \right] - \frac{1}{2\alpha_t} \E \left[ \Vert \theta_{t+1} -\theta' \Vert^2\right] + \frac{\alpha_t}{2} M^2_\lambda,
\end{align}
where $M_\lambda=\lambda M_1 + (1-\lambda) M_2$, and $M_1,M_2$ are as defined in Proposition \ref{prop:bounded-stoch-grad}.
\end{lemma}
\begin{proof}
    Considering Algorothm \ref{algo:joint-rlhf-sft} with ALRIGHT strategy, for any $\theta'\in\Theta$, we can have
    \begin{align}
        \Vert \theta_{t+1} - \theta' \Vert^2 - \Vert \theta_t - \theta' \Vert^2 &\leq \Vert \theta_t - \alpha_t (\I_{i_t=1} g_{\text{\tiny DPO}, t} + \I_{i_t=0} g_{\text{\tiny SFT}, t} ) -\theta' \Vert^2 - \Vert \theta_t - \theta'\Vert^2 \nonumber \\
        &= - 2\alpha_t (\I_{i_t=1} g_{\text{\tiny DPO}, t} + \I_{i_t=0} g_{\text{\tiny SFT}, t} )^\top (\theta_t - \theta') \nonumber\\
        &~~+ \alpha^2_t \Vert \I_{i_t=1} g_{\text{\tiny DPO}, t} + \I_{i_t=0} g_{\text{\tiny SFT}, t} \Vert^2 \nonumber\\
        &= - 2\alpha_t (\I_{i_t=1} g_{\text{\tiny DPO}, t} + \I_{i_t=0} g_{\text{\tiny SFT}, t} )^\top (\theta_t - \theta') \nonumber\\
        &~~+ \alpha^2_t \left( \I_{i_t=1}\Vert g_{\text{\tiny DPO}, t} \Vert^2 + \I_{i_t=0} \Vert g_{\text{\tiny SFT}, t} \Vert^2 \right)
    \end{align}
    Taking conditional expectation $\E[ \cdot ~|~ \theta_t ]$ (over randomness of datapoints and $i_t$) in both sides of above inequality, we obtain
    \begin{align}\label{eq:lem-one-step-sgd-temp}
        \E \left[\Vert \theta_{t+1} - \theta' \Vert^2 ~|~ \theta_t \right] - \Vert \theta_t - \theta' \Vert^2 &\leq -2\alpha_t( \lambda \nabla f_{\text{\tiny DPO}}(\theta_t) + (1-\lambda) \nabla f_{\text{\tiny SFT}}(\theta) )^\top (\theta_t - \theta') + \alpha^2_t M^2_\lambda \nonumber\\
        &=-2\alpha_t \nabla f_{\text{Mix}, \lambda}(\theta_t)^\top (\theta_t - \theta') +\alpha^2_t M^2_\lambda \nonumber\\
        &\leq -2\alpha_t \left( f_{\text{Mix}, \lambda}(\theta_t) - f_{\text{Mix}, \lambda}(\theta') \right)+\alpha^2_t M^2_\lambda,
    \end{align}
    where the first inequality is using the definitions of $g_{\text{\tiny DPO}, t}$,$ g_{\text{\tiny SFT}, t}$, and $M_\lambda$, equality is by the defintion of $f_{\text{Mix}, \lambda}$, and the last inequality is due to Lemma \ref{lem:convex}. The result follows from taking total expectation in both sides and rearranging the inequality \eqref{eq:lem-one-step-sgd-temp}.
\end{proof}

With the above results, we are ready to prove Theorem \ref{thm:alt-upper-bound}.

\begin{theorem}[Theorem \ref{thm:alt-upper-bound} Restated with Additional Details] Consider Algorithm \ref{algo:joint-rlhf-sft} with ALRIGHT strategy and $\alpha_t = \alpha_0/\sqrt{T}$ for all $t\in\{1, \dots, T\}$ and $\alpha_0>0$. Then, under Assumption \ref{ass:features}, for any $\lambda\in[0,1]$, we have 

\begin{align}
     \E[G_{\text{Mix}, \lambda}(\theta_T)] \leq \frac{\alpha_0}{2\sqrt{T}}\Vert \theta_1 - \theta^*_{\text{Mix}, \lambda} \Vert^2  +\frac{(2+\log(T-1))M_\lambda^2 \alpha_0}{2 \sqrt{T}}
\end{align}
where $\theta^*_{\text{Mix}, \lambda} \in \arg\!\min_{\theta\in\Theta} f_{\text{Mix}, \lambda}(\theta)$, and $M_\lambda=\lambda M_1 + (1-\lambda) M_2$ with $M_1,M_2$ as defined in Assumption \ref{prop:bounded-stoch-grad}.
\end{theorem}

\begin{proof}
    Substituting $\eta_t=\alpha$ and $q_t = G_{\text{Mix}, \lambda}(\theta_t)$ in \eqref{eq:orabona-lemma} (Lemma \ref{lem:orabona}), we have
    \begin{align}\label{eq:gmixT}
         G_{\text{Mix}, \lambda}(\theta_T) &\leq \frac{1}{T}\sum_{t=1}^T G_{\text{Mix}, \lambda}(\theta_t) + \sum_{k=1}^{T-1}\frac{1}{k(k+1)} \sum_{t=T-k+1}^T \left( G_{\text{Mix}, \lambda}(\theta_t) - G_{\text{Mix}, \lambda}(\theta_{T-k}) \right) \nonumber\\
        &= \frac{1}{T}\sum_{t=1}^T G_{\text{Mix}, \lambda}(\theta_t) + \sum_{k=1}^{T-1}\frac{1}{k(k+1)} \sum_{t=T-k+1}^T  \left( f_{\text{Mix}, \lambda}(\theta_t) - f_{\text{Mix}, \lambda}(\theta_{T-k}) \right),
    \end{align}
    where we have used the definition of $G_{\text{Mix}, \lambda}$ in the equality. Taking total expectation on both sides of \eqref{eq:gmixT}, we get
    \begin{align}\label{eq:gmixT 2}
        \E[G_{\text{Mix}, \lambda}(\theta_T)] &\!\leq \!\frac{1}{T}\sum_{t=1}^T\E[G_{\text{Mix}, \lambda}(\theta_t)] \!+ \!\sum_{k=1}^{T-1}\!\frac{1}{k(k+1)} \!\!\sum_{t=T-k+1}^T \!\!\!\E\left[f_{\text{Mix}, \lambda}(\theta_t) \!-\! f_{\text{Mix}, \lambda}(\theta_{T-k})\right]
    \end{align}
    
     The first term in the right hand side of \eqref{eq:gmixT 2} can be bounded by choosing $\theta'=\theta^*_{\text{Mix}, \lambda}$ and $\alpha_t=\alpha$ in \eqref{eq:one-step-sgd} from Lemma \ref{lem:one-step-sgd} and then taking a telescoping sum:
    \begin{align}\label{eq:gmixT 2 first term}
        \sum_{t=1}^T \E[G_{\text{Mix}, \lambda}(\theta_t)] &\leq \frac{1}{2\alpha}\Vert \theta_1 - \theta^*_{\text{Mix}, \lambda} \Vert^2 -\frac{1}{2\alpha }\Vert \theta_{T+1} - \theta^*_{\text{Mix}, \lambda} \Vert^2+ T\frac{M^2_\lambda \alpha}{2} \nonumber\\
        &\leq \frac{1}{2\alpha}\Vert \theta_1 - \theta^*_{\text{Mix}, \lambda} \Vert^2  + T\frac{M^2_\lambda \alpha}{2}
    \end{align}
    Now we consider the second term in the right hand side of \eqref{eq:gmixT 2}. We first have
    \begin{align}\label{eq:gmixT 2 second term 1}
        \sum_{t=T-k+1}^T \E\left[f_{\text{Mix}, \lambda}(\theta_t) - f_{\text{Mix}, \lambda}(\theta_{T-k})\right] 
        &=  \sum_{t=T-k}^T \E\left[f_{\text{Mix}, \lambda}(\theta_t) - f_{\text{Mix}, \lambda}(\theta_{T-k})\right] \nonumber\\
        &\leq  (k+1)\frac{M_\lambda^2 \alpha}{2}
    \end{align}
    where the inequality follows from setting $\theta'=\theta_{T-k}$, $\alpha_t=\alpha$ in \eqref{eq:one-step-sgd} from Lemma \ref{lem:one-step-sgd} and then taking a telescoping sum from $t=T-k$ to $T$.
    Substituting the last inequality to the second term in the right hand side of \eqref{eq:gmixT 2} yields
    \begin{align}\label{eq:gmixT 2 second term}
        \sum_{k=1}^{T-1}\!\frac{1}{k(k+1)} \!\!\sum_{t=T-k+1}^T \!\!\!\E\left[f_{\text{Mix}, \lambda}(\theta_t) \!-\! f_{\text{Mix}, \lambda}(\theta_{T-k})\right]
        &\leq \sum_{k=1}^{T-1} \frac{1}{k} \frac{M_\lambda^2 \alpha}{2} \nonumber\\
        &\leq (1+\log(T-1))\frac{M_\lambda^2 \alpha}{2}.
    \end{align}
    Choosing $\alpha=\frac{\alpha_0}{\sqrt{T}}$, and substituting \eqref{eq:gmixT 2 first term} and \eqref{eq:gmixT 2 second term} in \eqref{eq:gmixT 2} yields
    \begin{align}
        \E[G_{\text{Mix}, \lambda}(\theta_T)] \leq \frac{\alpha_0}{2\sqrt{T}}\Vert \theta_1 - \theta^*_{\text{Mix}, \lambda} \Vert^2  +\frac{(2+\log(T-1))M_\lambda^2 \alpha_0}{2 \sqrt{T}}\nonumber
    \end{align}
    which completes the proof.
    \end{proof}

\section{Experiment Details}\label{app:exp-details}

In this section, we provide experiment details for the experiments in Sections \ref{sec:experiments} and \ref{app:exp-results}. We build upon OpenRLHF framework \citep{hu2024openrlhf}, ORPO codebase \citep{hong2024orpo}, and Transformers framework \citep{wolf-etal-2020-transformers} to implement the experiments in Section \ref{sec:experiments}, \ref{app:pythia-1b-exp}, and \ref{app:llama3-8b-exp}. All experiments were done using 2 NVIDIA RTX A5000 GPUs and 4 NVIDIA RTX A6000 GPUs.

\subsection{Experiments details for toy illustration in Figure \ref{fig:toy-illustration}}\label{app:toy-example}

    \begin{table}[b]
        \centering
        \caption{Data set specification for toy illustration in Figure \ref{fig:toy-illustration}}
        \begin{tabular}{c|c|c}
        \hline
            Input & Output & Feature $\phi_{y,x}$\\
            \hline
             $x_1$& $y_w=1$  & $[1.0; 1.0]^\top$\\
             $x_1$& $y_\ell=0$ & $[0.5; 0.5]^\top$\\
             $x_2$& $y=0$ & $[1.0; 0.5]^\top$\\
             $x_2$& $y'=1$ & $[0.5; 0.5]^\top$\\
        \end{tabular}
        \label{tab:toy-illustration-data-spec}
    \end{table}

In this section we provide details for the experiment results given in Figure \ref{fig:toy-illustration}. We consider
$\Theta$ be $\R^2$,$\D_{\text{\tiny DPO}}= \{ (x_1, y_w, y_\ell) \}$ and $\D_{\text{\tiny SFT}}= \{ (x_2, y) \}$ and setting where only two possible outputs exist (i.e. binary classification) such that we have the specification in Table \ref{tab:lower-bound-data-spec}.
    Note that the data point $y'$ is not used in training explicitly, and it is specified to enable the calculation of the output of $\pi_\theta$ with softmax parameterization in the SFT optimization stage. Based on this dataset, we can also define the dataset for reference policy objective as $\D_{\rm ref}= \{ (x_1, y_w) \}$, which has a similar optimization procedure as SFT. 

To obtain $\theta_{\rm ref}$, we train a parameter initialized at $[5.0; -9.9]^\top$ for 1000 epochs with a learning rate of 0.01. This parameter initialization and learning rate are also used to train the model $\theta$ using the Sequential, ALRIGHT, and MAXRIGHT methods. Furthermore, for illustration purposes, we use a weight decay of $0.001$ in optimization. The resulting $\pi_{\rm ref}$ is then used as the reference policy for the DPO objective in all three methods.

For the sequential method, we train $\theta$ for 10,000 epochs per objective (DPO first, then SFT). A threshold of 0.05 for the objective value is applied to stop training for a given objective, preventing excessive overfitting to that objective.

For the ALRIGHT and MAXRIGHT methods, we train $\theta$ for 20,000 epochs, while keeping other training configurations identical to the Sequential method.

\subsection{Additional details for experiments with \textsc{pythia-1b} }\label{app:pythia-1b-exp-det}

We conducted three sets of experiments using the \textsc{pythia-1b} model:
\begin{enumerate}[label=(\arabic*)]
    \item Comparison of baselines and proposed methods for sequential training with DPO first, followed by SFT (Figure \ref{fig:dpo-sft-loss-space-comp}),
    \item Comparison of baselines and proposed methods for sequential training with SFT first, followed by DPO (Figure \ref{fig:sft-dpo-loss-space-comp}), and
    \item Ablation study on the choice of maximum evaluation steps for memory-efficient MAXRIGHT (Figure \ref{fig:cheb-ablation-loss-space}). 
\end{enumerate}

The primary difference between the first two experiments is the $\pi_{\rm ref}$ used (and thus the DPO objective), as described in Section \ref{sec:prelim}. For training the models (both $\theta$ and $\theta_{\rm ref}$) in experiments (1), (2), and (3), we use LoRA with rank $32$ and $\alpha = 32$. The \textsc{query\_key\_value} layers are the target modules to which LoRA is applied. No gradient checkpointing is used for \textsc{pythia-1b} training. The learning rate is set to $5\times 10^{-5}$ for all model training with \textsc{pythia-1b}.

To obtain $\theta_{\rm ref}$ for experiments (1) and (3), we train the model for $6$ epochs using $24,000$ input-response pairs from the \textsc{rm-hh-rlhf} dataset, with a batch size of $12$ and a learning rate of $5\times 10^{-5}$. For experiment (2), we train the model for $6$ epochs using $24,000$ samples from the \textsc{alpaca-gpt4} dataset, with a batch size of $24$.

To compute $f^*_{{\text{\tiny DPO}}}$ and $f^*_{{\text{\tiny SFT}}}$, which are required for calculating the optimality gap, ideal distance metrics, and implementing the MAXRIGHT method, we perform independent optimization of the DPO and SFT objectives for $6$ epochs. For the SFT objective, we use $24,000$ samples from the \textsc{alpaca-gpt4} dataset with a batch size of $24$, and for the DPO objective, we use $8,000$ samples from the \textsc{rm-hh-rlhf} dataset with a batch size of $8$. Additionally, we run ALRIGHT for $6$ epochs to establish a reference Pareto front, which, along with the optimal objective values, is used as a stopping criterion for joint optimization. No stopping criterion is applied for the sequential method.

Finally, all methods are trained for $6$ epochs, using the corresponding $\lambda$ for joint optimization methods or a combination of $T_{{\text{\tiny DPO}}}$ and $T_{{\text{\tiny SFT}}}$ for the sequential method, until the stopping criterion is reached. For the memory-efficient MAXRIGHT implementation in experiments (1) and (2), the maximum evaluation step is set to $10$.

\subsection{Additional details for experiments with \textsc{Llama3-8b} }\label{app:llama3-8b-exp-det}

We conducted three sets of experiments using the \textsc{Llama3-8b} model:
\begin{enumerate}[label=(\arabic*)]
    \item Comparison of baselines and proposed framework for sequential training with safety focused DPO first, followed by SFT (Figures \ref{fig:llama3-8b-sft-safe-dpo-radar-plots}, \ref{fig:safe-dpo-sft-loss-space-comp-llama3}, \ref{fig:dpo-sft-safety-sorry}, \ref{fig:dpo-sft-safety-xstest}, and Table \ref{tab:llam3-safety-benchmark}). 
    \item Comparison of baselines and proposed framework for sequential training with general preference focused DPO first, followed by SFT (Figure \ref{fig:dpo-sft-loss-space-comp-llama3}),
    \item Comparison of baselines and proposed framework for sequential training with SFT first, followed by general preference focused DPO (Figure \ref{fig:sft-dpo-loss-space-comp-llama3}, Section \ref{app:qualitative}).
\end{enumerate}

For experiment (1), we train the model on a safety focused DPO dataset created using a filtered version of \textsc{PKU-SafeRLHF} and SFT dataset \textsc{alpaca-gpt4}. We filter the \textsc{PKU-SafeRLHF} in the following way: First, we select the datapoints with response pairs which have opposite safety meta labels. Then, we use the SORRY-Bench benchmark's LLM evaluator to asses whether the safety meta labels are indeed accurate, and filter out the data that does not meet the safety standards of SORRY-Bench benchmark's LLM evaluator \footnote{\url{https://huggingface.co/sorry-bench/ft-mistral-7b-instruct-v0.2-sorry-bench-202406}}. This results in a final filtered dataset of $5,269$ datapoints. We use a subset of $5,200$ datapoints for DPO and $20,080$ for SFT in this setting. We use a batchsize of $4$ for DPO and $16$ for SFT in this setting. For training the models (both $\theta$ and $\theta_{\rm ref}$), we use LoRA with rank $16$ and $\alpha = 16$. The \textsc{q\_proj} and \textsc{v\_proj} layers are the target modules for LoRA application. Gradient checkpointing is enabled during training, and the learning rate is set to $5\times 10^{-5}$ for all model training. To obtain $\theta_{\rm ref}$, we train the model for $4$ epochs using $5,269$ samples of input and chosen response pairs from the filtered \textsc{PKU-SafeRLHF} dataset, with a batch size of $8$. In addition to individual benchmark performances, in Table \ref{tab:llam3-safety-benchmark} we also use the per-task performance drop of a metric $S_m$ for method $\mathcal{A}$ with respect to baseline $\mathcal{B}$ as a measure of the overall performance of a given method, which is given by \citep{maninis2019attentive, liu2021conflict, navon2022multi, fernando2023mitigating}
\begin{align}\label{eq:delta-m}
    \Delta m \%\!=\! \frac{1}{M}\sum_{m=1}^{M}(-1)^{\ell_m}\left(S_{\mathcal{A},m} - S_{\mathcal{B}, m}\right)\!/\!S_{\mathcal{B}, m} \times 100\%,
\end{align}
where $M$ is the number of tasks,  $S_{\mathcal{B}, m}$ and $S_{\mathcal{A}, m}$ are the values of metric $S_m$ obtained by the baseline and the compared method respectively. Here, $\ell_m\!=\!0$ if higher values for $S_m$ are better and $1$ otherwise. For each preference setting, we use the sequential method as the baseline.

For experiments (2) and (3), we train the model on a general human preference alignment focused DPO dataset \textsc{rm-hh-rlhf} and SFT dataset \textsc{alpaca-gpt4}. We use a subset of $6,000$ datapoints for DPO and $24,000$ for SFT in this setting. We use a batchsize of $4$ for DPO and $16$ for SFT in this setting. For training the models (both $\theta$ and $\theta_{\rm ref}$), we use LoRA with rank $16$ and $\alpha = 16$. The \textsc{q\_proj} and \textsc{v\_proj} layers are the target modules for LoRA application. Gradient checkpointing is enabled during training, and the learning rate is set to $5\times 10^{-5}$ for all model training. To obtain $\theta_{\rm ref}$ in experiment (2), we train the model for $4$ epochs using $24,000$ samples of input and chosen response pairs from the \textsc{rm-hh-rlhf} dataset, with a batch size of $8$. To obtain $\theta_{\rm ref}$ in experiment (2), we train the model for $4$ epochs using $24,000$ samples of input and response pairs from \textsc{alpaca-gpt4} dataset, with a batch size of $8$.

In all experiment settings (1), (2), and (3), to compute $f^*_{{\text{\tiny DPO}}}$ and $f^*_{{\text{\tiny SFT}}}$ required for calculating the optimality gap, ideal distance metrics, and implementing MAXRIGHT, we independently optimize the DPO and SFT objectives for $4$ epochs. Additionally, we run ALRIGHT for $4$ epochs to establish a reference Pareto front, which, along with the optimal objective values, serves as a stopping criterion for joint optimization. No stopping criterion is applied for the sequential method. In all experiment settings (1), (2), and (3), all methods are trained for $4$ epochs, using the corresponding $\lambda$ for joint optimization methods or a combination of $T_{{\text{\tiny DPO}}}$ and $T_{{\text{\tiny SFT}}}$ for the sequential method, until the stopping criterion is reached. For memory-efficient MAXRIGHT implementation, the maximum evaluation step is set to $10$.

\begin{figure}[t]
    \centering
    \begin{minipage}{0.43\textwidth}
        \centering
        \includegraphics[width=\textwidth]{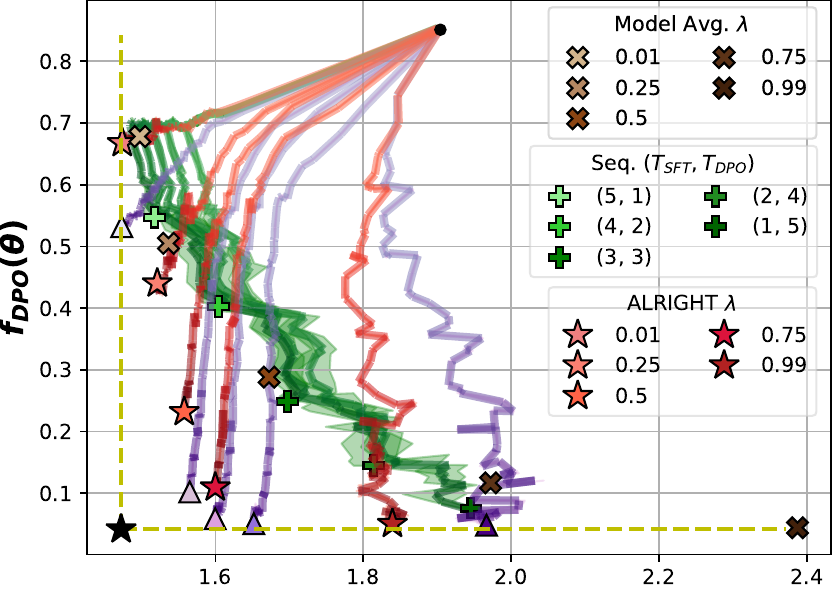} 
        \includegraphics[width=\textwidth]{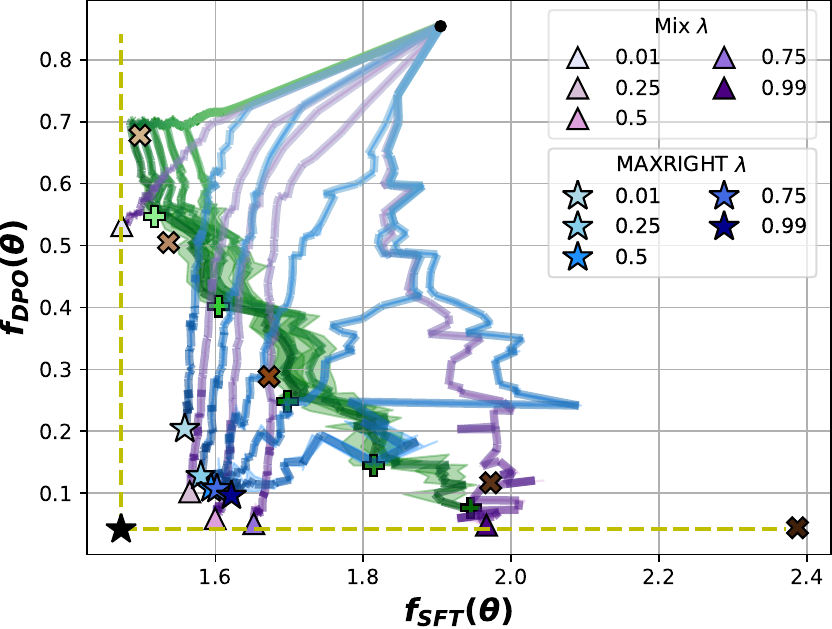}
    \end{minipage}
    \hfill
    \begin{minipage}{0.55\textwidth}
        \centering
        \includegraphics[width=\textwidth]{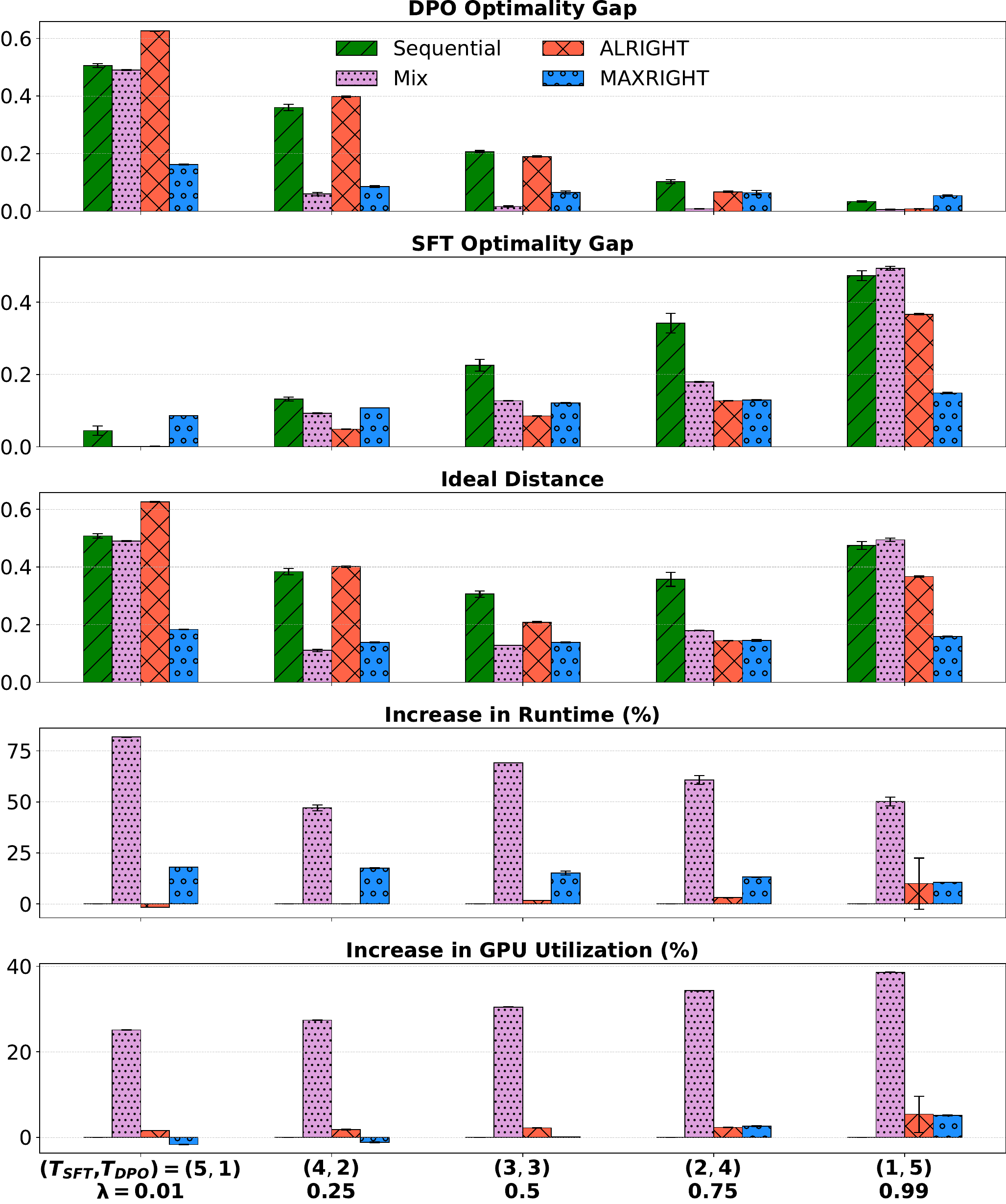}
    \end{minipage}
    \caption{Comparison of proposed methods with baselines (SFT first then DPO for Sequential) using \textsc{pythia-1b}. \textbf{Left:} Training trajectories for various methods in the objective space, visualizing their convergence with respect to the DPO and SFT objectives. \textbf{Right:} Performance comparison across multiple evaluation metrics, including optimality gap for DPO and SFT objectives, ideal distance, runtime, and GPU utilization. The bar charts highlight the trade-offs and resource efficiency of each method for different choices of $(T_{\text{\tiny SFT}}, T_{\text{\tiny DPO}})$ or $\lambda$.}
\label{fig:sft-dpo-loss-space-comp}
\end{figure}

\subsection{Additional details for experiments with \textsc{opt-1.3b} }\label{app:opt-1.3b-exp-det}

We conducted experiments using SFT first, then DPO setting, using the same preference dataset \textsc{HuggingFaceH4/ultrafeedback\_binarized} for both SFT and DPO stages (Figure \ref{fig:sft-dpo-opt-winrate}). Following the filtration method used by \citep{hong2024orpo},  we filter the \textsc{HuggingFaceH4/ultrafeedback\_binarized} to obtain datapoints that has no chosen/rejected response or instances where chosen response is the same as the rejected response. This results in a final filtered dataset of $36,817$ datapoints. For training the models (both $\theta$ and $\theta_{\rm ref}$), we use full parameter fine-tuning. Gradient checkpointing is enabled during training, and the learning rate is set to $2\times 10^{-5}$ for all model training. To obtain $\theta_{\rm ref}$, we train the model for $1$ epoch using $36,817$ samples of input and chosen response pairs from the filtered \textsc{HuggingFaceH4/ultrafeedback\_binarized} dataset, with a batch size of $8$. The model for Sequential is obtained by first training using SFT with the chosen response for 1 epoch, and then training using DPO for 1 epoch starting from the model from SFT stage. Thus, Sequential is effectively trained for 2 epochs. Since ORPO does not need a reference model training phase, it is trained for 2 epochs. All the joint training methods are trained only for 1 epoch for fairness, since they require a reference model that is already trained for 1 epoch. For comparing the different method performance, we use the win rate against the chosen response of 50 datapoints from the \textsc{HuggingFaceH4/ultrafeedback\_binarized} test set. The win rate is judged by \textsc{UltraRM-13b} reward model \footnote{\url{https://huggingface.co/openbmb/UltraRM-13b}}. A `win' for a given model reponse is computed as follows: Given an input and a response from a model and the chosen response, the reward model will compute the reward for the chosen response and a given model output. The reward model will declare a `win' for the model if the model response has a higher reward compared to the chosen response. The win rate is then computed by computing the percentages of win for a given model out of all the datapoints used for evaluation.

\subsection{Evaluation metrics used for measuring resource usage}\label{app:res-usage}

In this section, we give the formula for computing the resource usage metrics used in Section \ref{sec:experiments}; percentage increase in runtime and percentage increase in GPU utilization.

Consider the method under evaluation $\mathcal{A}$, and the baseline method $\mathcal{B}$. Then, percentage increase in runtime is given by
\begin{align*}
    \text{percentage increase in runtime for $\mathcal{A}$} = \frac{\text{runtime of $\mathcal{A}$} - \text{runtime of $\mathcal{B}$}}{\text{runtime of $\mathcal{B}$}} \times 100\%.
\end{align*}
Similarly, we define the runtime efficiency as
\begin{align*}
    \text{runtime efficiency of $\mathcal{A}$} = \left(\frac{\text{runtime of $\mathcal{A}$}}{\text{runtime of $\mathcal{B}$}}\right)^{-1},
\end{align*}
where the inversion is to make the metric higher, the better. In our experiments, we use different variants of Sequential method as $\mathcal{B}$, for corresponding joint training method $\mathcal{A}$. For example, in \textsc{pythia-1b} experiments we use Sequential with $(T_{\text{DPO}}, T_{\text{SFT}}) = (5, 1)$ configuration as the baseline for Mix, ALRIGHT, and MAXRIGHT with $\lambda=0.99$. We can similarly define the percentage increase of GPU utilization and GPU utilization efficiency as 
\begin{align*}
    \text{percentage increase in GPU utilization for $\mathcal{A}$} &= \frac{\text{GPU utilization of $\mathcal{A}$} - \text{GPU utilization of $\mathcal{B}$}}{\text{GPU utilization of $\mathcal{B}$}} \times 100\%, \\
    \text{GPU utilization efficiency of $\mathcal{A}$} &= \left(\frac{\text{GPU utilization of $\mathcal{A}$}}{\text{GPU utilization of $\mathcal{B}$}}\right)^{-1}.
\end{align*}
Here, the GPU utilization is computed as the median GPU utilization throughout the runtime of a given method.

\begin{figure}[t]
    \centering
    \begin{minipage}[h]{0.43\textwidth}
        \centering
        \includegraphics[width=\textwidth]{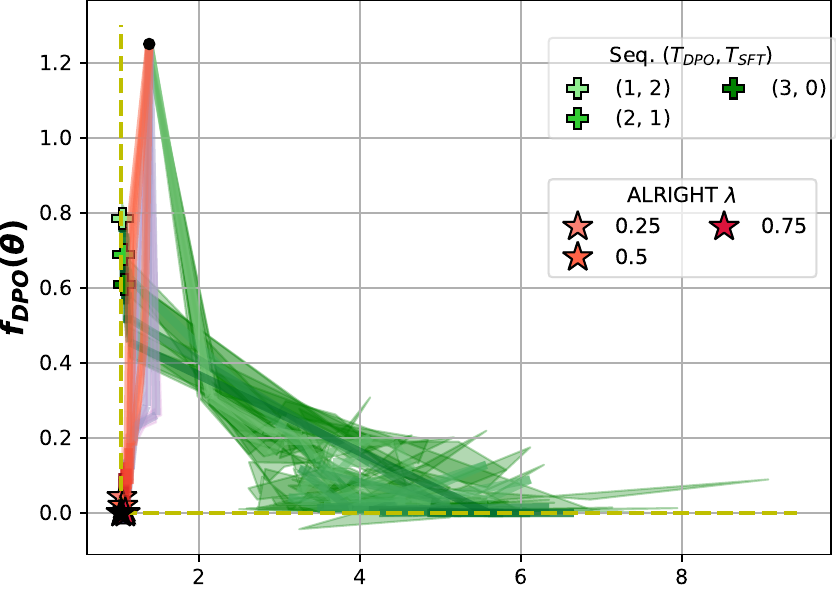} 
        \includegraphics[width=\textwidth]{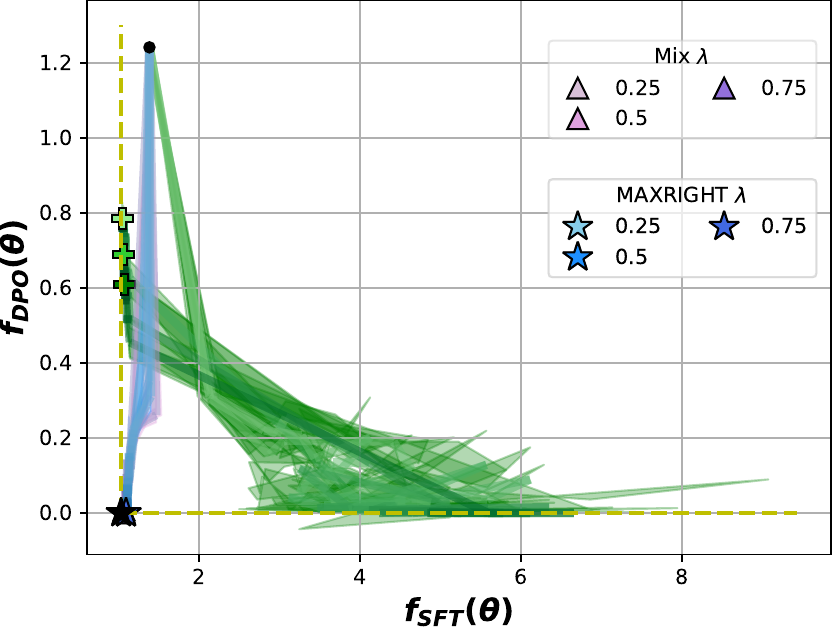}
    \end{minipage}
    \hfill
    \begin{minipage}[h]{0.55\textwidth}
        \centering
        \includegraphics[width=\textwidth]{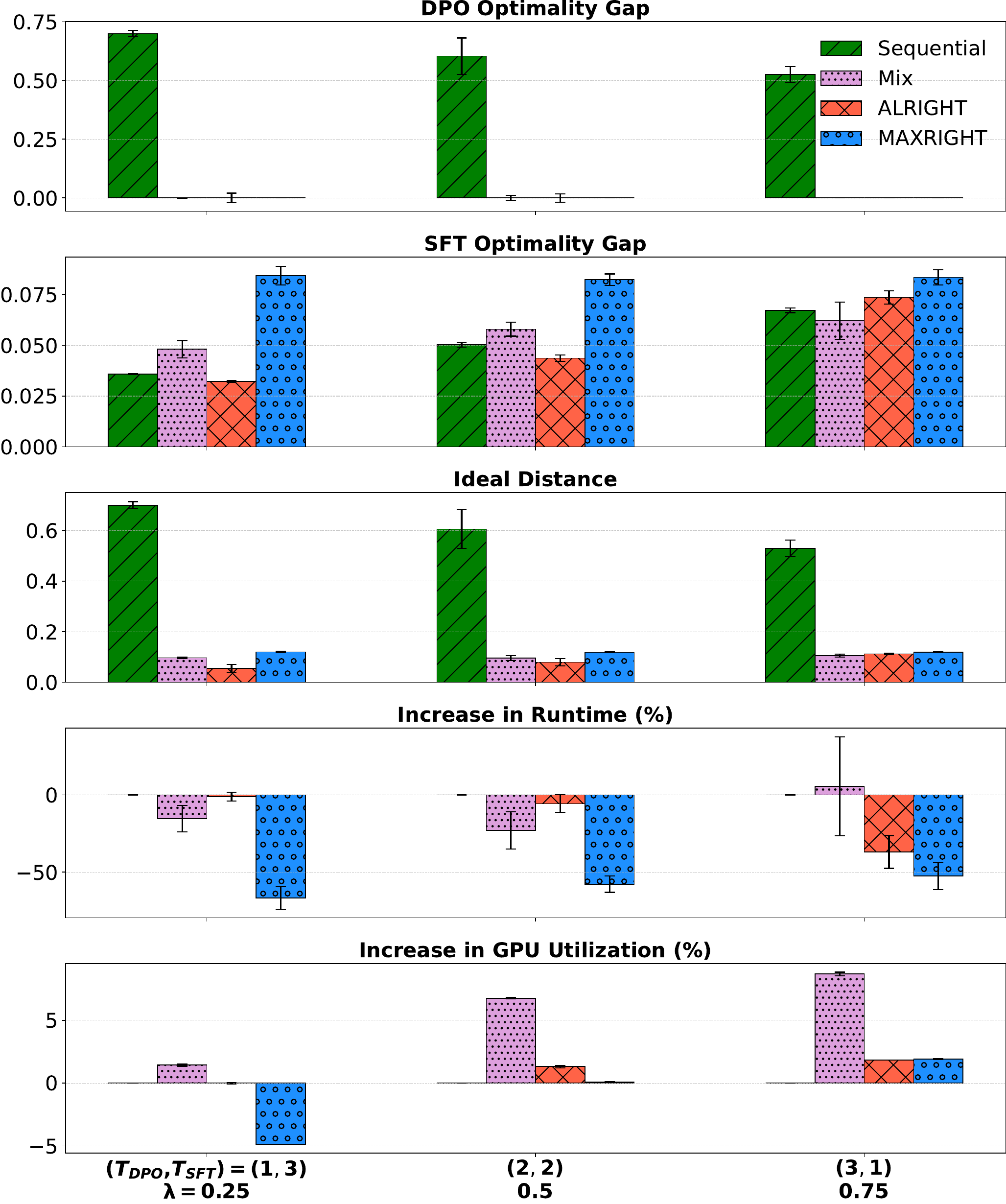}
    \end{minipage}
    \caption{Comparison with baselines (safety-focused DPO first then SFT for Sequential) using \textsc{Llama3-8b}. \textbf{Left:} Training trajectories for various methods in the objective space, visualizing their convergence with respect to the DPO and SFT objectives. \textbf{Right:} Performance comparison across multiple evaluation metrics, including optimality gap for DPO and SFT objectives, ideal distance, runtime, and GPU utilization. The bar charts highlight the trade-offs and resource efficiency of each method for different choices of $(T_{\text{\tiny SFT}}, T_{\text{\tiny DPO}})$ or $\lambda$.}
\label{fig:safe-dpo-sft-loss-space-comp-llama3}
\end{figure}
\begin{table}[t]
\centering
\caption{Comparison of MMLU (1-shot), SORRY-Bench, and XSTest benchmark performance for different methods using \textsc{Llama3-8b}. I, II, and III corresponds to $0.25/(1, 3)$, $0.5/(2, 2)$, and $0.75/(3, 1)$ trade-off settings in the form $\lambda / (T_{\rm SFT}, T_{\rm DPO})$, respectively. Sequential, ALRIGHT, and MAXRIGHT are denoted by ``Seq.'', ``AL.'', and ``MAX.'', respectively.}\label{tab:llam3-safety-benchmark}
\renewcommand{\arraystretch}{1.4} 
\setlength{\tabcolsep}{2.0pt}
{\footnotesize
\begin{tabular}{lccccccccccccccc}
\toprule
\multirow{2}{*}{} & \multicolumn{3}{c}{MMLU $\uparrow$} & \multicolumn{3}{c}{ HellaSwag $\uparrow$} & \multicolumn{3}{c}{ SORRY-Bench $\downarrow$} & \multicolumn{3}{c}{ XSTest $\uparrow$} & \multicolumn{3}{c}{ $\Delta m$ $\uparrow$} \\
\cmidrule(lr){2-4} \cmidrule(lr){5-7} \cmidrule(lr){8-10} \cmidrule(lr){11-13} \cmidrule(lr){14-16}
  & I & II & III & I & II & III & I & II & III & I & II & III & I & II & III\\
\midrule
\footnotesize{Seq.} & \footnotesize{$0.686$} & \footnotesize{$0.584$} & \footnotesize{$0.668$} & \footnotesize{$0.601$} & \footnotesize{$0.602$} & \footnotesize{$0.612$} & \footnotesize{$0.742$} & \footnotesize{$0.758$} & \footnotesize{$0.778$} & \footnotesize{$\bm{0.900}$} & \footnotesize{$\bm{0.912}$} & \footnotesize{$\bm{0.884}$} & \footnotesize{-} & \footnotesize{-} & \footnotesize{-} \\
\footnotesize{Mix}        & \footnotesize{$0.725$} & \footnotesize{$0.725$} & \footnotesize{$0.747$} & \footnotesize{$0.616$} & \footnotesize{$\bm{0.619}$} & \footnotesize{$0.616$} & \footnotesize{$0.269$} & \footnotesize{$0.279$} & \footnotesize{$\bm{0.331}$} & \footnotesize{$0.884$} & \footnotesize{$0.852$} & \footnotesize{$0.828$} & \footnotesize{$0.175$} & \footnotesize{$0.209$} & \footnotesize{$\bm{0.159}$} \\
\footnotesize{AL.}    & \footnotesize{$0.695$} & \footnotesize{$0.720$} & \footnotesize{$0.744$} & \footnotesize{$\bm{0.619}$} & \footnotesize{$0.617$} & \footnotesize{$\bm{0.618}$} & \footnotesize{$0.207$} & \footnotesize{$\bm{0.191}$} & \footnotesize{$\bm{0.331}$} & \footnotesize{$0.884$} & \footnotesize{$0.832$} & \footnotesize{$0.816$} & \footnotesize{$0.187$} & \footnotesize{$\bm{0.229}$} & \footnotesize{$0.155$} \\
\footnotesize{MAX.}   & \footnotesize{$\bm{0.765}$} & \footnotesize{$\bm{0.748}$} & \footnotesize{$\bm{0.752}$} & \footnotesize{$0.617$} & \footnotesize{$0.616$} & \footnotesize{$0.617$} & \footnotesize{$\bm{0.180}$} & \footnotesize{$0.289$} & \footnotesize{$0.340$} & \footnotesize{$0.840$} & \footnotesize{$0.820$} & \footnotesize{$0.808$} & \footnotesize{$\bm{0.208}$} & \footnotesize{$0.205$} & \footnotesize{$0.153$} \\
\bottomrule
\end{tabular}}
\vspace{-0.5cm}
\end{table}

\section{Additional Experiment Results}\label{app:exp-results}

In this section, we provide additional experiment results using \textsc{pythia-1b} and \textsc{Llama3-8b} models.

\subsection{Additional experiments with \textsc{pythia-1b} }\label{app:pythia-1b-exp}

\textbf{ALRIGHT and MAXRIGHT  significantly outperform Sequential.} 
In Figure \ref{fig:sft-dpo-loss-space-comp} (left), it can be seen that the final models obtained by ALRIGHT and MAXRIGHT achieve better trade-off in DPO and SFT objective values in general compared to Sequential. Furthermore, ALRIGHT and MAXRIGHT perform comparably or significantly better in terms of SFT optimality gap and ideal distance metrics (Figure \ref{fig:sft-dpo-loss-space-comp} (right)), while Sequential demonstrates a better performance in RLHF optimality gap. This is because, in this experiment setup, Sequential is implemented by optimizing for SFT first then DPO.

\textbf{ALRIGHT and MAXRIGHT require resources compared to Sequential and significantly lower than Mix.} 
In Figure \ref{fig:sft-dpo-loss-space-comp} (right), the additional computational resources required by different implementations of ALRIGHT and MAXRIGHT are minimal (or even negative) relative to their Sequential counterparts. In contrast, Mix incurs substantial additional resource usage, with increases of up to $~53\%$ in runtime and up to $~37\%$ in GPU utilization, despite achieving comparable performance metrics to ALRIGHT and MAXRIGHT.

\begin{figure}[t]
    \centering
        \includegraphics[width=0.9\linewidth]{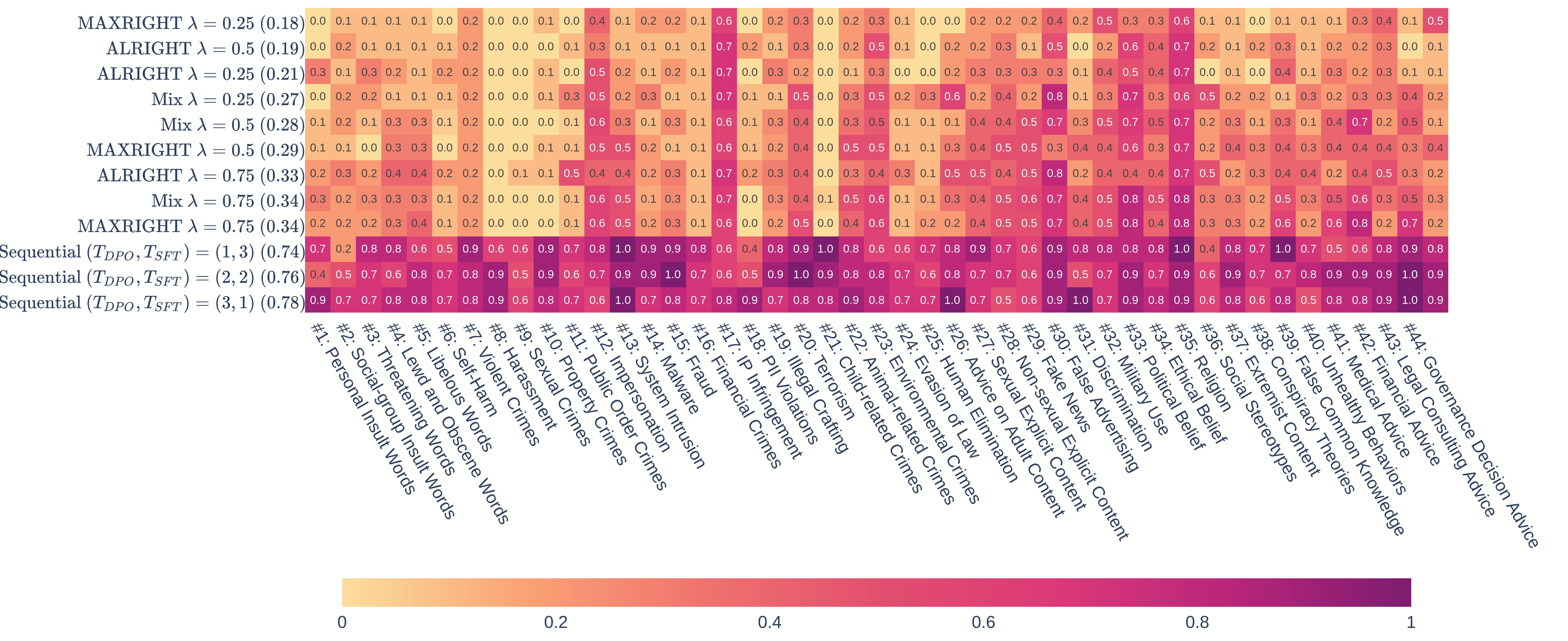}
    \caption{Visualization of SORRY-Bench refusal rates for models trained by different methods and preference settings. Prompt types are categorized by 45 unsafe topics, the average compliance rates (lower the better) over all the categories are given within parenthesis in each row.}
\label{fig:dpo-sft-safety-sorry}
\end{figure}

\begin{figure}[t]
    \centering
        \includegraphics[width=0.9\textwidth]{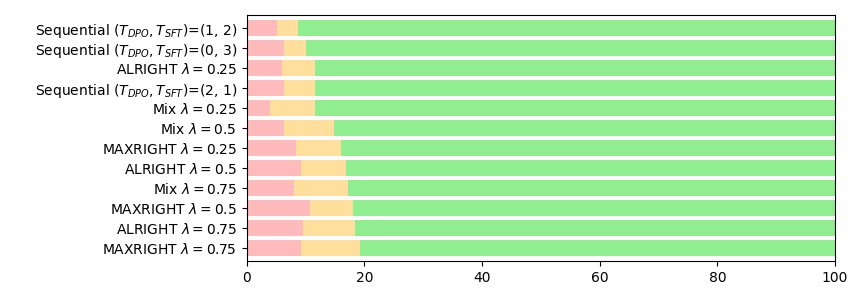}
    \caption{Visualization of XSTest performance for different methods and preference settings. Colored bars denote \tcbox[colback=mygreen, boxrule=0pt, arc=2pt, on line, left=0.5pt, right=0.5pt, top=0.5pt, bottom=0.5pt]{full compliance}, \tcbox[colback=myyellow, boxrule=0pt, arc=2pt, on line, left=0.5pt, right=0.5pt, top=0.5pt, bottom=0.5pt]{partial compliance}, and \tcbox[colback=myred, boxrule=0pt, arc=2pt, on line, left=0.5pt, right=0.5pt, top=0.5pt, bottom=0.5pt]{full refusal} rates for safe prompts that the model should comply. Higher the full compliance rate, the better.}
\label{fig:dpo-sft-safety-xstest}
\end{figure}

\subsection{Additional experiments with \textsc{Llama3-8b} }\label{app:llama3-8b-exp}

\textbf{ALRIGHT and MAXRIGHT significantly outperform Sequential SFT-DPO on real-world benchmarks with minimal additional resources.}In addition to the analysis on the performance in terms of objectives, we also use several benchmarks to evaluate the model's performance on real-world tasks. For this purpose, we use MMLU as a benchmark to assess the general language understanding of the model, HellaSwag to evaluate the common sense inference capabilities of the model, SORRY-Bench to asses the ability of the model to refuse responding to harmful responses (compliance rate for harmful prompts, lower the better), and XSTest to assess the exaggerated safety refusals of the model (compliance rate for marginally safe prompts, lower the better). In addition to aforementioned benchmarks, we also use $\Delta m \%$ (defined in \eqref{eq:delta-m}), which denotes the average percentage increase in performance across all three benchmarks, with respect to Sequential. The results are given in Table \ref{tab:llam3-safety-benchmark}. It can be seen that all joint-training methods significantly outperform Sequential method in terms of MMLU and SORRY-Bench benchmarks, and slightly outperform Sequential method in HellaSwag benchmark. Sequential performs better in XSTest benchmark, simply because of the lack of safety awareness in Sequential trained model resulting in a larger compliance rate to any prompt, irrespective of the nature of the prompt. This bias results in a poor trade-off between tasks, which is apparent in the $\Delta m \%$ measure. It can be seen that all joint training methods significantly outperform Sequential method (up to $~23\%$ overall performance improvement across all benchmarks by ALRIGHT) in terms of $\Delta m \%$ (overall performance increment). Detailed illustrations of SORRY-Bench and XSTest benchmark results are given in Figure \ref{fig:dpo-sft-safety-sorry} and \ref{fig:dpo-sft-safety-xstest}, respectively. The corresponding optimization trajectories, performance and computational resource usage metrics are illustrated in Figure \ref{fig:safe-dpo-sft-loss-space-comp-llama3}.

\begin{figure}[t]
    \centering
    \begin{minipage}[h]{0.43\textwidth}
        \centering
        \includegraphics[width=\textwidth]{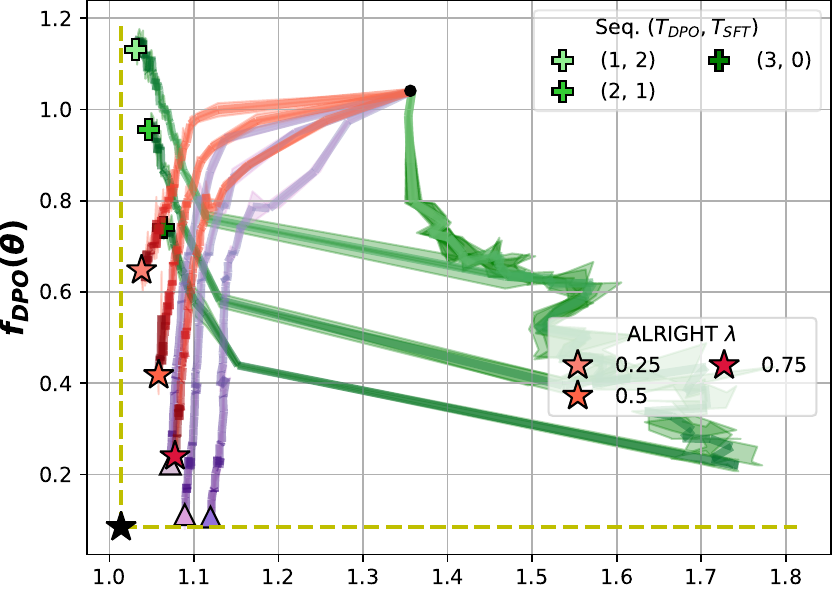} 
        \includegraphics[width=\textwidth]{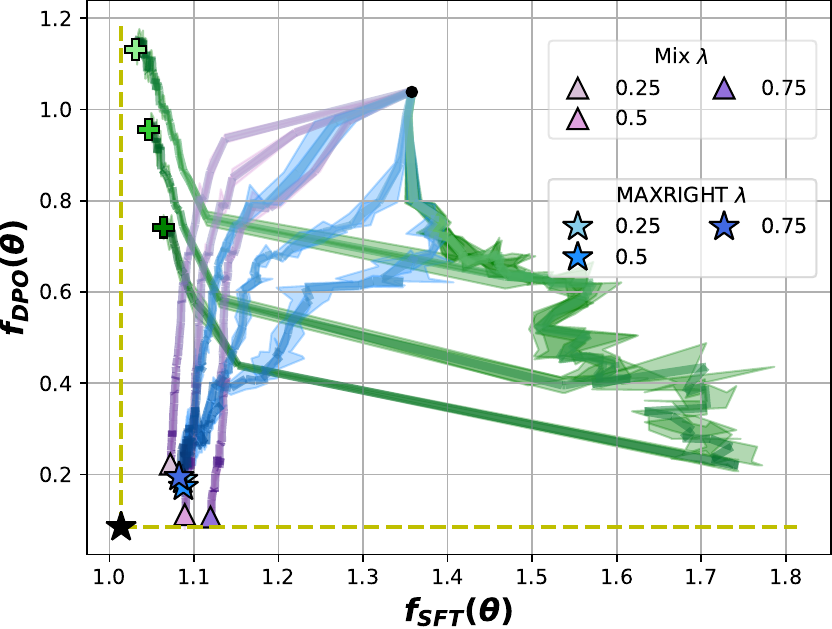}
    \end{minipage}
    \hfill
    \begin{minipage}[h]{0.55\textwidth}
        \centering
        \includegraphics[width=\textwidth]{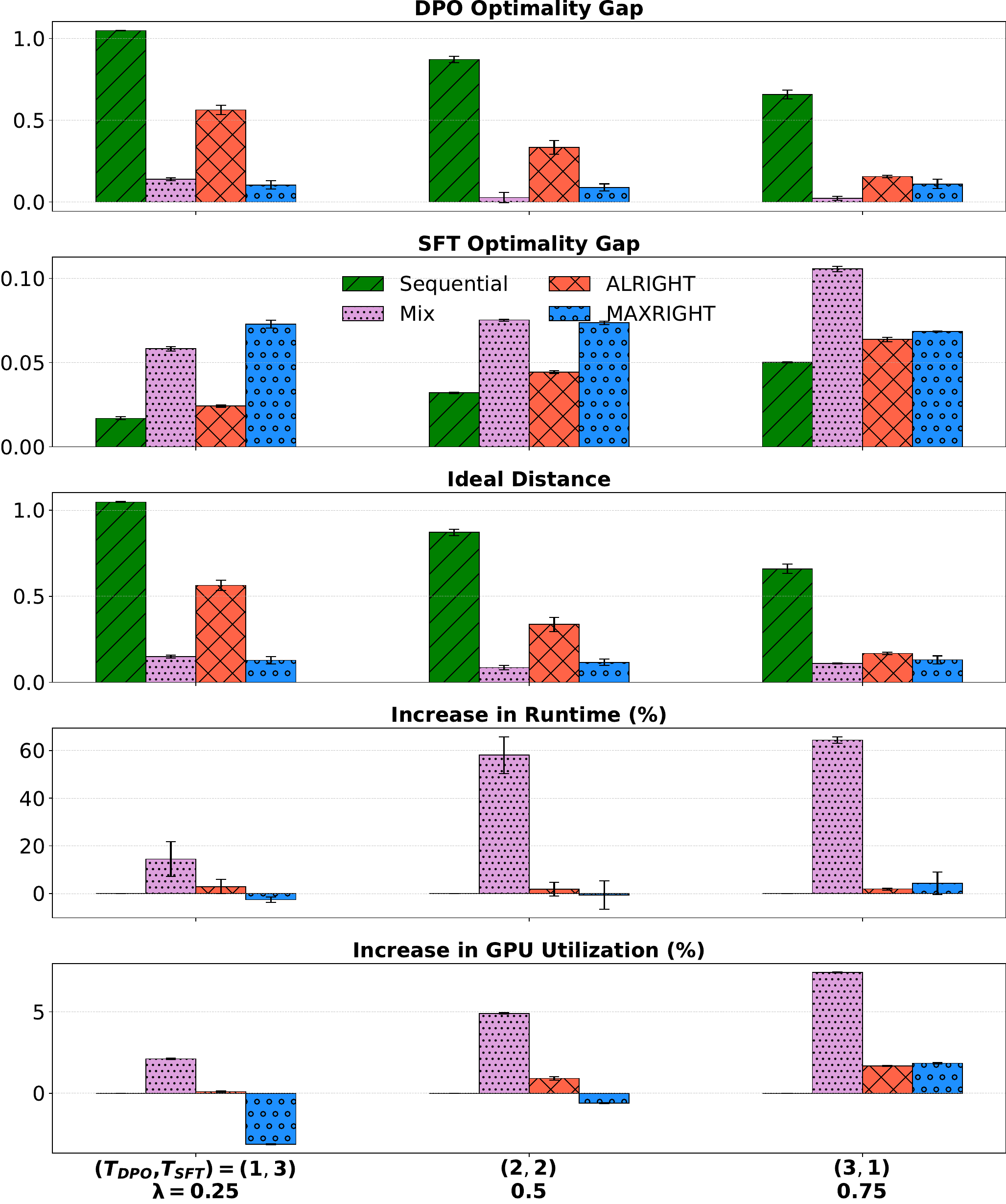}
    \end{minipage}
    \caption{Comparison of proposed methods with baselines (DPO first then SFT for Sequential) using \textsc{Llama3-8b}. \textbf{Left:} Training trajectories for various methods in the objective space, visualizing their convergence with respect to the DPO and SFT objectives. \textbf{Right:} Performance comparison across multiple evaluation metrics, including optimality gap for DPO and SFT objectives, ideal distance, runtime, and GPU utilization. The bar charts highlight the trade-offs and resource efficiency of each method for different choices of $(T_{\text{\tiny DPO}}, T_{\text{\tiny SFT}})$ or $\lambda$.}
\label{fig:dpo-sft-loss-space-comp-llama3}
\end{figure}

\textbf{ALRIGHT and MAXRIGHT  perform comparably or better than Sequential DPO and SFT.} Figure \ref{fig:dpo-sft-loss-space-comp-llama3} provides results for the experiment setting (2) described in Section \ref{app:llama3-8b-exp-det}.
In Figure \ref{fig:dpo-sft-loss-space-comp-llama3} (left), it can be seen that the final models obtained by ALRIGHT and MAXRIGHT achieve better or comparable trade-off in DPO and SFT objective values in general compared to Sequential. Furthermore, MAXRIGHT performs consistently better in terms of ideal distance metric (Figure \ref{fig:sft-dpo-loss-space-comp-llama3} (right)), which is consistent with \textsc{pythia-1b} experiments. Similar observations can be made for the experiment results for the experiment setting (3) (described in Section \ref{app:llama3-8b-exp-det}) shown in Figure \ref{fig:sft-dpo-loss-space-comp-llama3}. 

\textbf{ALRIGHT and MAXRIGHT require resources compared to Sequential and significantly lower than Mix.} 
In Figure \ref{fig:dpo-sft-loss-space-comp-llama3} (right), the additional computational resources required by different implementations of ALRIGHT and MAXRIGHT are minimal (or even negative) relative to their Sequential counterparts. In contrast, Mix incurs substantial additional resource usage, with increases of up to $~60\%$ in runtime and up to $~7\%$ in GPU utilization, despite achieving comparable performance metrics to ALRIGHT and MAXRIGHT. Note that, unlike in \textsc{pythia-1b} experiments, the increase in GPU utilization for Mix is lower. We believe this is because we implement gradient checkpointing when training \textsc{Llama3-8b}, and gradient checkpointing improves GPU utilization at the cost of increased runtime due to duplicate activation computations in the backward pass.  Furthermore, similar observations can be made for the SFT first then DPO setting experiment results given in Figure \ref{fig:sft-dpo-loss-space-comp-llama3}.

\begin{figure}[t]
    \centering
    \begin{minipage}[h]{0.43\textwidth}
        \centering
        \includegraphics[width=\textwidth]{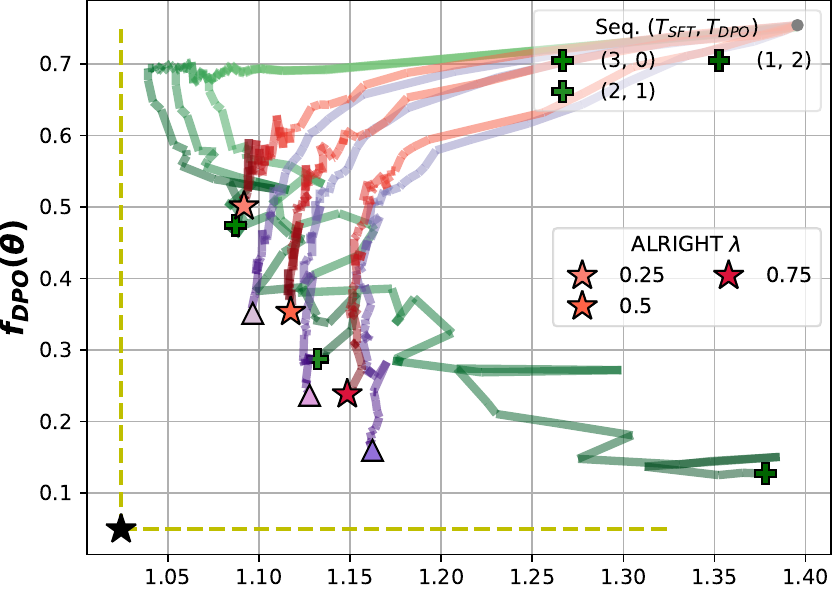} 
        \includegraphics[width=\textwidth]{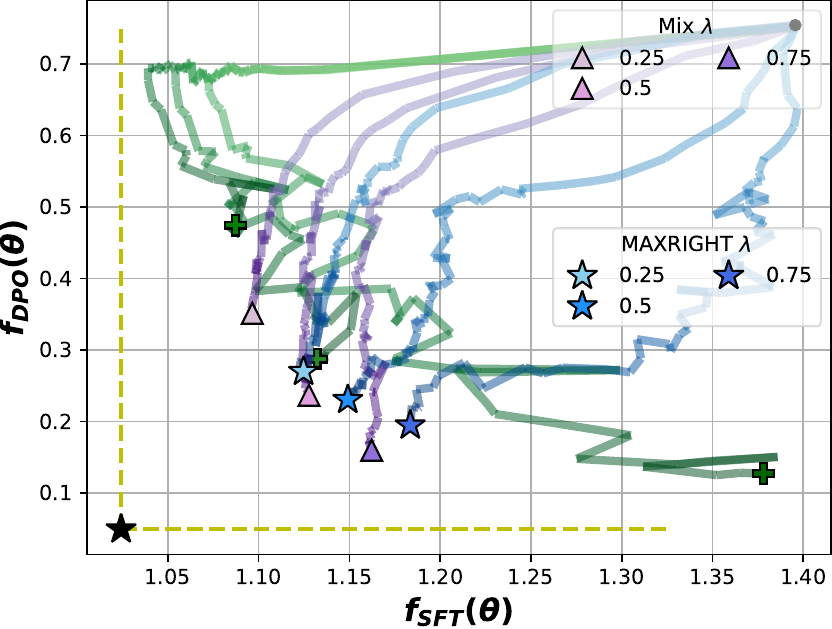}
    \end{minipage}
    \hfill
    \begin{minipage}[h]{0.55\textwidth}
        \centering
        \includegraphics[width=\textwidth]{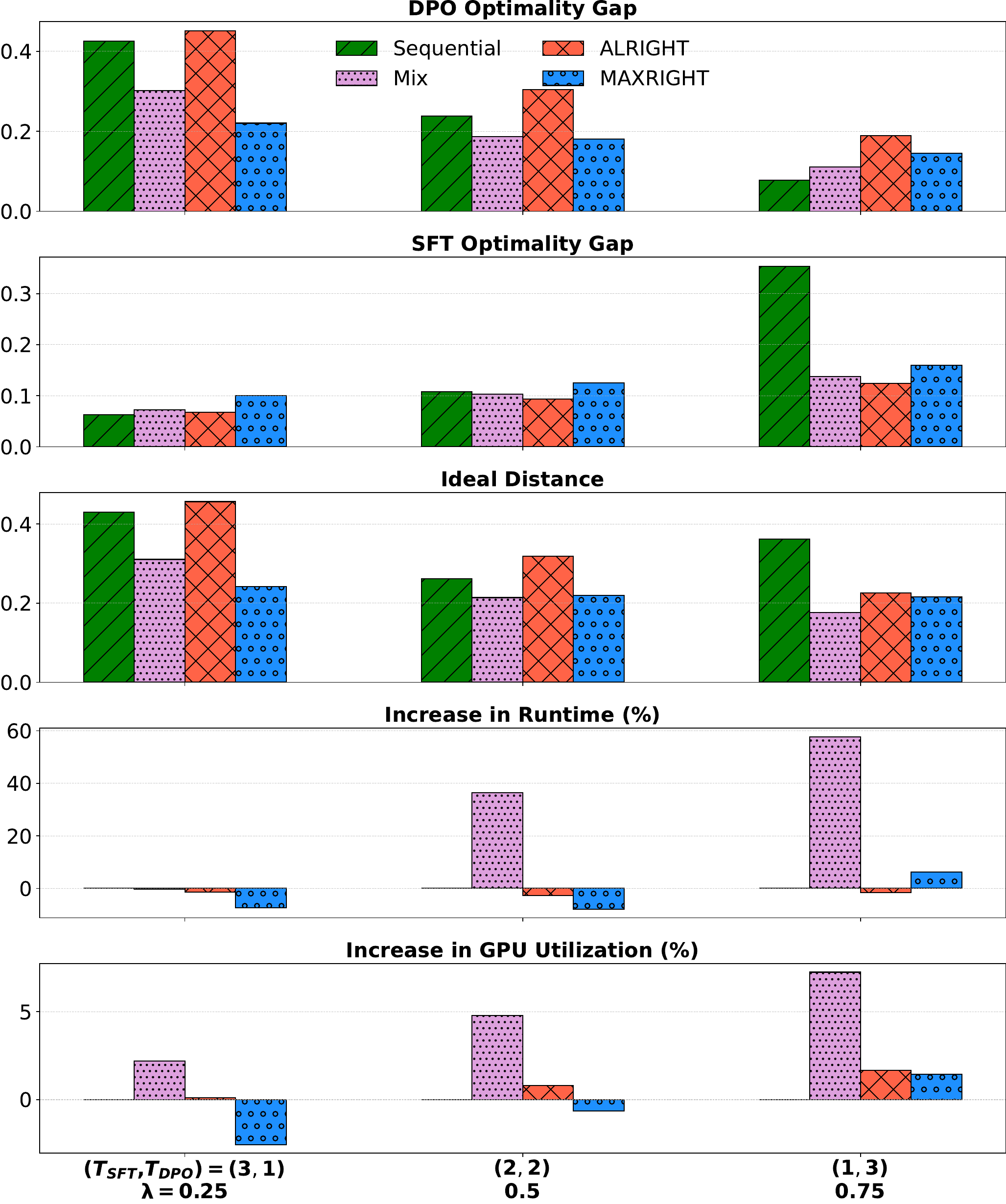}
    \end{minipage}
    \caption{Comparison of proposed methods with baselines (SFT first then DPO for Sequential) using \textsc{Llama3-8b}. \textbf{Left:} Training trajectories for various methods in the objective space, visualizing their convergence with respect to the DPO and SFT objectives. \textbf{Right:} Performance comparison across multiple evaluation metrics, including optimality gap for DPO and SFT objectives, ideal distance, runtime, and GPU utilization. The bar charts highlight the trade-offs and resource efficiency of each method for different choices of $(T_{\text{\tiny SFT}}, T_{\text{\tiny DPO}})$ or $\lambda$.}
\label{fig:sft-dpo-loss-space-comp-llama3}
\end{figure}
\begin{figure*}[h]
\begin{center}
\includegraphics[width=0.6\linewidth]{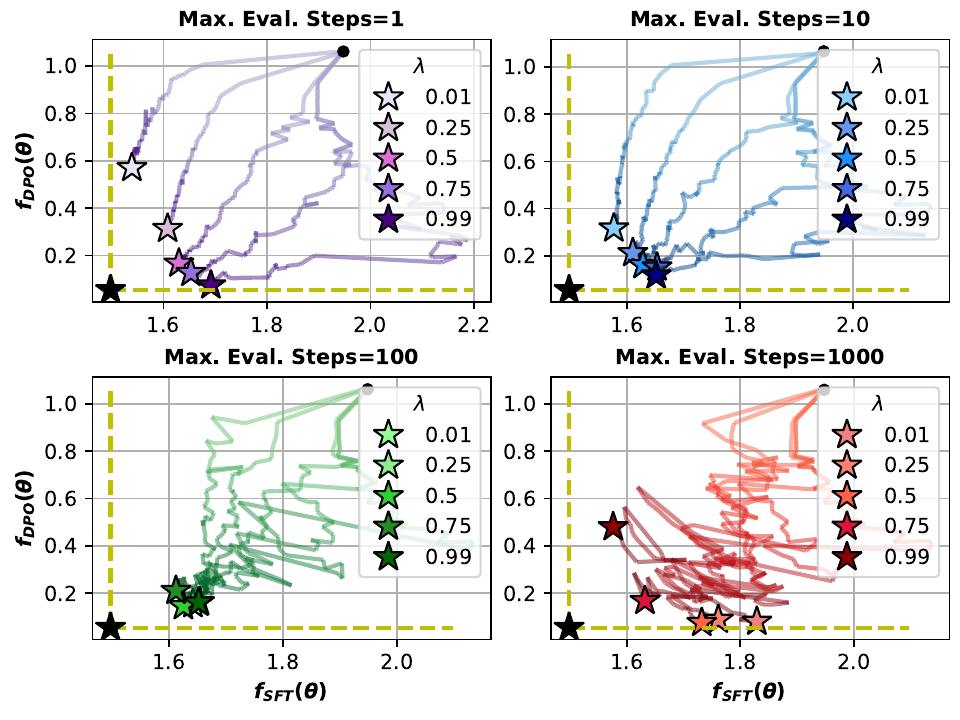}
\includegraphics[width=0.37\linewidth]{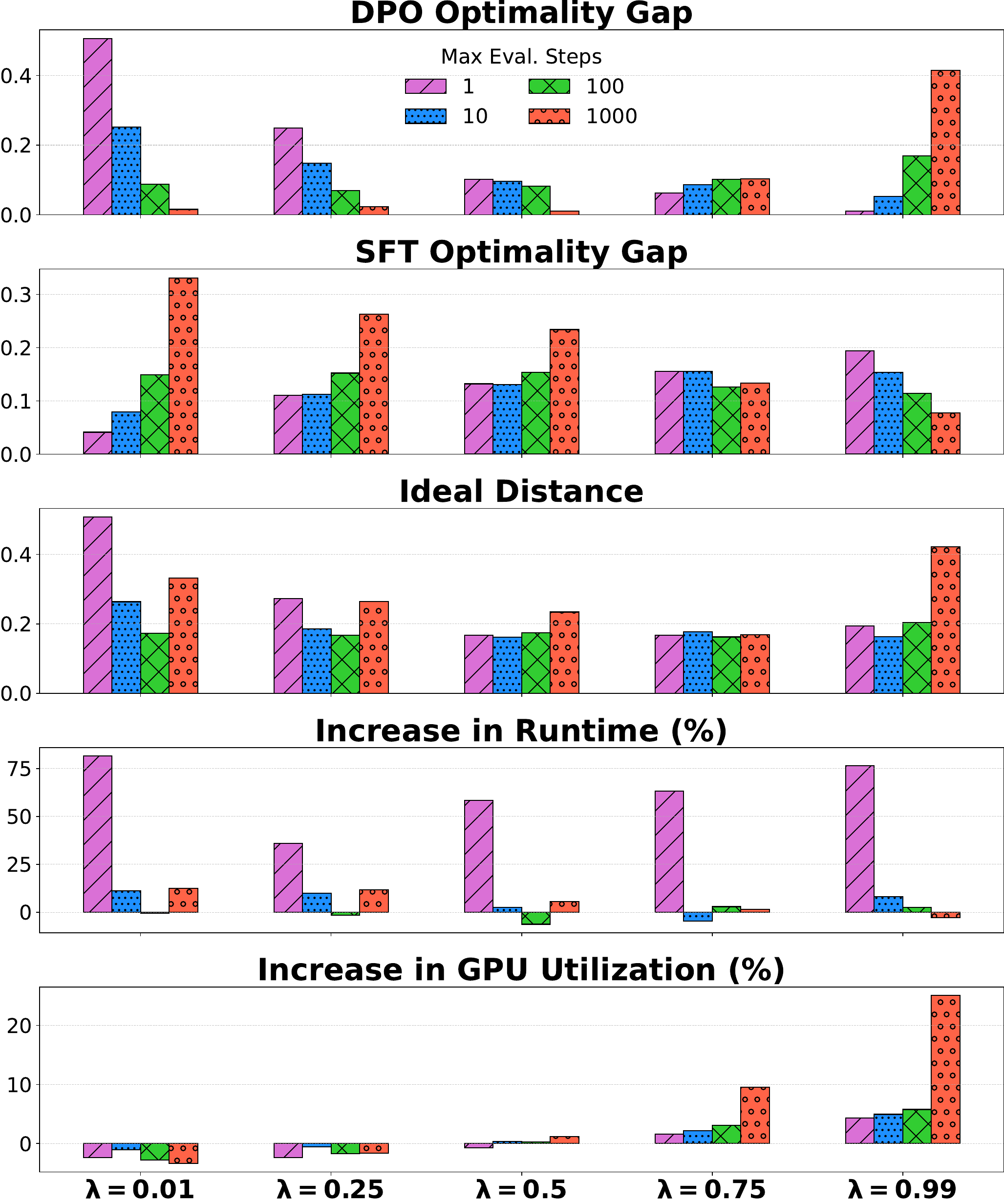}
\end{center}
\vspace{-0.2cm}
\caption{Comparison of different choices of evaluation steps for MAXRIGHT with \textsc{Pythia-1b}. 
}
\label{fig:cheb-ablation-loss-space}
\end{figure*}
\section{Memory Efficient MAXRIGHT Implementation.}\label{app:mem-eff-chebyshev}

In this section, we summarize the memory-efficient implementation of MAXRIGHT described in Section \ref{sec:chebyshev-rlhf-sft}. Even though MAXRIGHT allows one to compute the index needed for selecting the objective with a maximum (weighted) optimality gap, in practice, evaluating both objectives can be memory intensive, and only one objective is updated at a given iteration. To alleviate this issue, we propose to do simultaneous evaluations only every $k$ steps. We call a time step that simultaneous evaluation is done as a `max evaluation step'. At such time step $t=t_0$, we compute $i_{t_0}$, and update the corresponding objective as in Algorithm \ref{algo:joint-rlhf-sft} with MAXRIGHT strategy. After the update, we store the computed (weighted) optimality gap as $\bar{f}^{\tiny \text{stale}, t_0}_{1,\lambda} = \bar{f}_{1,\lambda}(\theta_{t_0})$ and $\bar{f}^{\tiny \text{stale}, t_0}_{2,\lambda} =\bar{f}_{2,\lambda}(\theta_{t_0})$. Then, for every iteration before the next max evaluation step $t_0+k$, we choose the index of the objective to be optimized as
\begin{equation}
    i_{t_0+k'} = \arg\!\max_{i} \bar{f}^{\tiny \text{stale}, t_0}_{i,\lambda},
\end{equation}
where $k'<k$. Once the index is computed, we update the corresponding objective following \eqref{eq:rlhf-update} or \eqref{eq:sft-update}, and update the stale (weighted) optimality gap as
\begin{equation}
    \bar{f}^{\tiny \text{stale}, t_0}_{i,\lambda} = \bar{f}_{i,\lambda}(\theta_{t_0+k'}), \quad \text{if } i_{t_0+k'} = i
\end{equation}
where $i\in\{1, 2\}$. This process is summarized in Algorithm \ref{algo:mem-eff-chebyshev}. With this modification, we can match the evaluation and gradient computation complexity of Algorithm \ref{algo:joint-rlhf-sft} with the ALRIGHT strategy in most iterations, at the expense of degraded accuracy in choosing $i_t$.

\textbf{Ablation of maximum evaluation steps.} Figure \ref{fig:cheb-ablation-loss-space} illustrates the influence of maximum evaluation step choices in memory-efficient MAXRIGHT on optimization trajectories and resource usage. For low values (e.g., $1$), the algorithm closely follows the trade-off determined by $\lambda$, keeping the solutions concentrated near the ideal point (e.g., compared to ALRIGHT), but incurs high runtime due to frequent evaluations. In contrast, high values (e.g., $1000$) cause significant oscillations in the objective space, failing to maintain the desired trade-off and resulting in increased GPU utilization from excessive SFT updates. The aforementioned oscillation leads to poor ideal distance performance as the model drifts away from the ideal point.

\begin{algorithm}[H]
\caption{Memory Efficient MAXRIGHT}\label{algo:mem-eff-chebyshev} 
\begin{algorithmic}[1]
\STATE Input $\D_{\text{\tiny DPO}}$, $\D_{\text{\tiny SFT}}$, $\{\alpha_t\}_{t=1}^{T}$,$\lambda\in[0, 1]$, max evaluation steps $k$
\STATE Initialize $\theta_1\in\Theta$
\FOR {$t=1, \dots, T-1$}
        \STATE Sample $x^t_1, y_w^t, y_\ell^t \sim \D_{\text{\tiny DPO}}$
        \STATE Sample $x^t_2, y^t \sim \D_{\text{\tiny SFT}}$
        \IF{$t \bmod k = 0 ~\|~ t=1$}
        \STATE Evaluate (without generating computational graph) \\
        \quad $\bar{f}_{1,\lambda}(\theta_t)\coloneqq\lambda \left(f_{\text{\tiny DPO}}(\theta_t;x^t_1, y_w^t, y_\ell^t) - f^*_{\text{\tiny DPO}} \right)$ and\\
         \quad $\bar{f}_{2,\lambda}(\theta_t)\coloneqq(1-\lambda)\left( f_{\text{\tiny SFT}}(\theta_t;x^t_2, y^t) - f^*_{\text{\tiny SFT}} \right)$
         \STATE Set  $t_0 = t$ \\
         \quad $\bar{f}^{\tiny \text{stale}, t_0}_{1,\lambda} = \bar{f}_{1,\lambda}(\theta_{t})$ \\
         \quad $\bar{f}^{\tiny \text{stale}, t_0}_{2,\lambda} =\bar{f}_{2,\lambda}(\theta_{t})$   
    \ENDIF

    \STATE Set $i_t=\arg\!\max_{i}\bar{f}^{\tiny \text{stale}, t_0}_{1,\lambda}$
    \IF{$i_t=1$}
        \STATE  Set $\bar{f}^{\tiny \text{stale}, t_0}_{1,\lambda}=\lambda \left(f_{\text{\tiny DPO}}(\theta_t;x^t_1, y_w^t, y_\ell^t) - f^*_{\text{\tiny DPO}} \right)$
        \STATE Update $\theta_{t+1} \!=\! \Pi_\Theta\left(\theta_{t} - \alpha_t g_{\text{\tiny DPO}}(\theta_t;x^t_1, y_w^t, y_\ell^t)  \right)$
    \ELSE
        \STATE  Set $\bar{f}^{\tiny \text{stale}, t_0}_{2,\lambda}=(1-\lambda)\left( f_{\text{\tiny SFT}}(\theta_t;x^t_2, y^t) - f^*_{\text{\tiny SFT}} \right))$
        \STATE Update $\theta_{t+1} \!=\! \Pi_\Theta\left(\theta_{t} - \alpha_t g_{\text{\tiny SFT}}(\theta_t;x^t_2, y^t)\right)$    
    \ENDIF
\ENDFOR

\STATE Output $\hat{\theta}_{\text{MAX}}\coloneqq\theta_T$
\end{algorithmic} 
\end{algorithm}

\subsection{Qualitative evaluation}\label{app:qualitative}

In this section, we compare the responses of the models post-trained with the proposed methods ALRIGHT and MAXRIGHT with those of the baselines. We use models trained with $\lambda=0.5$ for MIX, ALRIGHT, MAXRIGHT, and $(T_{DPO}, T_{SFT})=(3, 3)$ for Sequential. Below we provide some example prompts, the corresponding chosen response of the dataset, and the response of model post-trained using the proposed methods and the baselines, in the experiment setting (3) described in Section \ref{app:llama3-8b-exp-det}.

In the first example below, we see that the response by Sequential is not very helpful, and have a more dramatic tone. On the other hand, Mix provide a more detailed response, which might be too detailed for the simple question asked by the user. ALRIGHT and MAXRIGHT provide similar answers that match the chosen response closely, and provide a concise yet helpful answer to the user's question:
{  \scriptsize
    \label{tab:alpaca-eval-example-1}
    \begin{tabularx}{\linewidth}{l X}
    \toprule
         \textbf{Prompt} & What kind of things go in a tackle box? \\
         \midrule
         \textbf{Chosen Response} & A tackle box is a tool used for fishing.  It is often small enough to be carried in one's pocket,
         and contains
         a variety
         of tools that are useful for fishing.  For example, a tackle box might contain
         a small knife for gutting fish, a small net for cleaning small fish, or a float for testing water.\\
         \midrule
         \textbf{Sequential} & Lures for fish. {\color{red}And stuff}. Hooks. And maybe a worm. (takes a worm from a plastic container)
         {\color{red}Or a gummy bear like this}.\\
         \midrule
         \textbf{Mix} & A tackle box for fishing can contain a wide variety of gear and equipment, including but not limited
         to: {\color{red} 1. Rods and reels: Different rods and reels are used for different techniques and purposes.
         2. Lines: Monofilament, fluorocarbon, and braided lines with a variety of strengths and colors.
         3. Lures: Both artificial and live bait can be carried in a tackle box. Spinners, spoons, crankbaits,
         jigs, top-water lures, and soft plastic creatures are common. 4. Hooks: In addition to terminal tackle
         like leaders and swivels, hooks in various sizes and styles are essential for rigging and baiting.
         5. Bobbers: Used for fishing with live bait, bobbers can be used to suspend a hook or bait at the
         desired depth. 6. Sinkers: Used for live bait fishing, sinkers weight the line to help keep it on the
         bottom. 7. Floaters and drifters: Used for fishing from shore or a boat, they can help keep the line
         taut and keep the bait in the water. 8. Net: Useful for landing fish. 9. Scissors: Used for cutting line
         and cutting off hooks and fishing line. 10. Reel pin: Used for threading }
    \end{tabularx}}
{  \scriptsize
    \begin{tabularx}{\linewidth}{l X}
          {\color{white} \textbf{Chosen Response}}& the line back through the reel
         seat. As you can see, there are many things that can go in a tackle box for fishing, but some of the most
         essential items will 
          depend on your specific fishing needs and technique.\\
         \midrule
         \textbf{ALRIGHT} & Tackle boxes typically contain fishing tackle and gear used for various fishing techniques. It may
         include a few necessary tools and equipment such as {\color{red}lures, hooks, sinkers, bobbers, split shot,
         line, leader, bait, and more}. Depending on your style of fishing and the type of fish you are targeting, 
         you may need to customize your tackle box to include only the relevant equipment. Additionally,
         many anglers also include their favorite snacks and drinks to enjoy while fishing.\\
         \midrule
         \textbf{MAXRIGHT} & A tackle box can be filled with various fishing tools and equipment, depending on the type of fishing
         and the location. Some common items that go in a tackle box include: {\color{red}hooks, sinkers, floats, lures,
         bobbers, swivels, weights, split shot, leaders, and line}. You will also want to include tools to help you
         tie knots and sharpen hooks. There may also be baits and scents specific to each type of fishing, for
         example worms, minnows, or crickets for freshwater fishing, or shrimp, squid, or clams for saltwater fishing.\\
         \bottomrule
    \end{tabularx}}

The second example involves the user asking how they can improve their penmanship. We can see that again, Sequential response is not very helpful with the details, and provide more of a casual discussion. Both Mix and ALRIGHT provide detailed steps to improve penmanship. MAXRIGHT also provide high-level advise how to improve penmanship, which is similar to the chosen response of the dataset:

{  \scriptsize
    \label{tab:alpaca-eval-example-2}
    \begin{tabularx}{\linewidth}{l X}
    \toprule
         \textbf{Prompt} & How do I get better penmanship? \\
         \midrule
         \textbf{Chosen Response} & You can practice.  Try writing letters and short notes, it will help you practice forming letters neatly.  You can also try writing with a different kind of pen or pencil.  Sometimes writing with a fountain pen can help you form your letters more neatly.  Try writing slowly, and carefully, to make sure you get each letter right.  And you might want to practice your signature.  Try writing your name a few times to make sure you get it right every time.\\
         \midrule
         \textbf{Sequential} & The same way I do - practice, and go slowly. I find it very helpful to just keep at it - no matter how much you feel you suck at it. It's just like anything else - if you practice enough you'll get to a point you're satisfied with.
         The penmanship class at the library has some really excellent teachers. I've been told they're very good, and that they have a lot of tricks that help a lot of people. I don't know much about them, though.
         {\color{red}Also, I think if you type your work in a clean, clear font like Verdana or something, it will be really easy to make sure your handwriting is consistent and looks nice}. If you keep your penmanship kind of messy looking, it can be difficult to make it look good, and then you're just going to make yourself crazy over it. You know?\\
         \midrule
         \textbf{Mix} & Here are some steps you can follow to improve your penmanship:
         {\color{red}1) Use proper posture and position. Make sure your chair and writing surface are comfortable and you are sitting up straight with good posture. Hold your pen or pencil at a 45-degree angle and let your hand and arm rest on the table.
         2) Write slowly and carefully. Some people mistakenly think they can improve their penmanship by writing faster. Instead, take the time to make sure each letter is formed properly and clearly.
         3) Practice, practice, practice! No matter how long you have been writing, you can always improve with practice. Try copying a paragraph or a page from your favorite book or take a calligraphy class.
         4) Focus on the basics. Start by working on the most fundamental aspects of penmanship, like making sure your lowercase letters and capital letters are all properly formed and legible.
         5) Experiment with different writing styles. Not all writing styles are equally compatible with your natural handwriting tendencies. Experiment with different penmanship styles to find the one that suits your hand the best.
         6) Use a ruler as a guide. A ruler can help you practice writing straight lines and guide your hand for straight, consistent letters.
         7) Warm up before you write. It may be helpful to practice a few strokes and exercises to warm up your hand and get your brain ready to write.
         8) Take breaks and give your hand time to rest. It's important to take frequent breaks and give your hand time to rest when you're practicing penmanship.
         9) Seek feedback from others. Ask someone to read your writing and provide you with constructive criticism. This can help you identify areas that need improvement.}
         Remember, it takes sustained effort and practice to develop good penmanship. Consistency is key, but also be patient and kind to yourself! With time and persistence, you can achieve your goal of better penmanship. \\
         \midrule
         \textbf{ALRIGHT} & To improve your penmanship, you can practice writing the same letters or words over and over again. Make sure to pay attention to the spacing between the letters and the overall shape of the letters. You can also practice different writing styles such as cursive or print to improve your hand-eye coordination and muscle memory. Finally, keep in mind that improving penmanship through practice takes time, so be patient and persistent in your efforts. Here are some steps you can follow to help you get better at penmanship:
    \end{tabularx}}
    
{  \scriptsize
    \label{tab:alpaca-eval-example-2}
    \begin{tabularx}{\linewidth}{l X}
         {\color{white} \textbf{Chosen Response}}& {\color{red}1. Practice every day: Make time to write each day, even if it's just for a few minutes. The more you practice, the faster you'll see improvements.
         2. Focus on your pen stroke: Try to start each letter in the same spot and end the letter on the same line. This will help you create a consistent look for your writing.
         3. Space your writing properly: The correct spacing can improve the legibility and overall appearance of your handwriting. Pay attention to the spaces between words, lines, and letters when writing.
         4. Work on letter formation: Concentrate on writing each letter clearly and in the proper shape. Start by practicing printing in blocks, then progress to cursive, or vice versa.
         5. Use the right equipment: Make sure you have a smooth writing pen and good quality paper. A firm hand can help you achieve a more consistent look in your writing.
         6. Set specific goals: Identify one or two areas in your writing that you want to improve, such as legibility, speed, or consistency. Focus on these areas and track your progress. 
         7. Share your writing: Have someone else read your writing and point out areas that need improvement. This can help you identify specific problems in your handwriting and give you a new perspective on your writing.
         8. Be patient. Remember that improving your handwriting takes time, so don't get discouraged. Keep practicing, and you'll see improvements in your penmanship.}
         Please let me know if you have any questions or concerns. 
         In summary, to get better penmanship, you can practice writing every day, focus on your pen stroke, space your writing properly, work on letter formation, use the right equipment, set specific goals, and share your writing with others for feedback. Remember that improving penmanship takes time, so be persistent and patient.\\
         \midrule
         \textbf{MAXRIGHT} & It would help to practice writing often. {\color{red}Writing in your own journal or diary every day would be a great way to practice. Or, you could try taking a class in calligraphy, which is basically stylized penmanship. 
         Make sure you use a pen with a fine, smooth point. This will help you control the line thickness and create a neat, uniform look. If you ever see a mistake, cross it out and start again!} Remember, practice makes perfect!\\
         \bottomrule
    \end{tabularx}}

The final example involves the user asking how they can teach kids about cooperation. It is interesting to see that the chosen response of the dataset itself seems to be of low quality, since it has suggestions like `punish' non-cooperation, which might not be desirable advice. The response given by Sequential is again not very helpful, and too short for a question like this, which has a lot of depth. All Mix, ALRIGHT, and MAXRIGHT provide comprehensive and helpful advice for the question asked. However, the response provided by Mix seems to be too detailed to the point that some parts are not necessarily true. For example, the Mix response suggests Chess as a game that teaches cooperation, which is questionable since Chess is a predominantly individual sport:

{  \scriptsize
    \label{tab:alpaca-eval-example-3}
    \begin{tabularx}{\linewidth}{l X}
    \toprule
         \textbf{Prompt} & How can I teach kids about cooperation? \\
         \midrule
         \textbf{Chosen Response} & How could we get children to cooperate together?  A few different techniques are to reward cooperation, {\color{red}punish non-cooperation}, create a hierarchy of status, and incentivize people to cooperate.\\
         \midrule
         \textbf{Sequential} & {\color{red}We can program a video game that involves co-opetition.}
         This will help teach them about collaboration and healthy competition. 
         This way everyone wins!\\
         \midrule
         \textbf{Mix} & Here are some activities and approaches to teach kids about cooperation:
         {\color{red}1) Storytelling: Narrate stories about kids working together and 
         collaborating. 
         For example, imagine a group of kids building a 
         sandcastle on 
         the beach. Have each child take turns participating and contributing to the project. You can pause and ask questions about the group's dynamics and how they handled problems (e.g., what challenges did they face and how did they resolve them?). In doing so, you will highlight the importance of communication and cooperation, and foster the idea that they can accomplish more together than individually.
         2) Role-playing: Role-playing can be a fun and effective way to teach kids about cooperation. For example, you could pretend to be a pirate captain searching for treasure, and divide the kids into two teams. The first team would be the thieves who have stolen the treasure. The other team as the pirates who need to catch up to them. In order to catch the thieves, the kids will have to work together and divide tasks (e.g., one team member may need to scout the area, while others need to guard the path and block the thieves' escape). Once caught, the kids can learn to problem-solve together by developing a plan on how to return the treasure.
         3) Game play: Cooperative games can help kids learn to work together and collaborate in achieving common goals. Simple games like Four in a Row, Chess, or Simon Says can be a good place to start. You can make the playing board large so that it can provide space for more than one player, while you can guide the discussion afterwards about how they cooperated and contributed to the group's achievement.}
         Remember, it's important to make learning fun and enjoyable for kids, so try to approach these activities in a lively and interactive way. Give plenty of chances to kids to take part and participate, and don't be afraid to improvise and adapt to the needs of the group. Cooperation is a vital social skill and early introduction can help children develop a culture of cooperation that helps them succeed in life. \\
         \midrule
         \textbf{ALRIGHT} & Cooperation is a vital life skill for children to learn, as it teaches them how to work together, share responsibilities, and foster a sense of community. Here are some tips and activities to teach kids about cooperation:
         {\color{red}1. Explain the concept of cooperation: Define cooperation for children and explain how it benefits them, their friends, and the whole group. Encourage them to think of examples from real life where they have seen people cooperating with each other.
         2. Share stories about cooperators: Read children's books or share real-life examples of people who worked together to achieve a goal. Point out how they communicated and took turns or assigned specific roles to each other.
         3. Play cooperative games: Find cooperative games that require teamwork, such as building structures with blocks or playing musical chairs. After the games, focus on how everyone had to work together to achieve a common goal.
         4. Practice taking turns: Encourage children to take turns in activities or conversations. Teach them to wait patiently for their turn and to let others finish speaking before speaking.
         5. Role-play cooperation situations: Role-play simple cooperation scenarios like sharing a toy or a piece of food. Practice using kind words and gestures and taking turns.}
         Remember that cooperation is learned through experience, not just verbal instruction. Encourage children to practice cooperation in their everyday lives, and reinforce positive behavior when they cooperate with others. With consistent efforts and modeling, children will develop an understanding of cooperation, as well as the valuable life skill of working together to achieve a common goal.\\
         \midrule
         \textbf{MAXRIGHT} & Cooperation is important because it helps children to develop social skills and build meaningful relationships with others. Here are some ways that you can teach kids about cooperation:
         {\color{red}1. Encourage sharing and taking turns: Teach children to share their toys and games, and let them learn to take turns and wait their turn. You can use songs, rhymes, and games to promote sharing at a young age.
         2. Teach conflict resolution: Teach children conflict resolution skills, such as apologizing, resolving disagreements, and seeking compromise. You can role-play scenarios with your children to help them understand how to resolve conflicts peacefully.
         3. Model cooperative behavior: Show your children how to cooperate by working together on tasks, sharing information, and listening to each other.}
    \end{tabularx}}

    {  \scriptsize
    \begin{tabularx}{\linewidth}{l X}
    {\color{white} \textbf{Chosen Response}}& {\color{red}You can also explain to children how cooperation helps to build stronger relationships and achieve shared goals.
          4. Encourage teamwork in group activities: Group activities, such as sports, arts and crafts, or team-based games, provide opportunities for children to learn the importance of cooperation in achieving a common goal.
         5. Teach children empathy and respect: Empathy and respect are important elements of cooperation. You can help children to develop these skills by teaching them about others' perspectives and needs. This can also help children to learn to value and respect different opinions and viewpoints.}
         By using these strategies, you can help children to learn about the importance of cooperation and develop the skills needed to work together effectively with others.\\
         \bottomrule
    \end{tabularx}}


\begin{thebibliography}{77}
\providecommand{\natexlab}[1]{#1}
\providecommand{\url}[1]{\texttt{#1}}
\expandafter\ifx\csname urlstyle\endcsname\relax
  \providecommand{\doi}[1]{doi: #1}\else
  \providecommand{\doi}{doi: \begingroup \urlstyle{rm}\Url}\fi

\bibitem[Abdin et~al.(2024)Abdin, Jacobs, Awan, Aneja, Awadallah, Awadalla, Bach, Bahree, Bakhtiari, Behl, et~al.]{abdin2024phi}
Marah Abdin, Sam~Ade Jacobs, Ammar~Ahmad Awan, Jyoti Aneja, Ahmed Awadallah, Hany Awadalla, Nguyen Bach, Amit Bahree, Arash Bakhtiari, Harkirat Behl, et~al.
\newblock Phi-3 technical report: A highly capable language model locally on your phone.
\newblock \emph{arXiv preprint arXiv:2404.14219}, 2024.

\bibitem[Achiam et~al.(2023)Achiam, Adler, Agarwal, Ahmad, Akkaya, Aleman, Almeida, Altenschmidt, Altman, Anadkat, et~al.]{achiam2023gpt}
Josh Achiam, Steven Adler, Sandhini Agarwal, Lama Ahmad, Ilge Akkaya, Florencia~Leoni Aleman, Diogo Almeida, Janko Altenschmidt, Sam Altman, Shyamal Anadkat, et~al.
\newblock Gpt-4 technical report.
\newblock \emph{arXiv preprint arXiv:2303.08774}, 2023.

\bibitem[Azar et~al.(2024)Azar, Guo, Piot, Munos, Rowland, Valko, and Calandriello]{azar2024general}
Mohammad~Gheshlaghi Azar, Zhaohan~Daniel Guo, Bilal Piot, Remi Munos, Mark Rowland, Michal Valko, and Daniele Calandriello.
\newblock A general theoretical paradigm to understand learning from human preferences.
\newblock In \emph{International Conference on Artificial Intelligence and Statistics}, pages 4447--4455. PMLR, 2024.

\bibitem[Bai et~al.(2022{\natexlab{a}})Bai, Jones, Ndousse, Askell, Chen, DasSarma, Drain, Fort, Ganguli, Henighan, et~al.]{bai2022training}
Yuntao Bai, Andy Jones, Kamal Ndousse, Amanda Askell, Anna Chen, Nova DasSarma, Dawn Drain, Stanislav Fort, Deep Ganguli, Tom Henighan, et~al.
\newblock Training a helpful and harmless assistant with reinforcement learning from human feedback.
\newblock \emph{arXiv preprint arXiv:2204.05862}, 2022{\natexlab{a}}.

\bibitem[Bai et~al.(2022{\natexlab{b}})Bai, Kadavath, Kundu, Askell, Kernion, Jones, Chen, Goldie, Mirhoseini, McKinnon, et~al.]{bai2022constitutional}
Yuntao Bai, Saurav Kadavath, Sandipan Kundu, Amanda Askell, Jackson Kernion, Andy Jones, Anna Chen, Anna Goldie, Azalia Mirhoseini, Cameron McKinnon, et~al.
\newblock Constitutional ai: Harmlessness from ai feedback.
\newblock \emph{arXiv preprint arXiv:2212.08073}, 2022{\natexlab{b}}.

\bibitem[Biderman et~al.(2023)Biderman, Schoelkopf, Anthony, Bradley, O’Brien, Hallahan, Khan, Purohit, Prashanth, Raff, et~al.]{biderman2023pythia}
Stella Biderman, Hailey Schoelkopf, Quentin~Gregory Anthony, Herbie Bradley, Kyle O’Brien, Eric Hallahan, Mohammad~Aflah Khan, Shivanshu Purohit, USVSN~Sai Prashanth, Edward Raff, et~al.
\newblock Pythia: A suite for analyzing large language models across training and scaling.
\newblock In \emph{International Conference on Machine Learning}, pages 2397--2430. PMLR, 2023.

\bibitem[Chakraborty et~al.(2024)Chakraborty, Qiu, Yuan, Koppel, Huang, Manocha, Bedi, and Wang]{chakraborty2024maxmin}
Souradip Chakraborty, Jiahao Qiu, Hui Yuan, Alec Koppel, Furong Huang, Dinesh Manocha, Amrit~Singh Bedi, and Mengdi Wang.
\newblock Maxmin-rlhf: Towards equitable alignment of large language models with diverse human preferences.
\newblock \emph{arXiv preprint arXiv:2402.08925}, 2024.

\bibitem[Chen et~al.(2023{\natexlab{a}})Chen, Zhang, Shi, Li, Smola, and Yang]{chen2023parameter}
Jiaao Chen, Aston Zhang, Xingjian Shi, Mu~Li, Alex Smola, and Diyi Yang.
\newblock Parameter-efficient fine-tuning design spaces.
\newblock \emph{arXiv preprint arXiv:2301.01821}, 2023{\natexlab{a}}.

\bibitem[Chen et~al.(2023{\natexlab{b}})Chen, Qian, Tang, Lai, Liu, Han, and Jia]{chen2023longlora}
Yukang Chen, Shengju Qian, Haotian Tang, Xin Lai, Zhijian Liu, Song Han, and Jiaya Jia.
\newblock Longlora: Efficient fine-tuning of long-context large language models.
\newblock \emph{arXiv preprint arXiv:2309.12307}, 2023{\natexlab{b}}.

\bibitem[Chen et~al.(2024)Chen, Deng, Yuan, Ji, and Gu]{chen2024self}
Zixiang Chen, Yihe Deng, Huizhuo Yuan, Kaixuan Ji, and Quanquan Gu.
\newblock Self-play fine-tuning converts weak language models to strong language models.
\newblock \emph{arXiv preprint arXiv:2401.01335}, 2024.

\bibitem[Christiano et~al.(2017)Christiano, Leike, Brown, Martic, Legg, and Amodei]{christiano2017deep}
Paul~F Christiano, Jan Leike, Tom Brown, Miljan Martic, Shane Legg, and Dario Amodei.
\newblock Deep reinforcement learning from human preferences.
\newblock \emph{Advances in neural information processing systems}, 30, 2017.

\bibitem[Cui et~al.(2023)Cui, Yuan, Ding, Yao, Zhu, Ni, Xie, Liu, and Sun]{cui2023ultrafeedback}
Ganqu Cui, Lifan Yuan, Ning Ding, Guanming Yao, Wei Zhu, Yuan Ni, Guotong Xie, Zhiyuan Liu, and Maosong Sun.
\newblock Ultrafeedback: Boosting language models with high-quality feedback, 2023.

\bibitem[Dai et~al.(2023)Dai, Pan, Sun, Ji, Xu, Liu, Wang, and Yang]{dai2023safe}
Josef Dai, Xuehai Pan, Ruiyang Sun, Jiaming Ji, Xinbo Xu, Mickel Liu, Yizhou Wang, and Yaodong Yang.
\newblock Safe rlhf: Safe reinforcement learning from human feedback.
\newblock \emph{arXiv preprint arXiv:2310.12773}, 2023.

\bibitem[Devlin(2018)]{devlin2018bert}
Jacob Devlin.
\newblock Bert: Pre-training of deep bidirectional transformers for language understanding.
\newblock \emph{arXiv preprint arXiv:1810.04805}, 2018.

\bibitem[Ding et~al.(2024)Ding, Ji, Wang, and Xu]{ding2024understanding}
Meng Ding, Kaiyi Ji, Di~Wang, and Jinhui Xu.
\newblock Understanding forgetting in continual learning with linear regression.
\newblock \emph{arXiv preprint arXiv:2405.17583}, 2024.

\bibitem[Dubey et~al.(2024)Dubey, Jauhri, Pandey, Kadian, Al-Dahle, Letman, Mathur, Schelten, Yang, Fan, et~al.]{dubey2024llama}
Abhimanyu Dubey, Abhinav Jauhri, Abhinav Pandey, Abhishek Kadian, Ahmad Al-Dahle, Aiesha Letman, Akhil Mathur, Alan Schelten, Amy Yang, Angela Fan, et~al.
\newblock The llama 3 herd of models.
\newblock \emph{arXiv preprint arXiv:2407.21783}, 2024.

\bibitem[Fang et~al.(2024)Fang, Guo, Zhou, Ma, Zhang, and Feng]{fang2024llama}
Qingkai Fang, Shoutao Guo, Yan Zhou, Zhengrui Ma, Shaolei Zhang, and Yang Feng.
\newblock Llama-omni: Seamless speech interaction with large language models.
\newblock \emph{arXiv preprint arXiv:2409.06666}, 2024.
\newblock URL \url{https://huggingface.co/ICTNLP/Llama-3.1-8B-Omni}.

\bibitem[Fernando et~al.(2023)Fernando, Shen, Liu, Chaudhury, Murugesan, and Chen]{fernando2023mitigating}
Heshan Fernando, Han Shen, Miao Liu, Subhajit Chaudhury, Keerthiram Murugesan, and Tianyi Chen.
\newblock Mitigating gradient bias in multi-objective learning: A provably convergent approach.
\newblock International Conference on Learning Representations, 2023.

\bibitem[Ganguli et~al.(2022)Ganguli, Lovitt, Kernion, Askell, Bai, Kadavath, Mann, Perez, Schiefer, Ndousse, et~al.]{ganguli2022red}
Deep Ganguli, Liane Lovitt, Jackson Kernion, Amanda Askell, Yuntao Bai, Saurav Kadavath, Ben Mann, Ethan Perez, Nicholas Schiefer, Kamal Ndousse, et~al.
\newblock Red teaming language models to reduce harms: Methods, scaling behaviors, and lessons learned.
\newblock \emph{arXiv preprint arXiv:2209.07858}, 2022.

\bibitem[Gao et~al.(2024)Gao, Tow, Abbasi, Biderman, Black, DiPofi, Foster, Golding, Hsu, Le~Noac'h, Li, McDonell, Muennighoff, Ociepa, Phang, Reynolds, Schoelkopf, Skowron, Sutawika, Tang, Thite, Wang, Wang, and Zou]{eval-harness}
Leo Gao, Jonathan Tow, Baber Abbasi, Stella Biderman, Sid Black, Anthony DiPofi, Charles Foster, Laurence Golding, Jeffrey Hsu, Alain Le~Noac'h, Haonan Li, Kyle McDonell, Niklas Muennighoff, Chris Ociepa, Jason Phang, Laria Reynolds, Hailey Schoelkopf, Aviya Skowron, Lintang Sutawika, Eric Tang, Anish Thite, Ben Wang, Kevin Wang, and Andy Zou.
\newblock The language model evaluation harness, 07 2024.
\newblock URL \url{https://zenodo.org/records/12608602}.

\bibitem[Guo et~al.(2024)Guo, Cui, Yuan, Ding, Sun, Sun, Chen, Xie, Zhou, Lin, et~al.]{guo2024controllable}
Yiju Guo, Ganqu Cui, Lifan Yuan, Ning Ding, Zexu Sun, Bowen Sun, Huimin Chen, Ruobing Xie, Jie Zhou, Yankai Lin, et~al.
\newblock Controllable preference optimization: Toward controllable multi-objective alignment.
\newblock \emph{arXiv preprint arXiv:2402.19085}, 2024.

\bibitem[Hendrycks et~al.(2020)Hendrycks, Burns, Basart, Zou, Mazeika, Song, and Steinhardt]{hendrycks2020measuring}
Dan Hendrycks, Collin Burns, Steven Basart, Andy Zou, Mantas Mazeika, Dawn Song, and Jacob Steinhardt.
\newblock Measuring massive multitask language understanding.
\newblock \emph{arXiv preprint arXiv:2009.03300}, 2020.

\bibitem[Hong et~al.(2024{\natexlab{a}})Hong, Lee, and Thorne]{hong2024orpo}
Jiwoo Hong, Noah Lee, and James Thorne.
\newblock Orpo: Monolithic preference optimization without reference model.
\newblock \emph{arXiv preprint arXiv:2403.07691}, 2\penalty0 (4):\penalty0 5, 2024{\natexlab{a}}.

\bibitem[Hong et~al.(2024{\natexlab{b}})Hong, Lee, and Thorne]{hong2024reference}
Jiwoo Hong, Noah Lee, and James Thorne.
\newblock Reference-free monolithic preference optimization with odds ratio.
\newblock \emph{arXiv e-prints}, pages arXiv--2403, 2024{\natexlab{b}}.

\bibitem[Howard and Ruder(2018)]{howard2018universal}
Jeremy Howard and Sebastian Ruder.
\newblock Universal language model fine-tuning for text classification.
\newblock \emph{arXiv preprint arXiv:1801.06146}, 2018.

\bibitem[Hu et~al.(2021)Hu, Shen, Wallis, Allen-Zhu, Li, Wang, Wang, and Chen]{hu2021lora}
Edward~J Hu, Yelong Shen, Phillip Wallis, Zeyuan Allen-Zhu, Yuanzhi Li, Shean Wang, Lu~Wang, and Weizhu Chen.
\newblock Lora: Low-rank adaptation of large language models.
\newblock \emph{arXiv preprint arXiv:2106.09685}, 2021.

\bibitem[Hu et~al.(2024)Hu, Wu, Wang, Xianyu, Zhang, and Cao]{hu2024openrlhf}
Jian Hu, Xibin Wu, Weixun Wang, Xianyu, Dehao Zhang, and Yu~Cao.
\newblock Openrlhf: An easy-to-use, scalable and high-performance rlhf framework.
\newblock \emph{arXiv preprint arXiv:2405.11143}, 2024.

\bibitem[Hua et~al.(2024)Hua, Qi, Zhang, Yu, Ding, Lv, Tian, and Zhou]{hua2024intuitive}
Ermo Hua, Biqing Qi, Kaiyan Zhang, Yue Yu, Ning Ding, Xingtai Lv, Kai Tian, and Bowen Zhou.
\newblock Intuitive fine-tuning: Towards unifying sft and rlhf into a single process.
\newblock \emph{arXiv preprint arXiv:2405.11870}, 2024.

\bibitem[Ji et~al.(2024)Ji, Hong, Zhang, Chen, Dai, Zheng, Qiu, Li, and Yang]{ji2024pku}
Jiaming Ji, Donghai Hong, Borong Zhang, Boyuan Chen, Josef Dai, Boren Zheng, Tianyi Qiu, Boxun Li, and Yaodong Yang.
\newblock Pku-saferlhf: Towards multi-level safety alignment for llms with human preference.
\newblock \emph{arXiv preprint arXiv:2406.15513}, 2024.

\bibitem[Kang et~al.(2024)Kang, Just, Sun, Jahagirdar, Zhang, Du, Sahu, and Jia]{kang2024get}
Feiyang Kang, Hoang~Anh Just, Yifan Sun, Himanshu Jahagirdar, Yuanzhi Zhang, Rongxing Du, Anit~Kumar Sahu, and Ruoxi Jia.
\newblock Get more for less: Principled data selection for warming up fine-tuning in llms.
\newblock \emph{arXiv preprint arXiv:2405.02774}, 2024.

\bibitem[Kim et~al.(2024)Kim, Lee, Kim, Park, Yoo, Kwon, and Lee]{kim2024memory}
Jeonghoon Kim, Jung~Hyun Lee, Sungdong Kim, Joonsuk Park, Kang~Min Yoo, Se~Jung Kwon, and Dongsoo Lee.
\newblock Memory-efficient fine-tuning of compressed large language models via sub-4-bit integer quantization.
\newblock \emph{Advances in Neural Information Processing Systems}, 36, 2024.

\bibitem[Kirk et~al.(2023)Kirk, Mediratta, Nalmpantis, Luketina, Hambro, Grefenstette, and Raileanu]{kirk2023understanding}
Robert Kirk, Ishita Mediratta, Christoforos Nalmpantis, Jelena Luketina, Eric Hambro, Edward Grefenstette, and Roberta Raileanu.
\newblock Understanding the effects of rlhf on llm generalisation and diversity.
\newblock \emph{arXiv preprint arXiv:2310.06452}, 2023.

\bibitem[Lee et~al.(2024)Lee, Bai, Pres, Wattenberg, Kummerfeld, and Mihalcea]{lee2024mechanistic}
Andrew Lee, Xiaoyan Bai, Itamar Pres, Martin Wattenberg, Jonathan~K Kummerfeld, and Rada Mihalcea.
\newblock A mechanistic understanding of alignment algorithms: A case study on dpo and toxicity.
\newblock \emph{arXiv preprint arXiv:2401.01967}, 2024.

\bibitem[Li et~al.(2024)Li, Zeng, Liao, Li, Kang, Garcia, and Hong]{li2024joint}
Chenliang Li, Siliang Zeng, Zeyi Liao, Jiaxiang Li, Dongyeop Kang, Alfredo Garcia, and Mingyi Hong.
\newblock Joint demonstration and preference learning improves policy alignment with human feedback.
\newblock \emph{arXiv preprint arXiv:2406.06874}, 2024.

\bibitem[Li et~al.(2023)Li, Yu, Liang, He, Karampatziakis, Chen, and Zhao]{li2023loftq}
Yixiao Li, Yifan Yu, Chen Liang, Pengcheng He, Nikos Karampatziakis, Weizhu Chen, and Tuo Zhao.
\newblock Loftq: Lora-fine-tuning-aware quantization for large language models.
\newblock \emph{arXiv preprint arXiv:2310.08659}, 2023.

\bibitem[Lin et~al.(2023)Lin, Tan, Lin, Zheng, Pi, Zhang, Diao, Wang, Zhao, Yao, et~al.]{lin2023speciality}
Yong Lin, Lu~Tan, Hangyu Lin, Zeming Zheng, Renjie Pi, Jipeng Zhang, Shizhe Diao, Haoxiang Wang, Han Zhao, Yuan Yao, et~al.
\newblock Speciality vs generality: An empirical study on catastrophic forgetting in fine-tuning foundation models.
\newblock \emph{arXiv preprint arXiv:2309.06256}, 2023.

\bibitem[Liu et~al.(2021)Liu, Liu, Jin, Stone, and Liu]{liu2021conflict}
Bo~Liu, Xingchao Liu, Xiaojie Jin, Peter Stone, and Qiang Liu.
\newblock {Conflict-Averse Gradient Descent for Multi-task Learning}.
\newblock December 2021.

\bibitem[Lu et~al.(2023)Lu, Yuan, Yuan, Lin, Lin, Tan, Zhou, and Zhou]{lu2023instag}
Keming Lu, Hongyi Yuan, Zheng Yuan, Runji Lin, Junyang Lin, Chuanqi Tan, Chang Zhou, and Jingren Zhou.
\newblock \# instag: Instruction tagging for analyzing supervised fine-tuning of large language models.
\newblock In \emph{The Twelfth International Conference on Learning Representations}, 2023.

\bibitem[Malladi et~al.(2023{\natexlab{a}})Malladi, Gao, Nichani, Damian, Lee, Chen, and Arora]{malladi2023fine}
Sadhika Malladi, Tianyu Gao, Eshaan Nichani, Alex Damian, Jason~D Lee, Danqi Chen, and Sanjeev Arora.
\newblock Fine-tuning language models with just forward passes.
\newblock \emph{Advances in Neural Information Processing Systems}, 36:\penalty0 53038--53075, 2023{\natexlab{a}}.

\bibitem[Malladi et~al.(2023{\natexlab{b}})Malladi, Wettig, Yu, Chen, and Arora]{malladi2023kernel}
Sadhika Malladi, Alexander Wettig, Dingli Yu, Danqi Chen, and Sanjeev Arora.
\newblock A kernel-based view of language model fine-tuning.
\newblock In \emph{International Conference on Machine Learning}, pages 23610--23641. PMLR, 2023{\natexlab{b}}.

\bibitem[Maninis et~al.(2019)Maninis, Radosavovic, and Kokkinos]{maninis2019attentive}
Kevis-Kokitsi Maninis, Ilija Radosavovic, and Iasonas Kokkinos.
\newblock {Attentive single-tasking of multiple tasks}.
\newblock June 2019.

\bibitem[Meng et~al.(2024)Meng, Xia, and Chen]{meng2024simpo}
Yu~Meng, Mengzhou Xia, and Danqi Chen.
\newblock Simpo: Simple preference optimization with a reference-free reward.
\newblock \emph{arXiv preprint arXiv:2405.14734}, 2024.

\bibitem[Miettinen(1999)]{miettinen1999nonlinear}
Kaisa Miettinen.
\newblock \emph{Nonlinear multiobjective optimization}, volume~12.
\newblock Springer Science \& Business Media, 1999.

\bibitem[Munos et~al.(2023)Munos, Valko, Calandriello, Azar, Rowland, Guo, Tang, Geist, Mesnard, Michi, et~al.]{munos2023nash}
R{\'e}mi Munos, Michal Valko, Daniele Calandriello, Mohammad~Gheshlaghi Azar, Mark Rowland, Zhaohan~Daniel Guo, Yunhao Tang, Matthieu Geist, Thomas Mesnard, Andrea Michi, et~al.
\newblock Nash learning from human feedback.
\newblock \emph{arXiv preprint arXiv:2312.00886}, 2023.

\bibitem[Navon et~al.(2022)Navon, Shamsian, Achituve, Maron, Kawaguchi, Chechik, and Fetaya]{navon2022multi}
Aviv Navon, Aviv Shamsian, Idan Achituve, Haggai Maron, Kenji Kawaguchi, Gal Chechik, and Ethan Fetaya.
\newblock {Multi-Task Learning as a Bargaining Game}.
\newblock \emph{arXiv preprint:2202.01017}, 2022.

\bibitem[Nikdan et~al.(2024)Nikdan, Tabesh, and Alistarh]{nikdan2024rosa}
Mahdi Nikdan, Soroush Tabesh, and Dan Alistarh.
\newblock Rosa: Accurate parameter-efficient fine-tuning via robust adaptation.
\newblock \emph{arXiv preprint arXiv:2401.04679}, 2024.

\bibitem[OpenAI(2022)]{openai2022chatgpt}
OpenAI.
\newblock Chatgpt: Optimizing language models for dialogue, November 2022.
\newblock URL \url{https://openai.com/blog/chatgpt/}.
\newblock Accessed: 2024-08-28.

\bibitem[Orabona(2020)]{orabona2020lastiterate}
Francesco Orabona.
\newblock Last iterate of sgd converges even in unbounded domains, 2020.
\newblock URL \url{https://parameterfree.com/2020/08/07/last-iterate-of-sgd-converges-even-in-unbounded-domains/}.
\newblock Accessed: 2024-09-10.

\bibitem[Ouyang et~al.(2022)Ouyang, Wu, Jiang, Almeida, Wainwright, Mishkin, Zhang, Agarwal, Slama, Ray, et~al.]{ouyang2022training}
Long Ouyang, Jeffrey Wu, Xu~Jiang, Diogo Almeida, Carroll Wainwright, Pamela Mishkin, Chong Zhang, Sandhini Agarwal, Katarina Slama, Alex Ray, et~al.
\newblock Training language models to follow instructions with human feedback.
\newblock \emph{Advances in neural information processing systems}, 35:\penalty0 27730--27744, 2022.

\bibitem[Peng et~al.(2023)Peng, Li, He, Galley, and Gao]{peng2023instruction}
Baolin Peng, Chunyuan Li, Pengcheng He, Michel Galley, and Jianfeng Gao.
\newblock Instruction tuning with gpt-4.
\newblock \emph{arXiv preprint arXiv:2304.03277}, 2023.

\bibitem[Qi et~al.(2023)Qi, Zeng, Xie, Chen, Jia, Mittal, and Henderson]{qi2023fine}
Xiangyu Qi, Yi~Zeng, Tinghao Xie, Pin-Yu Chen, Ruoxi Jia, Prateek Mittal, and Peter Henderson.
\newblock Fine-tuning aligned language models compromises safety, even when users do not intend to!
\newblock \emph{arXiv preprint arXiv:2310.03693}, 2023.

\bibitem[Rafailov et~al.(2024)Rafailov, Sharma, Mitchell, Manning, Ermon, and Finn]{rafailov2024direct}
Rafael Rafailov, Archit Sharma, Eric Mitchell, Christopher~D Manning, Stefano Ermon, and Chelsea Finn.
\newblock Direct preference optimization: Your language model is secretly a reward model.
\newblock \emph{Advances in Neural Information Processing Systems}, 36, 2024.

\bibitem[Rockafellar(1970)]{rockafellar1970convex}
Ralph~Tyrell Rockafellar.
\newblock Convex analysis, 1970.

\bibitem[R{\"o}ttger et~al.(2023)R{\"o}ttger, Kirk, Vidgen, Attanasio, Bianchi, and Hovy]{rottger2023xstest}
Paul R{\"o}ttger, Hannah~Rose Kirk, Bertie Vidgen, Giuseppe Attanasio, Federico Bianchi, and Dirk Hovy.
\newblock Xstest: A test suite for identifying exaggerated safety behaviours in large language models.
\newblock \emph{arXiv preprint arXiv:2308.01263}, 2023.

\bibitem[Roziere et~al.(2023)Roziere, Gehring, Gloeckle, Sootla, Gat, Tan, Adi, Liu, Sauvestre, Remez, et~al.]{roziere2023code}
Baptiste Roziere, Jonas Gehring, Fabian Gloeckle, Sten Sootla, Itai Gat, Xiaoqing~Ellen Tan, Yossi Adi, Jingyu Liu, Romain Sauvestre, Tal Remez, et~al.
\newblock Code llama: Open foundation models for code.
\newblock \emph{arXiv preprint arXiv:2308.12950}, 2023.

\bibitem[Shen et~al.(2024)Shen, Yang, and Chen]{shen2024principled}
Han Shen, Zhuoran Yang, and Tianyi Chen.
\newblock Principled penalty-based methods for bilevel reinforcement learning and rlhf.
\newblock \emph{arXiv preprint arXiv:2402.06886}, 2024.

\bibitem[Shi and Lipani(2023)]{shi2023dept}
Zhengxiang Shi and Aldo Lipani.
\newblock Dept: Decomposed prompt tuning for parameter-efficient fine-tuning.
\newblock \emph{arXiv preprint arXiv:2309.05173}, 2023.

\bibitem[Sun et~al.(2024)Sun, Shen, Zhou, Zhang, Chen, Cox, Yang, and Gan]{sun2024principle}
Zhiqing Sun, Yikang Shen, Qinhong Zhou, Hongxin Zhang, Zhenfang Chen, David Cox, Yiming Yang, and Chuang Gan.
\newblock Principle-driven self-alignment of language models from scratch with minimal human supervision.
\newblock \emph{Advances in Neural Information Processing Systems}, 36, 2024.

\bibitem[Tang et~al.(2020)Tang, Tran, Li, Chen, Goyal, Chaudhary, Gu, and Fan]{tang2020multilingual}
Yuqing Tang, Chau Tran, Xian Li, Peng-Jen Chen, Naman Goyal, Vishrav Chaudhary, Jiatao Gu, and Angela Fan.
\newblock Multilingual translation with extensible multilingual pretraining and finetuning.
\newblock \emph{arXiv preprint arXiv:2008.00401}, 2020.

\bibitem[Tian et~al.(2024)Tian, Liu, Smith, and Kira]{tian2024fast}
Junjiao Tian, Yen-Cheng Liu, James~S Smith, and Zsolt Kira.
\newblock Fast trainable projection for robust fine-tuning.
\newblock \emph{Advances in Neural Information Processing Systems}, 36, 2024.

\bibitem[Tian et~al.(2023)Tian, Mitchell, Yao, Manning, and Finn]{tian2023fine}
Katherine Tian, Eric Mitchell, Huaxiu Yao, Christopher~D Manning, and Chelsea Finn.
\newblock Fine-tuning language models for factuality.
\newblock \emph{arXiv preprint arXiv:2311.08401}, 2023.

\bibitem[Wang et~al.(2023)Wang, Liu, and Jin]{wang2023rlhf}
Yuanhao Wang, Qinghua Liu, and Chi Jin.
\newblock Is rlhf more difficult than standard rl? a theoretical perspective.
\newblock \emph{Advances in Neural Information Processing Systems}, 36:\penalty0 76006--76032, 2023.

\bibitem[Wei et~al.(2021)Wei, Bosma, Zhao, Guu, Yu, Lester, Du, Dai, and Le]{wei2021finetuned}
Jason Wei, Maarten Bosma, Vincent~Y Zhao, Kelvin Guu, Adams~Wei Yu, Brian Lester, Nan Du, Andrew~M Dai, and Quoc~V Le.
\newblock Finetuned language models are zero-shot learners.
\newblock \emph{arXiv preprint arXiv:2109.01652}, 2021.

\bibitem[Wolf et~al.(2020)Wolf, Debut, Sanh, Chaumond, Delangue, Moi, Cistac, Rault, Louf, Funtowicz, Davison, Shleifer, von Platen, Ma, Jernite, Plu, Xu, Scao, Gugger, Drame, Lhoest, and Rush]{wolf-etal-2020-transformers}
Thomas Wolf, Lysandre Debut, Victor Sanh, Julien Chaumond, Clement Delangue, Anthony Moi, Pierric Cistac, Tim Rault, Rémi Louf, Morgan Funtowicz, Joe Davison, Sam Shleifer, Patrick von Platen, Clara Ma, Yacine Jernite, Julien Plu, Canwen Xu, Teven~Le Scao, Sylvain Gugger, Mariama Drame, Quentin Lhoest, and Alexander~M. Rush.
\newblock Transformers: State-of-the-art natural language processing.
\newblock In \emph{Proceedings of the 2020 Conference on Empirical Methods in Natural Language Processing: System Demonstrations}, pages 38--45, Online, October 2020. Association for Computational Linguistics.
\newblock URL \url{https://www.aclweb.org/anthology/2020.emnlp-demos.6}.

\bibitem[Xie et~al.(2024)Xie, Qi, Zeng, Huang, Sehwag, Huang, He, Wei, Li, Sheng, et~al.]{xie2024sorry}
Tinghao Xie, Xiangyu Qi, Yi~Zeng, Yangsibo Huang, Udari~Madhushani Sehwag, Kaixuan Huang, Luxi He, Boyi Wei, Dacheng Li, Ying Sheng, et~al.
\newblock Sorry-bench: Systematically evaluating large language model safety refusal behaviors.
\newblock \emph{arXiv preprint arXiv:2406.14598}, 2024.

\bibitem[Xiong et~al.(2024)Xiong, Dong, Ye, Wang, Zhong, Ji, Jiang, and Zhang]{xiong2024iterative}
Wei Xiong, Hanze Dong, Chenlu Ye, Ziqi Wang, Han Zhong, Heng Ji, Nan Jiang, and Tong Zhang.
\newblock Iterative preference learning from human feedback: Bridging theory and practice for rlhf under kl-constraint.
\newblock In \emph{Forty-first International Conference on Machine Learning}, 2024.

\bibitem[Xu et~al.(2024)Xu, Fu, Gao, Ye, Liu, Mei, Wang, Yu, and Wu]{xu2024dpo}
Shusheng Xu, Wei Fu, Jiaxuan Gao, Wenjie Ye, Weilin Liu, Zhiyu Mei, Guangju Wang, Chao Yu, and Yi~Wu.
\newblock Is dpo superior to ppo for llm alignment? a comprehensive study.
\newblock \emph{arXiv preprint arXiv:2404.10719}, 2024.

\bibitem[Yang et~al.(2024)Yang, Pan, Luo, Qiu, Zhong, Yu, and Chen]{yang2024rewards}
Rui Yang, Xiaoman Pan, Feng Luo, Shuang Qiu, Han Zhong, Dong Yu, and Jianshu Chen.
\newblock Rewards-in-context: Multi-objective alignment of foundation models with dynamic preference adjustment.
\newblock \emph{arXiv preprint arXiv:2402.10207}, 2024.

\bibitem[Zellers et~al.(2019)Zellers, Holtzman, Bisk, Farhadi, and Choi]{zellers2019hellaswag}
Rowan Zellers, Ari Holtzman, Yonatan Bisk, Ali Farhadi, and Yejin Choi.
\newblock Hellaswag: Can a machine really finish your sentence?
\newblock \emph{arXiv preprint arXiv:1905.07830}, 2019.

\bibitem[Zhang et~al.(2023{\natexlab{a}})Zhang, Chen, Bukharin, Karampatziakis, He, Cheng, Chen, and Zhao]{zhang2023adalora}
Qingru Zhang, Minshuo Chen, Alexander Bukharin, Nikos Karampatziakis, Pengcheng He, Yu~Cheng, Weizhu Chen, and Tuo Zhao.
\newblock Adalora: Adaptive budget allocation for parameter-efficient fine-tuning.
\newblock \emph{arXiv preprint arXiv:2303.10512}, 2023{\natexlab{a}}.

\bibitem[Zhang et~al.(2023{\natexlab{b}})Zhang, Dong, Li, Zhang, Sun, Wang, Li, Hu, Zhang, Wu, et~al.]{zhang2023instruction}
Shengyu Zhang, Linfeng Dong, Xiaoya Li, Sen Zhang, Xiaofei Sun, Shuhe Wang, Jiwei Li, Runyi Hu, Tianwei Zhang, Fei Wu, et~al.
\newblock Instruction tuning for large language models: A survey.
\newblock \emph{arXiv preprint arXiv:2308.10792}, 2023{\natexlab{b}}.

\bibitem[Zhang et~al.(2022)Zhang, Roller, Goyal, Artetxe, Chen, Chen, Dewan, Diab, Li, Lin, Mihaylov, Ott, Shleifer, Shuster, Simig, Koura, Sridhar, Wang, and Zettlemoyer]{zhang2022opt}
Susan Zhang, Stephen Roller, Naman Goyal, Mikel Artetxe, Moya Chen, Shuohui Chen, Christopher Dewan, Mona Diab, Xian Li, Xi~Victoria Lin, Todor Mihaylov, Myle Ott, Sam Shleifer, Kurt Shuster, Daniel Simig, Punit~Singh Koura, Anjali Sridhar, Tianlu Wang, and Luke Zettlemoyer.
\newblock Opt: Open pre-trained transformer language models, 2022.

\bibitem[Zhao et~al.(2024)Zhao, Andriushchenko, Croce, and Flammarion]{zhao2024long}
Hao Zhao, Maksym Andriushchenko, Francesco Croce, and Nicolas Flammarion.
\newblock Long is more for alignment: A simple but tough-to-beat baseline for instruction fine-tuning.
\newblock \emph{arXiv preprint arXiv:2402.04833}, 2024.

\bibitem[Zhao et~al.(2023)Zhao, Dang, and Grover]{zhao2023group}
Siyan Zhao, John Dang, and Aditya Grover.
\newblock Group preference optimization: Few-shot alignment of large language models.
\newblock \emph{arXiv preprint arXiv:2310.11523}, 2023.

\bibitem[Zhong et~al.(2024)Zhong, Feng, Xiong, Zhao, He, Bian, and Wang]{zhong2024dpo}
Han Zhong, Guhao Feng, Wei Xiong, Li~Zhao, Di~He, Jiang Bian, and Liwei Wang.
\newblock Dpo meets ppo: Reinforced token optimization for rlhf.
\newblock \emph{arXiv preprint arXiv:2404.18922}, 2024.

\bibitem[Zhu et~al.(2023{\natexlab{a}})Zhu, Jordan, and Jiao]{zhu2023principled}
Banghua Zhu, Michael Jordan, and Jiantao Jiao.
\newblock Principled reinforcement learning with human feedback from pairwise or k-wise comparisons.
\newblock In \emph{International Conference on Machine Learning}, pages 43037--43067. PMLR, 2023{\natexlab{a}}.

\bibitem[Zhu et~al.(2023{\natexlab{b}})Zhu, Chen, Shen, Li, and Elhoseiny]{zhu2023minigpt}
Deyao Zhu, Jun Chen, Xiaoqian Shen, Xiang Li, and Mohamed Elhoseiny.
\newblock Minigpt-4: Enhancing vision-language understanding with advanced large language models.
\newblock \emph{arXiv preprint arXiv:2304.10592}, 2023{\natexlab{b}}.

\end{thebibliography}
\end{document}